%% file: arxiv_main.tex
\documentclass{article}
\usepackage{spconf,amsmath,graphicx,hyperref}
\usepackage{latexsym}
\usepackage{amssymb}
\usepackage{amsmath}
\usepackage{amsthm}
\usepackage{booktabs}
\usepackage{enumitem}
\usepackage{graphicx}
\usepackage{color}

\usepackage{float}
\usepackage{amsmath,amssymb,amsfonts}  

\usepackage{algorithm}
\usepackage{algorithmic}
\usepackage{array}
\usepackage[caption=false,font=normalsize,labelfont=sf,textfont=sf]{subfig}
\usepackage{textcomp}
\usepackage{stfloats}
\usepackage{url}
\usepackage{verbatim}
\usepackage{graphicx}
\usepackage{booktabs} 

\usepackage{latexsym}
\usepackage{amssymb}
\usepackage{amsmath}
\usepackage{amsthm}
\usepackage{booktabs}
\usepackage{enumitem}
\usepackage{graphicx}
\usepackage{color}

\usepackage{times}
\usepackage{soul}
\usepackage{url}
\usepackage{textcomp}
\usepackage{framed}
\usepackage{multirow}
\theoremstyle{definition}
\newtheorem{assumption}{Assumption}
\newtheorem{definition}{Definition}
\usepackage{comment}
\usepackage{subfig} 
\usepackage{xcolor}

\newtheorem{theorem}{Theorem}

\title{FedSSG: Expectation-Gated and History-Aware Drift Alignment for Federated Learning}
\name{Zhanting Zhou$^{\ast}$\thanks{$^{\ast}$Equal Contribution.}, 
      Jinshan Lai$^{\ast}$, Fengchun Zhang, Zeqin Wu and Fengli Zhang}

\address{School of Information and Software Engineering\\
University of Electronic Science and Technology of China\\
Chengdu, Sichuan, China}
\begin{document}
\ninept
\maketitle
\begin{abstract}
Non-IID data and partial participation induce client drift and inconsistent local optima in federated learning, causing unstable convergence and accuracy loss. We present \textbf{FedSSG}, a \emph{Stochastic Sampling-Guided}, \emph{history-aware} drift alignment method. FedSSG maintains a per-client drift memory that accumulates local model differences as a lightweight sketch of historical gradients; crucially, it \emph{gates} both the memory update and the local alignment term by a smooth function of the observed/expected participation ratio (a phase-by-expectation signal derived from the server sampler). This statistically grounded gate stays weak and smooth when sampling noise dominates early, then strengthens once participation statistics stabilize, contracting the local–global gap without extra communication. Across CIFAR-10/100 with 100/500 clients and 2–15\% participation, FedSSG consistently outperforms strong drift-aware baselines and accelerates convergence; on our benchmarks it improves test accuracy by up to a few points (e.g., $\sim$+0.9 on CIFAR-10 and $\sim$+2.7 on CIFAR-100 on average over the Top-2 baseline) and yields $\sim$4.5$\times$ faster target-accuracy convergence on average. The method adds only $O(d)$ client memory and a constant-time gate, and degrades gracefully to a mild regularizer under near-IID or uniform sampling. FedSSG shows that sampling statistics can be turned into a principled, history-aware phase control to stabilize and speed up federated training. CODE: \url{https://github.com/itoritsu/FedSSG}
\end{abstract}
\section{Introduction}

\begin{figure*}[ht]
    \centering
    \setlength{\belowcaptionskip}{10pt}
    \subfloat[]{\includegraphics[width=0.23\textwidth]{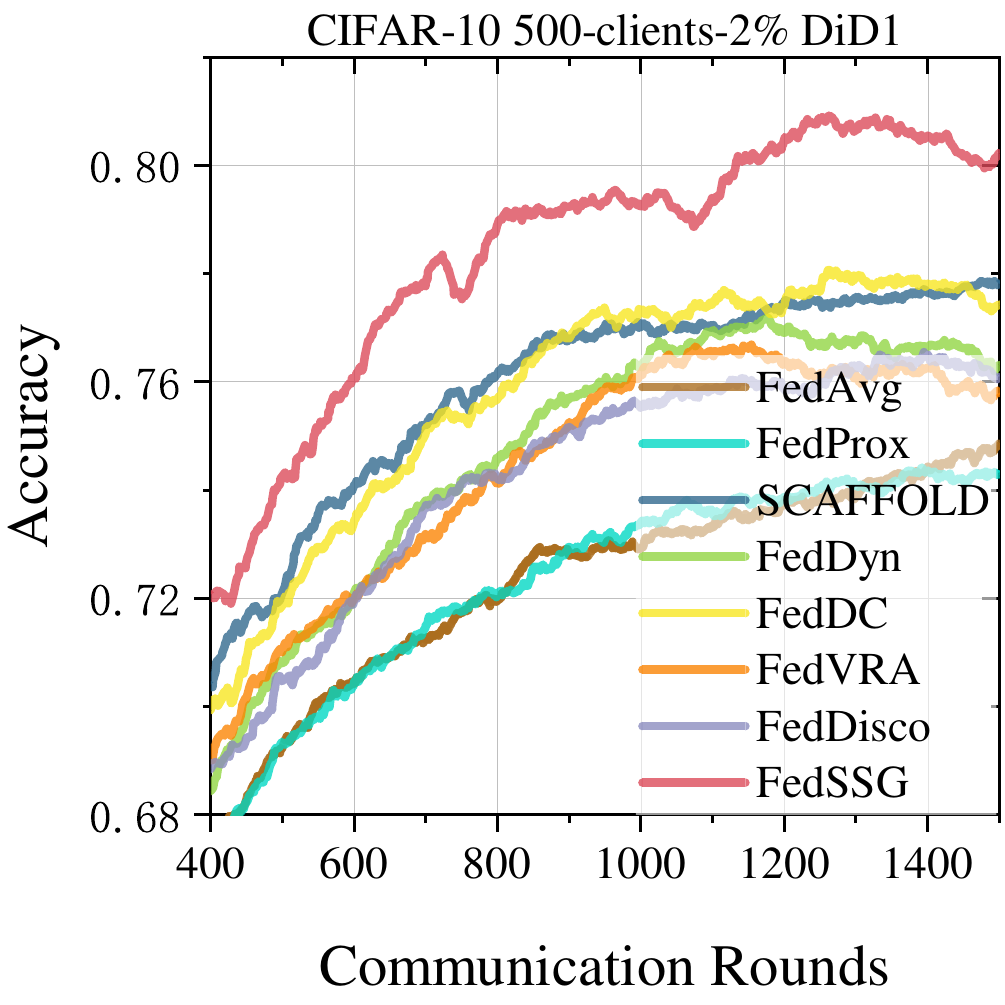}} \hfill
    \subfloat[]{\includegraphics[width=0.23\textwidth]{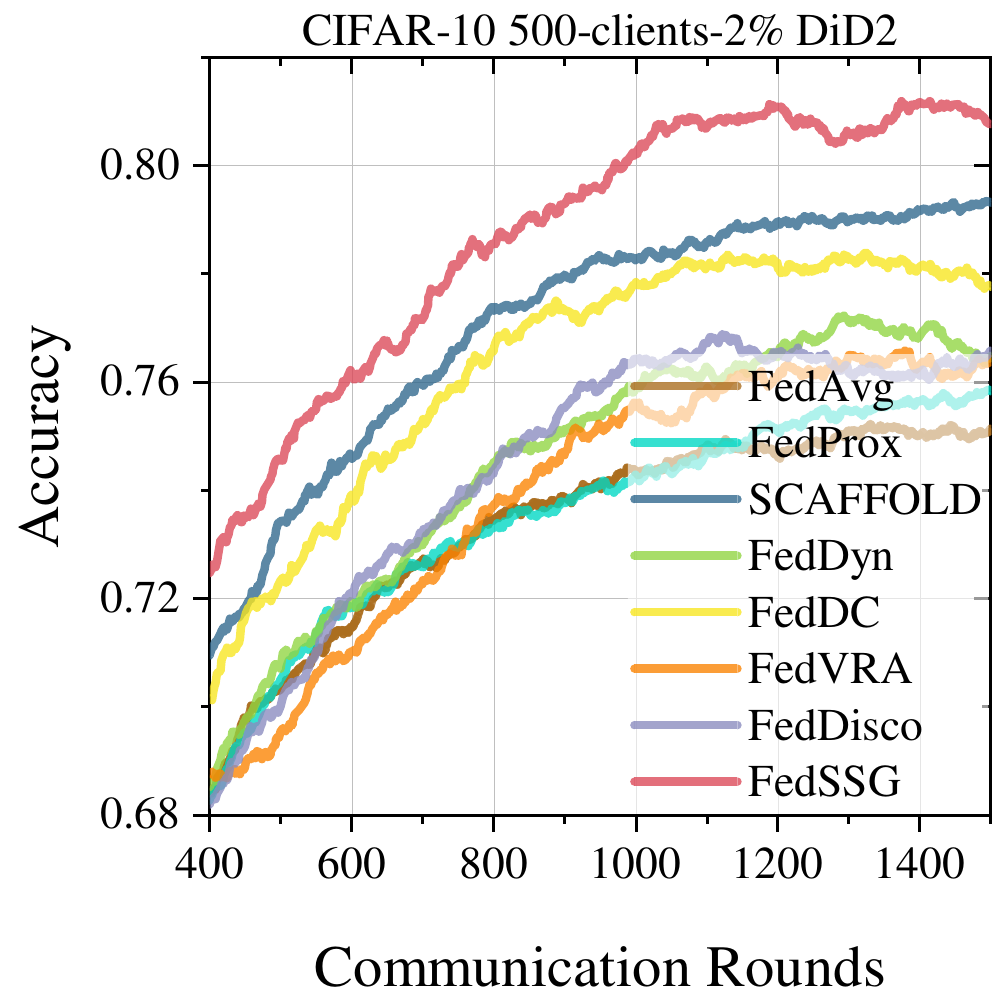}} \hfill
    \subfloat[]{\includegraphics[width=0.23\textwidth]{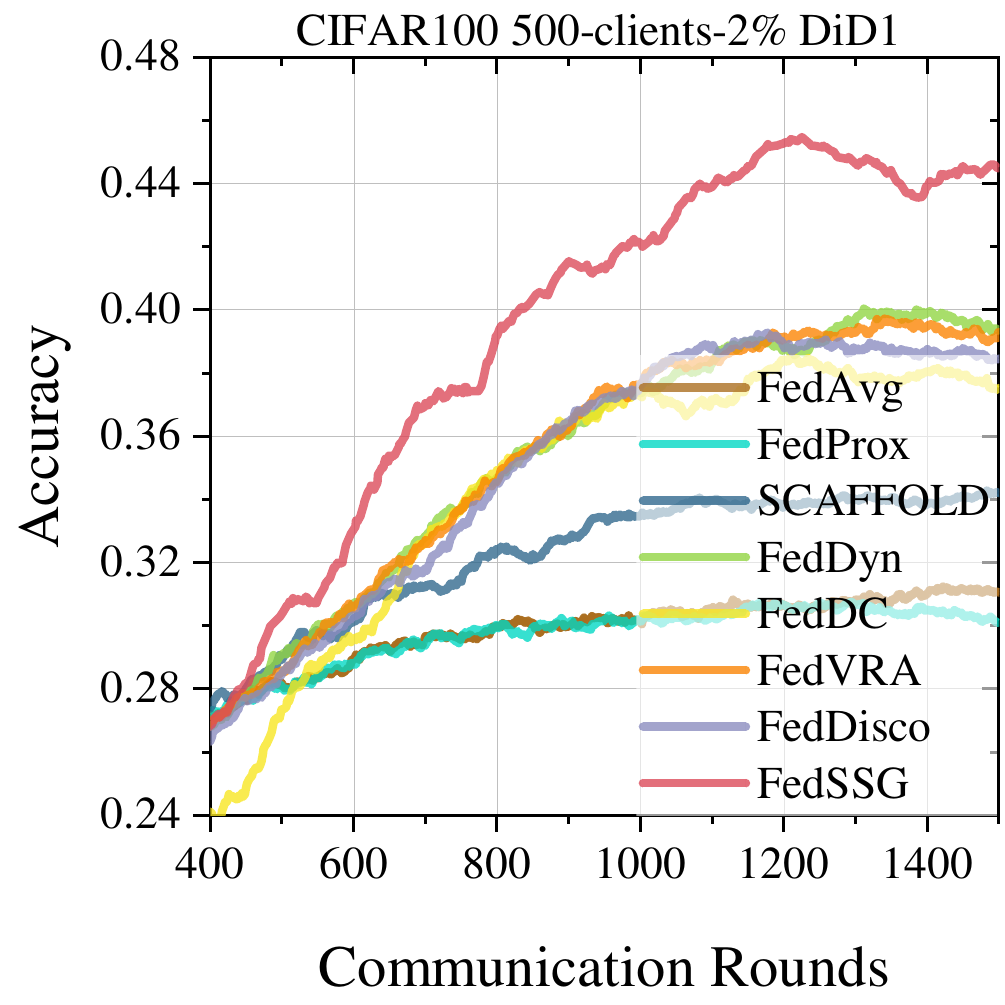}} \hfill
    \subfloat[]{\includegraphics[width=0.23\textwidth]{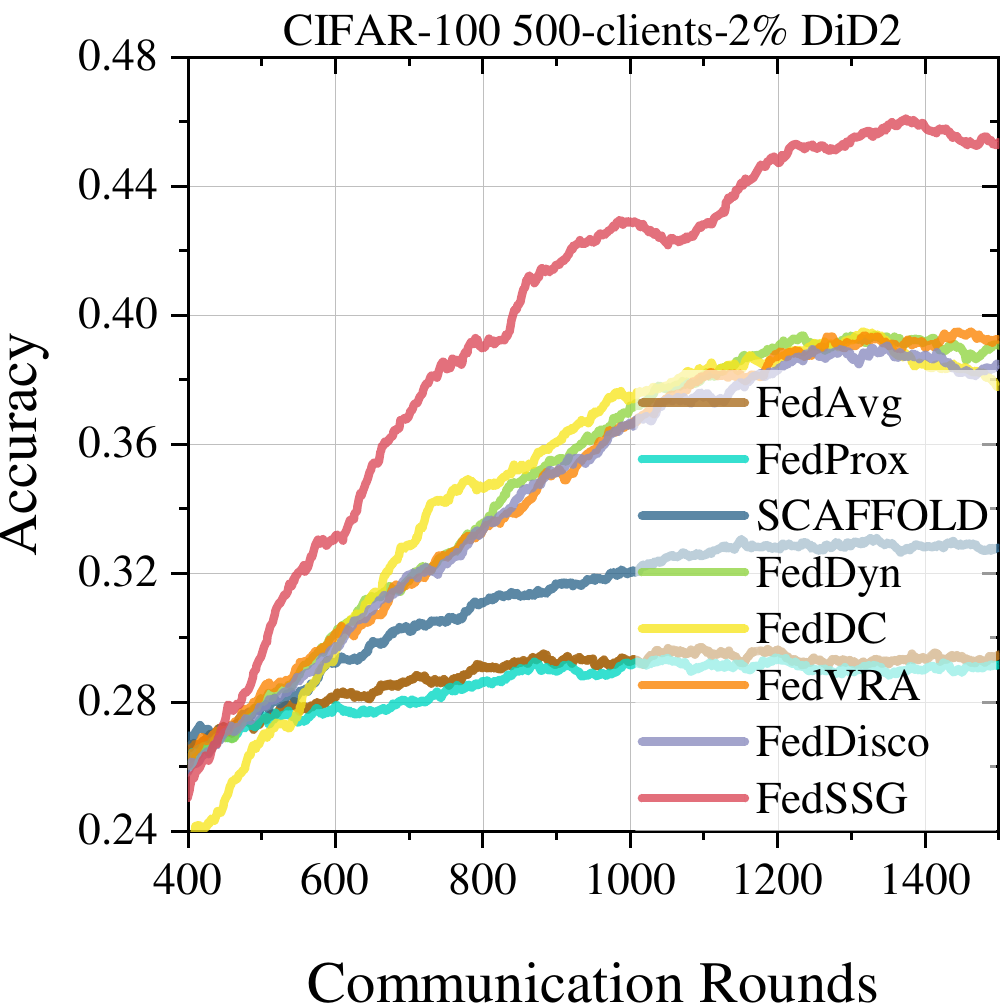}}
    \caption{\label{fig:performance}The learning curves of experiments on CIFAR-10/CIFAR-100 with 500-clients-2\%-participate-in DiD1/DiD2 settings. The test accuracy and convergence speed of FedSSG performs much superior than baselines. We use DiD1 and DiD2 to present the Dirichlet parameters are 0.3 and 0.6 respectively. }
        \setlength{\belowcaptionskip}{10pt}
\end{figure*}



\textbf{Background \& Challenge.}
Federated learning (FL) trains a global model on decentralized data while keeping raw data local~\cite{fedavg,kairouz2021openproblems}. 
In practical deployments, \emph{statistical heterogeneity} (non-IID) and \emph{partial participation} are ubiquitous and jointly induce \emph{client drift}—a persistent local--global mismatch that slows convergence and degrades final accuracy~\cite{scaffold,fedprox}. 
This phenomenon is exacerbated by stochastic client sampling: uneven selections across rounds change how “large” vs.\ “small” local updates should be interpreted across phases and inflate aggregation variance~\cite{wang2022unified,oort2021,tifl2020,fraboni2021clustered,rizk2021isfl}.

\textbf{Limitations of existing methods (inline related work).}
FedAvg and FedProx either ignore client drift or damp it using a \emph{fixed} proximal term~\cite{fedavg,fedprox}. 
SCAFFOLD reduces drift with control variates but incurs extra state and communication for these variates~\cite{scaffold}. 
Alignment-based methods such as FedDyn and FedDC introduce dynamic regularizers or explicit drift variables; however, the \emph{alignment strength} is typically kept fixed across training and thus ignores the statistical signal in the server's stochastic sampling (who participates how often)~\cite{feddyn,feddc}. 
Normalization-style fixes (e.g., FedNova) address objective inconsistency from unequal local steps, yet still do not leverage participation-phase statistics to adapt alignment strength~\cite{fednova}.

\textbf{Problem Setup.}
Let the global objective be
\begin{equation}
\min_{\omega\in\mathbb{R}^d}~ f(\omega)=\frac{1}{N}\sum_{i=1}^{N} f_i(\omega),\quad
f_i(\omega)=\mathbb{E}_{\zeta_i}\!\left[\ell(\omega;\zeta_i)\right],
\end{equation}
and define a drift proxy $\delta_i(\omega)=\big\|\mathbb{E}_{\zeta_i}\nabla f_i(\omega;\zeta_i)-\nabla f(\omega)\big\|$ (non-IID $\Rightarrow \delta_i$ large).
Following the ADMM/dynamic-regularization view~\cite{feddyn,feddc}, the local objective at round $t$ can be written as
\begin{equation}
\min_{\theta_i}\; f_i(\theta_i)\;+\;R_i\!\big(\theta_i;\,\omega^{t-1},\,h_i^{t-1}\big),
\end{equation}
where $h_i$ is a \emph{local drift variable} updated from local model differences and $R_i$ penalizes local--global inconsistency (e.g., inner-product or proximal forms).

\textbf{Our observation \& motivation.}
Under partial participation, update magnitudes have phase-specific meaning: early large steps often reflect global \emph{catch-up}, whereas later small but systematic steps reveal \emph{true heterogeneity}. 
A fixed alignment either over-regularizes early or under-regularizes late. 
This motivates a \emph{phase-by-expectation} gating: use the \emph{observed/expected participation ratio} to \emph{statistically} gate both the drift memory and the alignment term—weak and smooth when sampling noise dominates, stronger once participation statistics stabilize.

\textbf{Method in one line.}
For client $i$ at round $t$, let $c_i^{1:t}$ be the observed selection count and $\mu_i^t=\mathbb{E}[C_i^{1:t}]=p_i t$ the expected count under the server's sampler ($p_i=m/N$ for uniform sampling). 
Define $r_i^t=c_i^{1:t}/(\mu_i^t+\varepsilon)$ and a bounded, smooth mapping $\phi(\cdot)$. 
We update the drift memory by
\begin{equation}
h_i^t \leftarrow h_i^{t-1} + \underbrace{\phi(r_i^t)}_{\text{expectation-gated factor}}\;\Delta\theta_i^t,
\end{equation}
and use the \emph{same} statistical gate in the local objective:
\begin{equation}
\min_{\theta_i}\; f_i(\theta_i)\;+\;\underbrace{\phi(r_i^t)}_{\text{expectation-gated}}\;
\mathcal{A}\!\big(\theta_i,\;\omega^{t-1}-h_i^{t-1}\big),
\end{equation}
where $\mathcal{A}$ follows the inner-product/proximal alignment used in FedDyn/FedDC-style formulations.
This expectation-gated phase keeps early training unconstrained enough to approach the global optimum, then progressively strengthens late-stage alignment to shrink the local–global gap.

\textbf{Effectiveness \& Efficiency.}
On CIFAR-10/100 with heavy non-IID and $2\%$ participation in Fig.\ref{fig:performance}, our expectation-gated phase consistently accelerates convergence to a target accuracy and improves final accuracy over FedAvg, FedProx, SCAFFOLD, FedDyn, FedDC, FedVRA \cite{fedvra} and FedDisco \cite{feddisco}. The gating requires only $O(d)$ extra memory for $h_i$ and a constant-time scalar weight per step, preserving communication/computation efficiency while improving robustness to sampling-induced drift fluctuations. 

\section{Methodology}

\begin{figure}[t!]
    \centering
        \setlength{\belowcaptionskip}{10pt}
    \subfloat[]{\includegraphics[width=0.23\textwidth]{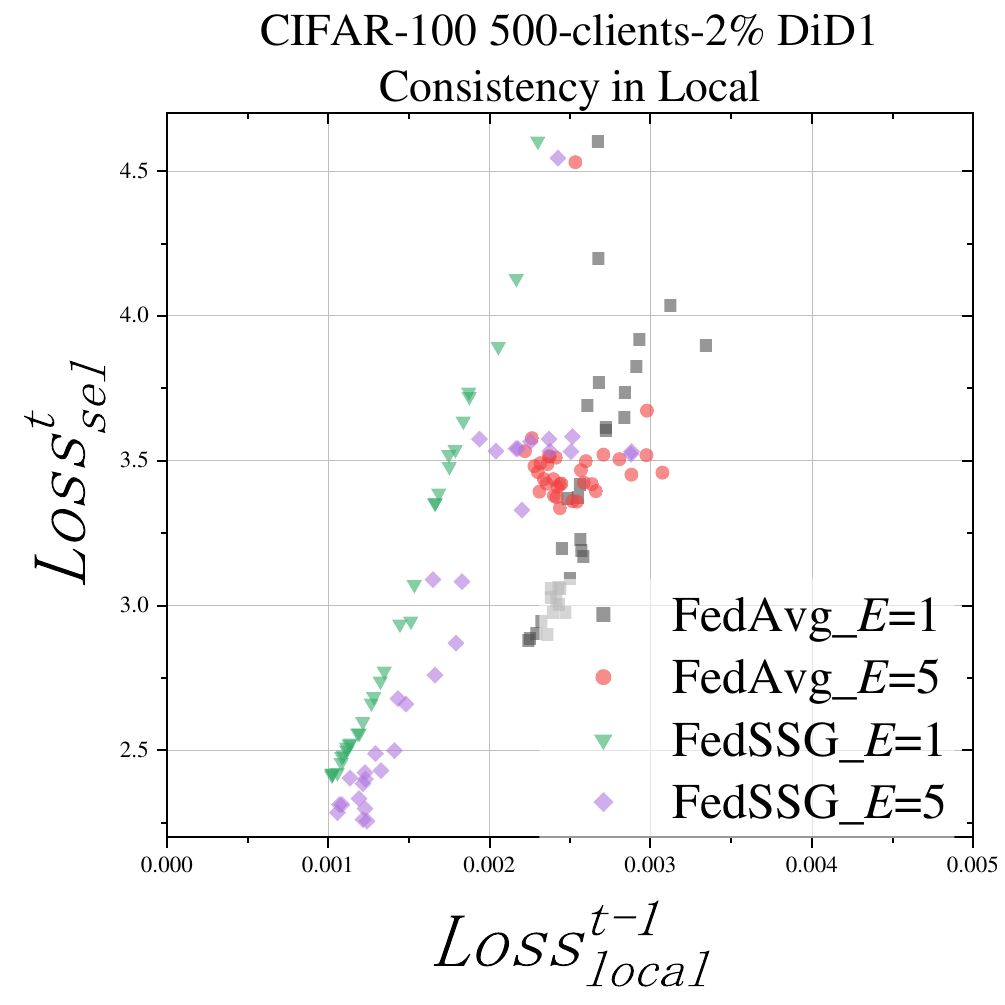}} \hfill
    \subfloat[]{\includegraphics[width=0.23\textwidth]{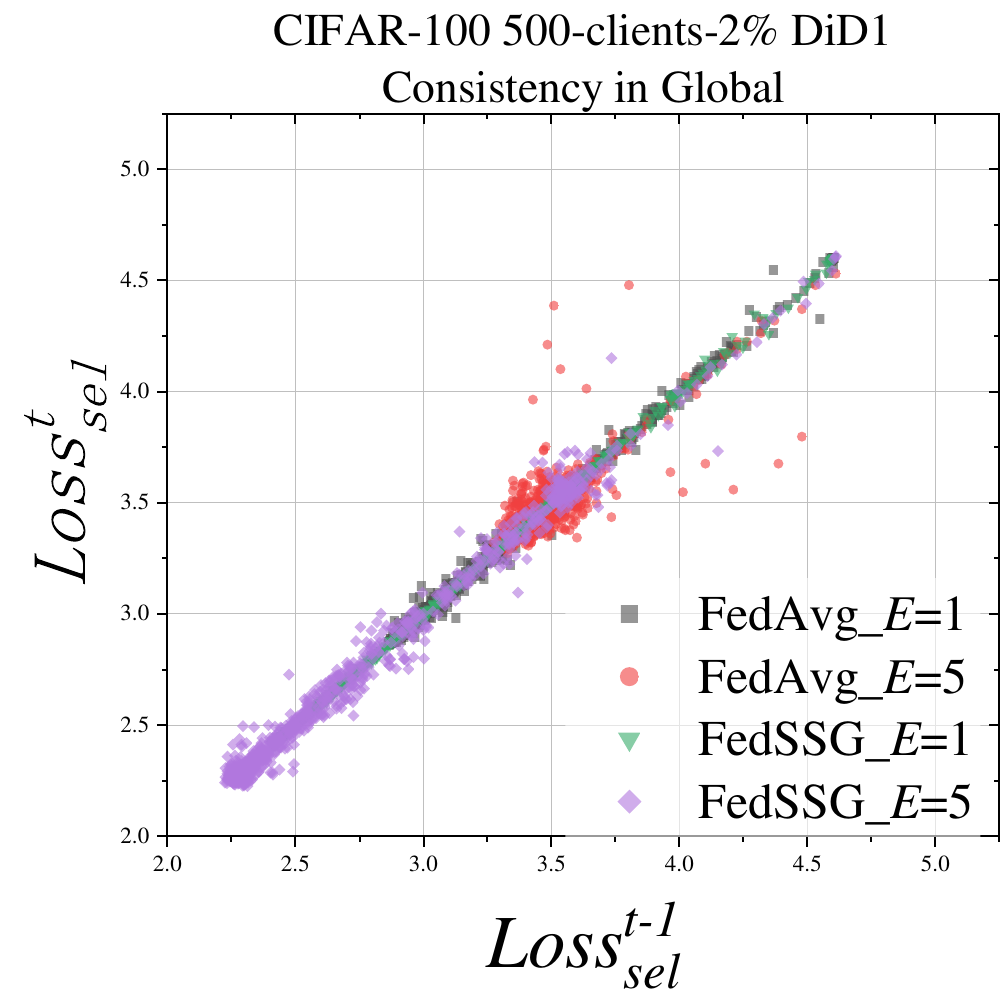}}
	
	\caption{Comparison for consistency of the same client between FedAvg and FedSSG between $(t-1)$ and $t$ times. For both figures, the more the scatter points are concentrated on the origin of the coordinate axis, the more consistent and efficient the method is.}
	\label{fig:kfc}
\end{figure}

\textbf{Observation motivating Fig.~\ref{fig:kfc}.}
Under non-IID and partial participation, early large local steps typically reflect \emph{global catch-up}, while late small but systematic steps reveal \emph{true heterogeneity}. A fixed-strength alignment either over-regularizes early or under-regularizes late, leading to the spread of points away from the origin in Fig.~\ref{fig:kfc}. We therefore gate both the drift memory and the alignment term by a \emph{statistical phase} inferred from sampling.

\textbf{Expectation-gated factor.}
Let $c_i^{1:t}$ be the observed number of selections of client $i$ up to round $t$, and let $\mu_i^{t}\triangleq \mathbb{E}[C_i^{1:t}]$ be the \emph{expected} selections under the server sampler (for uniform sampling, $\mu_i^{t}= (m/N)\,t$). Define the observed–expected ratio
\begin{equation}
r_i^{t} \;=\; \frac{c_i^{1:t}}{\mu_i^{t}+\varepsilon},
\qquad \varepsilon>0.
\label{eq:ratio}
\end{equation}
We map $r_i^{t}$ to a bounded gate via a smooth, monotone function $g:\mathbb{R}_+\!\to\![0,1]$ (e.g., logistic),
\begin{equation}
\phi(r_i^{t}) \;=\; \alpha \cdot g\!\big(r_i^{t}\big),
\qquad 0<\alpha\le \alpha_{\max},
\label{eq:gate}
\end{equation}
where $\alpha$ is the \emph{user hyperparameter} controlling the maximum gate strength (dataset-dependent).

\textbf{Memory update (unchanged logic).}
We update the drift memory by accumulating the local difference with the expectation-gated factor:
\begin{equation}
h_i^{t} \;\leftarrow\; h_i^{t-1} \;+\; \phi(r_i^{t})\,\Delta\theta_i^{t}.
\label{eq:memory}
\end{equation}

\textbf{Expectation-gated alignment in the local objective.}
The local objective solved at client $i$ in round $t$ is
\begin{equation}
\min_{\theta_i}\; f_i(\theta_i) \;+\; \phi(r_i^{t})\;
\mathcal{A}\!\left(\theta_i,\;\omega^{t-1}-h_i^{t-1}\right),
\label{eq:local}
\end{equation}
where $\mathcal{A}$ is an alignment/consistency term in the style of dynamic-regularization methods. Two canonical instantiations (both used in prior work) are:
\begin{align}
\text{(proximal)}~\mathcal{A}_{\text{prox}}(\theta_i,u)&=\tfrac{1}{2}\|\theta_i-u\|_2^2, \\
\text{(inner-product)}~\mathcal{A}_{\text{ip}}(\theta_i,u)&=\left\langle \theta_i-u,\; h_i^{t-1}\right\rangle.
\end{align}
Our implementation uses the same $\phi(r_i^{t})$ to gate both the memory update \eqref{eq:memory} and the alignment weight in \eqref{eq:local}.

\textbf{Why this contracts the "consistency scatter" in Fig.~\ref{fig:kfc}.}
Early in training, $t$ is small and $r_i^{t}$ has high variance; the bounded, smooth map $g(\cdot)$ keeps $\phi(r_i^{t})$ \emph{moderate}, avoiding premature over-regularization and allowing fast global catch-up. As $t$ grows, $\mu_i^{t}$ increases and $\mathrm{Var}(r_i^{t})\!\downarrow$, so $\phi(r_i^{t})$ becomes \emph{stable and stronger}, shrinking the local–global gap. Empirically, this manifests as the concentration of scatter points toward the origin in both local and global consistency plots.


\textbf{Complexity and Degenerate Cases}
PBE introduces $O(d)$ extra memory per client for $h_i$ and a constant-time scalar gate $\phi(\cdot)$. When sampling is (nearly) uniform or data is close to IID, $r_i^{t}\!\approx\!1$ and $\phi(r_i^{t})\!\approx\!\alpha\,g(1)$, so the method smoothly degenerates to a mild, nearly constant regularizer without hurting performance.

\textbf{Hyper-parameters and Default Choices}
We expose no extra knobs: the gate uses the observed/expected participation ratio directly, \(\phi(r)=r\) (identity), with a fixed numerical safeguard \(\varepsilon=10^{-6}\) and \(\beta=0\); there is no user-tuned amplitude or temperature. This calibration-free choice is consistent with analyses showing that partial participation primarily enters through sampling-variance terms—hence simple, data-driven reweighting suffices (optional bounded reweighting such as clipping \(r\) to \([0,1]\) can be adopted without introducing new hyper-parameters). We fix the alignment to the proximal form \(\mathcal{A}_{\text{prox}}(\theta_i,u)=\tfrac12\|\theta_i-u\|_2^2\); related drift-correction methods safely initialize auxiliary states at zero, supporting \(\beta\!=\!0\). These defaults match our implementation and keep FedSSG free of tunable strength parameters while remaining theoretically aligned with standard FL treatments of drift and participation. \cite{wang2022unified,fraboni2021clustered,rizk2021isfl,scaffold,feddyn,feddc}

\begin{table*}[t!]
	\centering
    			\caption{\label{tab:performances}Test accuracy (\%) comparison among baselines and our proposed method on the datasets from last 50 tests, presented in the format of "mean$\pm$standard". The bolded text highlights the best performance in the corresponding experimental setting. FedSSG performs superior than all advance methods on CIFAR-10/100, breaking through the bottleneck of these methods. All data is accurate to two decimal places.}
	\setlength{\tabcolsep}{1.6mm}

			\begin{tabular}{ccccccccc}
				\toprule
				Settings & FedAvg & FedProx & SCAFFOLD & FedDyn & FedDC & FedVRA & FedDisco & \textbf{FedSSG}\\
				\midrule
				\multicolumn{9}{c}{100-clients-15\%}\\
				\midrule
				MNIST\&IID   & $98.08_{\pm0.03}$ & $98.06_{\pm0.02}$ & $98.39_{\pm0.02}$ & $98.38_{\pm0.02}$ & $98.38_{\pm0.02}$ & $98.28_{\pm0.04}$ & $98.23_{\pm0.02}$ & $\mathbf{98.43_{\pm0.04}}$	 \\
				MNIST\&DiD1  & $97.70_{\pm0.03}$ & $97.70_{\pm0.04}$ & $98.32_{\pm0.03}$ & $98.10_{\pm0.04}$ & $98.29_{\pm0.03}$ & $98.10_{\pm0.02}$ & $98.08_{\pm0.02}$ & $\mathbf{98.42_{\pm0.04}}$ \\
				MNIST\&DiD2  & $97.84_{\pm0.02}$ & $97.81_{\pm0.02}$ & $98.41_{\pm0.02}$ & $98.12_{\pm0.03}$ & $98.40_{\pm0.02}$ & $98.17_{\pm0.02}$ & $98.13_{\pm0.03}$ & $\mathbf{98.46_{\pm0.02}}$ \\
				EMNIST-L\&IID   & $94.77_{\pm0.05}$ & $94.73_{\pm0.07}$ & $\mathbf{95.44_{\pm0.05}}$ & $95.15_{\pm0.08}$ & $95.32_{\pm0.07}$ & $94.81_{\pm0.06}$ & $95.05_{\pm0.08}$ & $95.32_{\pm0.06}$	 \\
				EMNIST-L\&DiD1  & $93.94_{\pm0.07}$ & $93.73_{\pm0.23}$ & $94.85_{\pm0.06}$ & $94.57_{\pm0.06}$ & $95.01_{\pm0.07}$ & $94.66_{\pm0.05}$ & $94.65_{\pm0.06}$ & $\mathbf{95.05_{\pm0.08}}$	 \\
				EMNIST-L\&DiD2  & $94.38_{\pm0.04}$ & $94.47_{\pm0.07}$ & $95.23_{\pm0.07}$ & $94.71_{\pm0.06}$ & $95.11_{\pm0.09}$ & $94.80_{\pm0.10}$ & $94.87_{\pm0.06}$ & $\mathbf{95.31_{\pm0.07}}$	 \\
				\midrule
				\multicolumn{9}{c}{100-clients-10\%}\\
				\midrule
				CIFAR-10\&IID   & $81.82_{\pm0.11}$ & $81.97_{\pm0.07}$ & $84.45_{\pm0.07}$ & $84.05_{\pm0.07}$ & $84.90_{\pm0.17}$ & $84.12_{\pm0.11}$ & $84.09_{\pm0.08}$ & $\mathbf{85.18_{\pm0.10}}$ \\
				CIFAR-10\&DiD1  & $79.94_{\pm0.09}$ & $79.99_{\pm0.06}$ & $83.08_{\pm0.07}$ & $82.27_{\pm0.10}$ & $83.20_{\pm0.15}$ & $82.21_{\pm0.19}$ & $82.44_{\pm0.11}$ & $\mathbf{83.61_{\pm0.17}}$ \\
				CIFAR-10\&DiD2  & $80.97_{\pm0.07}$ & $81.26_{\pm0.06}$ & $83.52_{\pm0.06}$ & $83.38_{\pm0.20}$ & $84.23_{\pm0.18}$ & $83.52_{\pm0.10}$ & $83.05_{\pm0.13}$ & $\mathbf{84.47_{\pm0.23}}$ \\
				CIFAR-100\&IID  & $40.16_{\pm0.14}$ & $40.54_{\pm0.11}$ & $48.23_{\pm0.10}$ & $51.29_{\pm0.12}$ & $54.20_{\pm0.15}$ & $51.21_{\pm0.18}$ & $50.88_{\pm0.18}$ & $\mathbf{55.06_{\pm0.31}}$ \\
				CIFAR-100\&DiD1 & $41.99_{\pm0.16}$ & $41.62_{\pm0.13}$ & $49.23_{\pm0.11}$ & $50.64_{\pm0.19}$ & $53.28_{\pm0.15}$ & $50.37_{\pm0.14}$ & $50.20_{\pm0.11}$ & $\mathbf{53.81_{\pm0.43}}$ \\
				CIFAR-100\&DiD2 & $42.01_{\pm0.08}$ & $41.25_{\pm0.14}$ & $48.94_{\pm0.15}$ & $50.52_{\pm0.13}$ & $53.54_{\pm0.16}$ & $50.30_{\pm0.15}$ & $50.47_{\pm0.11}$ & $\mathbf{55.09_{\pm0.32}}$ \\
				\midrule
				\multicolumn{9}{c}{500-clients-2\%}\\
				\midrule
				CIFAR-10\&IID   & $74.20_{\pm0.05}$ & $74.91_{\pm0.14}$ & $79.11_{\pm0.05}$ & $77.79_{\pm0.08}$ & $77.63_{\pm0.08}$ & $76.43_{\pm0.08}$ & $75.73_{\pm0.15}$ & $\mathbf{81.55_{\pm0.08}}$ \\
				CIFAR-10\&DiD1  & $74.66_{\pm0.10}$ & $74.25_{\pm0.06}$ & $77.75_{\pm0.09}$ & $76.51_{\pm0.17}$ & $77.59_{\pm0.17}$ & $75.84_{\pm0.08}$ & $76.25_{\pm0.09}$ & $\mathbf{80.17_{\pm0.17}}$ \\
				CIFAR-10\&DiD2  & $75.06_{\pm0.06}$ & $75.73_{\pm0.07}$ & $79.25_{\pm0.05}$ & $76.61_{\pm0.17}$ & $77.94_{\pm0.11}$ & $76.22_{\pm0.11}$ & $76.39_{\pm0.08}$ & $\mathbf{81.01_{\pm0.10}}$ \\
				CIFAR-100\&IID  & $27.72_{\pm0.17}$ & $27.58_{\pm0.10}$ & $31.65_{\pm0.08}$ & $36.27_{\pm0.12}$ & $37.22_{\pm0.13}$ & $36.08_{\pm0.73}$ & $37.25_{\pm0.06}$ & $\mathbf{47.17_{\pm0.13}}$ \\
				CIFAR-100\&DiD1 & $31.11_{\pm0.04}$ & $30.29_{\pm0.10}$ & $34.09_{\pm0.15}$ & $39.57_{\pm0.15}$ & $37.80_{\pm0.18}$ & $39.20_{\pm0.11}$ & $38.58_{\pm0.10}$ & $\mathbf{44.15_{\pm0.32}}$ \\
				CIFAR-100\&DiD2 & $29.34_{\pm0.07}$ & $29.08_{\pm0.07}$ & $32.83_{\pm0.09}$ & $38.90_{\pm0.15}$ & $38.24_{\pm0.17}$ & $39.32_{\pm0.10}$ & $38.33_{\pm0.12}$ & $\mathbf{45.46_{\pm0.14}}$ \\
				\bottomrule
			\end{tabular}

\end{table*}

\begin{table}[!t]
	\centering
	\caption{\label{tab:kf}Comparing performance across different epoch settings reveals the robustness on increasing local training of FedSSG. As the number of local epochs increases, FedSSG's performance improves. FedSSG has always been optimal method.}
	\begin{tabular}{cccc}
		\toprule
		\multirow{2}{*}{Methods}  & \multicolumn{3}{c}{CIFAR-100/500-clients-2\%/{DiD1}}\\
		&$E=1$& $E=5$ & $E=10$\\
		\midrule
		{FedAvg}  & $29.65_{\pm0.11}$ & $31.11_{\pm0.04}$ &  $29.80_{\pm0.16}$ \\
		{FedProx} & $30.24_{\pm0.06}$ & $30.29_{\pm0.10}$ &  $29.58_{\pm0.05}$ \\
		{SCAFFOLD}& $27.88_{\pm0.06}$ & $34.09_{\pm0.09}$ &  $31.45_{\pm0.08}$ \\
		{FedDyn}  & $37.03_{\pm0.13}$ & $39.57_{\pm0.15}$ &  $39.08_{\pm0.09}$ \\
		{FedDC}   & $34.34_{\pm0.11}$ & $37.80_{\pm0.18}$ &  $40.44_{\pm0.31}$ \\
		{FedVRA}  & $36.83_{\pm0.05}$ & $39.20_{\pm0.11}$ & $39.21_{\pm0.20}$\\
		{FedDisco}& $36.66_{\pm0.07}$ & $38.58_{\pm0.10}$ & $39.21_{\pm0.26}$\\
		{\textbf{FedSSG}}  & $\mathbf{38.87_{\pm0.17}}$ & $\mathbf{44.15_{\pm0.32}}$ &  $\mathbf{44.16_{\pm0.12}}$ \\
		\bottomrule
	\end{tabular}
    
\end{table}

\section{Convergence}

We have proved in Appendix \ref{appendix:b} that given a non-convex and $L$-Lipschitz smooth local objective function $f_i$, and $f_i$ is $B$-dissimilarity, there exists $L\_ > 0$, such that $\nabla^2 f_i \geq -L\_I$ and $\overline{\alpha} = \alpha-L\_ > 0$. The global objective function $f$ satisfies: 
\begin{equation}  \label{eqn:18}
	\mathbb{E}_{S^t}f(\omega^{t+1})\leq f(\omega^{t})-\frac{p||\nabla f(\omega^r)||^2}{2},
\end{equation}
where $p= (1-\psi B) (\frac{2}{\alpha} ) - \frac{B(1+\psi)\sqrt{2}}{\sqrt{K}} (\frac{2}{\overline{\alpha}} ) - L(1+B)\psi (\frac{2}{\alpha \overline{\alpha}} ) - L B^2 (1+\psi)^2 (\frac{1}{2} + \frac{2\sqrt{2K} + 2}{K}) (\frac{2}{\overline{\alpha}^2}) > 0$, and $S^t$ is selected clients in $t$-th round, means that achieving a rapid $\mathcal{O}(1/T)$ convergence rate under smooth convex conditions.

\section{Experiments}

\begin{table*}[t!]
	\centering
	\setlength{\tabcolsep}{1.4mm}
			\caption{\label{tab:speedup}The communication rounds of different methods to approach the same target accuracy. We denote the communication round of each method to approach the target accuracy as "R$\#$", the corresponding convergence speedup relative to FedAvg as "S$\uparrow$". We use "$>T$" sign to represent that the method could not approach the target accuracy. The bolded text highlights the best performance in the corresponding experimental setting. All data is accurate to two decimal places.}
			\begin{tabular}{ccccccccccccccccc}
				\toprule
				\multirow{2}{*}{Distrib.} & \multicolumn{2}{c}{FedAvg} & \multicolumn{2}{c}{FedProx} & \multicolumn{2}{c}{SCAFFOLD} & \multicolumn{2}{c}{FedDyn} & \multicolumn{2}{c}{FedDC} & \multicolumn{2}{c}{FedVRA} & \multicolumn{2}{c}{FedDisco} & \multicolumn{2}{c}{\textbf{FedSSG}}\\
				& R$\#$ & S$\uparrow$ & R$\#$ & S$\uparrow$ & R$\#$ & S$\uparrow$ & R$\#$ & S$\uparrow$ & R$\#$ & S$\uparrow$ & R$\#$ & S$\uparrow$ & R$\#$ & S$\uparrow$ & R$\#$ & S$\uparrow$\\
				\midrule
				\multicolumn{17}{c}{MNIST\&100-clients-15\%\&Target Accuracy=97\%}\\
				\midrule
				IID  & 45 & $1.00\times$ & 45 & $1.00\times$ & 26 & $1.73\times$ & 28 & $1.61\times$ & 25 & $1.80\times$ & 27 & $1.67\times$ & 29 & $1.55\times$ & \textbf{22} & $\mathbf{2.05\times}$ \\
				DiD1 & 76 & $1.00\times$ & 79 & $0.96\times$ & 35 & $2.17\times$ & 39 & $1.95\times$ & 31 & $2.45\times$ & 35 & $2.17\times$ & 37 & $2.05\times$ & \textbf{24} & $\mathbf{3.17\times}$\\
				DiD2 & 61 & $1.00\times$ & 62 & $0.98\times$ & 34 & $1.79\times$ & 35 & $1.74\times$ & 31 & $1.97\times$ & 33 & $1.97\times$ & 32 & $1.91\times$ & \textbf{23} & $\mathbf{2.65\times}$\\
                \midrule
				\multicolumn{17}{c}{EMNIST-L\&100-clients-15\%\&Target Accuracy=93\%}\\
				\midrule
				IID  & 46 & $1.00\times$ & 44 & $1.05\times$ & 27 & $1.70\times$ & 34 & $1.35\times$ & 25 & $1.84\times$ & 36 & $1.28\times$ & 34 & $1.35\times$ & \textbf{24} & $\mathbf{1.92\times}$ \\
				DiD1 & 94 & $1.00\times$ & 93 & $1.01\times$ & 43 & $2.19\times$ & 48 & $1.96\times$ & 37 & $2.54\times$ & 49 & $1.92\times$ & 44 & $2.14\times$ & \textbf{31} & $\mathbf{3.03\times}$\\
				DiD2 & 66 & $1.00\times$ & 71 & $0.93\times$ & 35 & $1.89\times$ & 39 & $1.69\times$ & 32 & $2.06\times$ & 37 & $1.78\times$ & 38 & $1.74\times$ & \textbf{29} & $\mathbf{2.28\times}$\\
				\midrule
				\multicolumn{17}{c}{CIFAR-10\&100-clients-10\%\&Target Accuracy=80\%}\\
				\midrule
				IID & 274 & $1.00\times$ & 284 & $0.96\times$ & 139 & $1.97\times$ & $160$ & $1.71\times$ & 138 & $1.99\times$ & 152 & $1.80\times$ & 142 & $1.93\times$ & \textbf{112} & $\mathbf{2.45\times}$ \\
				DiD1& $>$1000 & $1.00\times$ & 977 & $1.02\times$ & 253 & $3.95\times$ & 261 & $3.83\times$ & 207 & $4.83\times$ & 199 & $5.03\times$ & 282 & $3.55\times$ & \textbf{174} & $\mathbf{5.75\times}$ \\
				DiD2& 416 & $1.00\times$ & 404 & $1.03\times$ & 184 & $2.26\times$ & 185 & $2.25\times$ & 153 & $2.72\times$ & 163 & $2.55\times$ & 178 & $2.34\times$ & \textbf{134} & $\mathbf{3.10\times}$ \\
				\midrule
				\multicolumn{17}{c}{CIFAR-10\&500-clients-2\%\&Target Accuracy=75\%}\\
				\midrule
				IID & $>$1500 & $1.00\times$ & 1456 & $1.03\times$ & \textbf{533} & $\mathbf{2.81\times}$ & 879 & $1.71\times$ & 707 & $2.12\times$ & 1065 & $1.41\times$ & 1049 & $1.43\times$ & 571 & $2.63\times$ \\
				DiD1& $>$1500 & $1.00\times$ & $>$1500 & $1.00\times$ & 679 & $2.21\times$ & 838 & $1.79\times$ & 699 & $2.15\times$ & 880 & $1.70\times$ & 880 & $1.70\times$ & \textbf{544} & $\mathbf{2.76\times}$ \\
				DiD2 & 1285 & $1.00\times$ & 1164 & $1.10\times$ & 625 & $2.06\times$ & 888 & $1.45\times$ & 686 & $1.87\times$ & 910 & $1.41\times$ & 848 & $1.52\times$ & \textbf{515} & $\mathbf{2.50\times}$\\
				\midrule
				\multicolumn{17}{c}{CIFAR-100\&100-clients-10\%\&Target Accuracy=42\%}\\
				\midrule
				IID & $>$1000 & $1.00\times$ & $>$1000 & $1.00\times$ & 177 & $5.65\times$ & 190 & $5.26\times$ & 140 & $7.14\times$ & 198 & $5.05\times$ & 189 & $5.29\times$ & \textbf{120} & $\mathbf{8.33\times}$\\
				DiD1& 721 & $1.00\times$ & $>$1000 & $0.72\times$ & 151 & $4.77\times$ & 175 & $4.12\times$ & 143 & $5.04\times$ & 156 & $4.62\times$ & 192 & $3.76\times$ & \textbf{136} & $\mathbf{5.30\times}$\\
				DiD2& 790 & $1.00\times$ & $>$1000 & $0.79\times$ & 153 & $5.16\times$ & 180 & $4.39\times$ & 140 & $5.64\times$ & 157 & $5.03\times$ & 185 & $4.27\times$ & \textbf{121} & $\mathbf{6.53\times}$\\
				\midrule
				\multicolumn{17}{c}{CIFAR-100\&500-clients-2\%\&Target Accuracy=30\%}\\
				\midrule
				IID & $>$1500 & $1.00\times$ & $>$1500 & $1.00\times$ & 904 & $1.66\times$ & 659 & $2.29\times$ & 699 & $2.15\times$ & 706 & $2.12\times$ & 651 & $2.30\times$ & \textbf{571} & $\mathbf{2.63\times}$\\
				DiD1& 806 & $1.00\times$ & 805 & $1.00\times$ & 592 & $1.36\times$ & 549 & $1.47\times$ & 629 & $1.28\times$ & 559 & $1.44\times$ & 582 & $1.38\times$ & \textbf{497} & $\mathbf{1.62\times}$\\
				DiD2& $>$1500 & $1.00\times$ & $>$1500 & $1.00\times$ & 671 & $2.24\times$ & 593 & $2.53\times$ & 610 & $2.46\times$ & 598 & $2.51\times$ & 612 & $2.45\times$ & \textbf{508} & $\mathbf{2.95\times}$\\
				\bottomrule
			\end{tabular}

\end{table*}

\subsection{Setup}

\textbf{Datsets} To ensure the validity of the experiments, we complete our experiments on four widely-acknowledged real-world datasets as follow: MNIST \cite{mnist}, EMNIST-L \cite{emnistl}, CIFAR-10 and CIFAR-100 \cite{cifar}. For 100-clients experiments, we run 300 communication rounds for MNIST and EMNIST-L and 1000 communication rounds for CIFAR-10 and CIFAR-100; For 500-clients experiments, we run 600 communication rounds for EMNIST-L and 1500 communication rounds for CIFAR-10 and CIFAR-100. To simulate the real-world effects of non-iid data, we set the label ratios follow the Dirichlet distribution. In subsection Results, we use DiD1 and DiD2 to present the Dirichlet parameters are 0.3 and 0.6 respectively.

\textbf{Baselines.} To demonstrate the effectiveness of FedSSG, we compare it with several widely-recognized and advanced methods, including FedAvg \cite{fedavg}, FedProx \cite{fedprox}, SCAFFOLD \cite{scaffold}, FedDyn \cite{feddyn}, FedDC \cite{feddc}, FedVRA \cite{fedvra} and FedDisco \cite{feddisco}.

\textbf{Hyper-parameters.} Our empirical study already shows broad robustness without per-run tuning, while theory only requires $\alpha$ to stay within a conservative band.

\subsection{Effectiveness \& Robustness}
Table~\ref{tab:performances} shows that \textbf{FedSSG} achieves the highest test accuracy on both CIFAR-10 and CIFAR-100 across most settings, exceeding the Top-2 baseline by \(\approx\) 0.91 points on CIFAR-10 and \(\approx\)2.66 points on CIFAR-100 on average (rounded to two decimals). 
The gains persist under strong heterogeneity: across the nine combinations of client-participation rates (100/15\%, 100/10\%, 500/2\%) and data splits (IID, DiD1, DiD2), FedSSG consistently outperforms or ties the best baseline, with the largest margins in the challenging 500-clients–2\% regime. 
Figure~\ref{fig:kfc} further confirms \emph{consistency contraction}: the local/global loss deltas concentrate closer to the origin for FedSSG than for baselines, indicating reduced client drift and a smaller local–global gap. 
Robustness to local training depth is evidenced by Table~\ref{tab:kf}: FedSSG remains optimal as the local epoch \(E\) increases (e.g., \(E\!=\!1,5,10\)), whereas dynamic/alignment baselines (e.g., FedDyn, FedDC) fluctuate—supporting that expectation-gated alignment stabilizes late-phase training under non-IID and partial participation. These observations align with prior findings shown in Appendix \ref{appendix:a} on drift correction and dynamic regularization in FL.

\subsection{Effiency}
We evaluate convergence by fixing a target accuracy per setting and reporting speedup versus FedAvg. As summarized in Table~\ref{tab:speedup}, \textbf{FedSSG} attains the best (or tied-best) convergence time in nearly all conditions, with an average \(\mathbf{4.48\times}\) speedup over CIFAR-10/100. Notably, in CIFAR-100/100-clients–15\%/IID, FedSSG reaches the target in \(\mathbf{10\times}\) fewer rounds than FedAvg. 
Because the target is set at the highest level barely attainable by FedAvg, the comparison is conservative; yet FedSSG still delivers consistent acceleration under stronger heterogeneity and reduced participation, indicating that expectation-gated phase both improves optimization efficiency and preserves stability—complementary to the literature on proximal/control-variate and dynamic-regularization methods. These highlighting FedSSG's effiency in non-iid settings.

\clearpage
\bibliographystyle{IEEEbib}
\bibliography{refs}

\clearpage

\appendix
\input{appendixA}

\clearpage
\input{appendixB}

\end{document}

%% file: appendixA.tex
\appendix

\section{More Details of Experiment Results}\label{appendix:a}
Our experiments are realized on the real-world datasets including MNIST, EMNIST-L, CIFAR10 and CIFAR100. We utilize comparison experiments with FedAvg, FedProx, SCAFFOLD, FedDyn, FedDC, FedVRA and FedDisco methods as baselines to prove the superior performance of FedSSG. We use Dirichlet distribution with parameters as 0.3 (DiD1), 0.6 (DiD2) and 1 (IID) to simulate non-iid scenarios in the real world. The heatmap of different Dirichlet distribution is shown as Figure \ref{fig:heatmap}. Meanwhile, we run experiments on different participation rate, including 15\% on 100 clients, 10\% on 100 clients and 2\% on 500 clients, to prove the effectiveness of FedSSG algorithm in the case of non-iid. To ensure the fairness and effectiveness of experiments, we set the hyper-parameters $\alpha$ as text which make the performance of each methods best. The detailed descriptions of experiment settings of each datasets are shown in the following.
\subsection{Datasets, Models and Hyper-parameters}

\subsubsection{MNIST\&EMNIST-L} As widely accepted real-world datasets, MNIST is a 10-categorical handwritten-classification dataset and EMNIST-L is a 26-categorical characters-classification dataset. We set the sample size as $\{1\times28\times28\}$. We use a fully-connected network (FCN) as the classification model for these two datasets. For MNIST, we set training set as 60000 and test set as 10000, and use two hidden layers with 200 neurons for classification; For EMNIST-L, we set training set as 48000 and test set as 8000, and use one hidden layer with 100 neurons for classification.

\subsubsection{CIFAR10\&CIFAR100} CIFAR10 and CIFAR100 are both more challenging computer visual classification datasets than MNIST and EMNIST-L, which have 10 or 100 categories respectively. We set the sample size to $\{3\times32\times32\}$ and split the dataset into 50000 training images and 10000 test images. We use the network with CNN-based structure, including two conventional layers with 64 of $5\times5$ convolution kernels, to realize images classification.

\subsubsection{Hyper-parameters} All training hyper-parameters in our experiments for different datasets are set as follow: the batch size is set as 50; the number of epochs in each client is set as 5; the initial learning rate is set as 0.1; the learning rate decay per round is set as 0.998; and the weight decay is set as 0.001. For FedSSG, we search $\alpha$ as following space: (1) For MNIST, we search $\{0.01,0.02,0.03,0.04,0.05,0.06\}$; (2) For EMNIST-L, we search $\{0.01,0.02,0.05,0.06,0.1\}$ with 100-clients and $\{0.1,0.2,0.5,0.6,0.7\}$ with 500-clients; (3) For CIFAR-10 and CIFAR-100, we search $\{0.01,0.02,0.03,0.05,0.1\}$ with 100-clients and $\{0.01,0.02,0.5,0.1,0.2\}$. We set $\alpha$ as following, where introduced $\alpha$ on different datasets in terms of \{IID,DiD1,DiD2\}: For MNIST with all 100-clients setting, $\alpha=\{0.03,0.05,0.04\}$. For EMNIST-L with all 100-clients setting, $\alpha=\{0.05,0.06,0.06\}$; with all 500-clients setting, $\alpha=\{0.7,0.6,0.6\}$. For CIFAR-10 with all 100-clients setting, $\alpha=0.05$; with all 500-clients setting, $\alpha=0.2$. For CIFAR-100 with all 100-clients setting, $\alpha=0.02$; with all 500-clients setting, $\alpha=0.2$. We display the performance and loss of different hyper-parameters in Figure \ref{fig:MP15H}-\ref{fig:C100P2H}. And we will discuss their performance in the later sections. 
For the hyper-parameters in other baselines, we set FedDC's $\alpha=0.1$ for MNIST and EMNIST-L with 100-clients; set $\alpha=0.2$ for E-MNIST with 500-clients; set $\alpha=0.01$ for CIFAR10 and CIFAR-100 with 100-clients and set $\alpha=0.05$ with 500-clients. We set $\alpha=0.01$ in FedDyn, $\mu=0.0001$ and weight decay as $=10^{-5}$ in FedProx. For FedVRA, we set $d_i^{r}=10$, $a_i^{r}=7$ and $\gamma_i=0.1$. For FedDisco, we use FedDyn with Disco to build.
To ensure the rigor of the experiment, the same random seed was used for all random processes in the same setting.
\begin{figure}[htbp]
    \centering
    \subfloat[]{\includegraphics[width=0.23\textwidth]{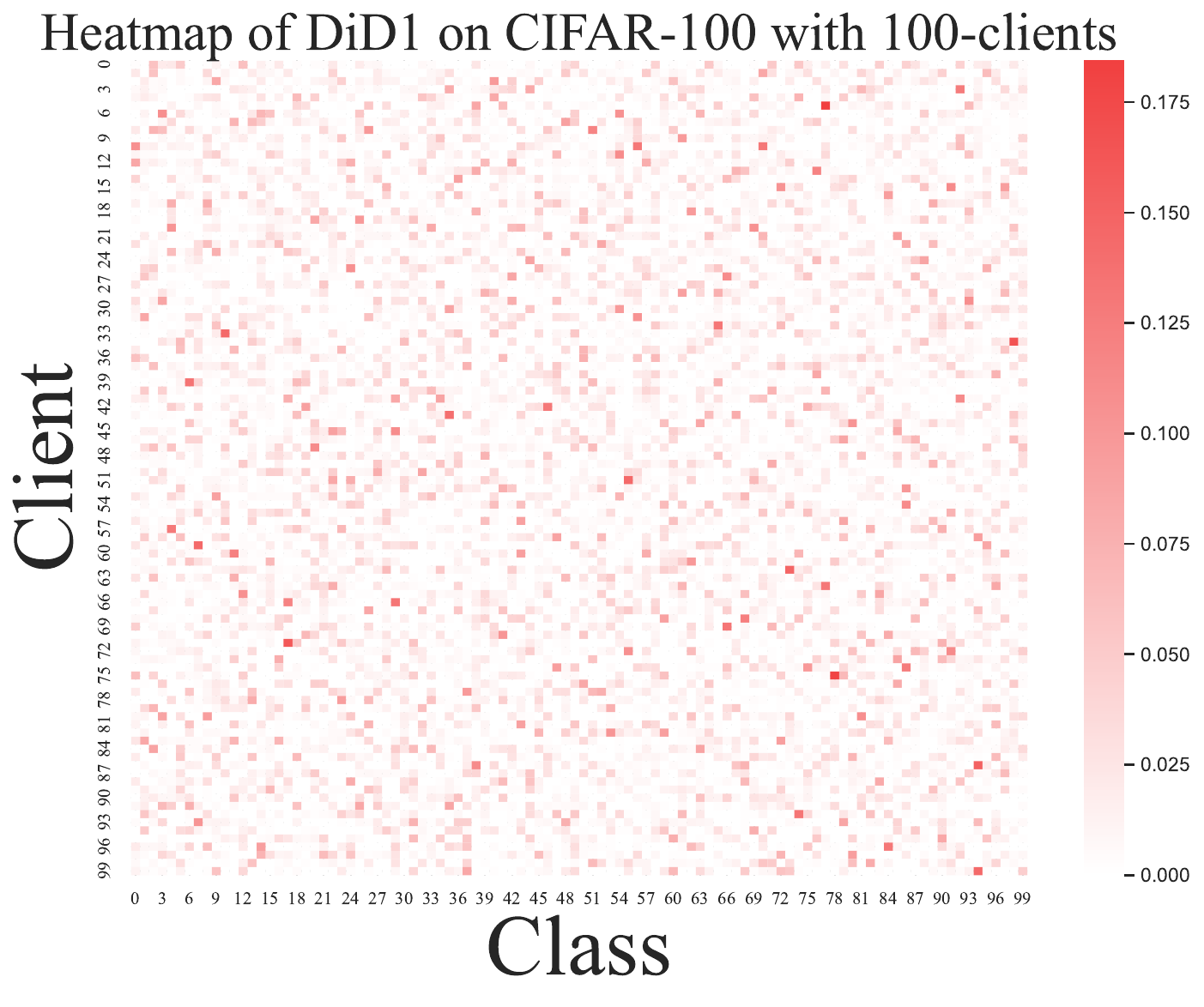}} \hfill
    \subfloat[]{\includegraphics[width=0.23\textwidth]{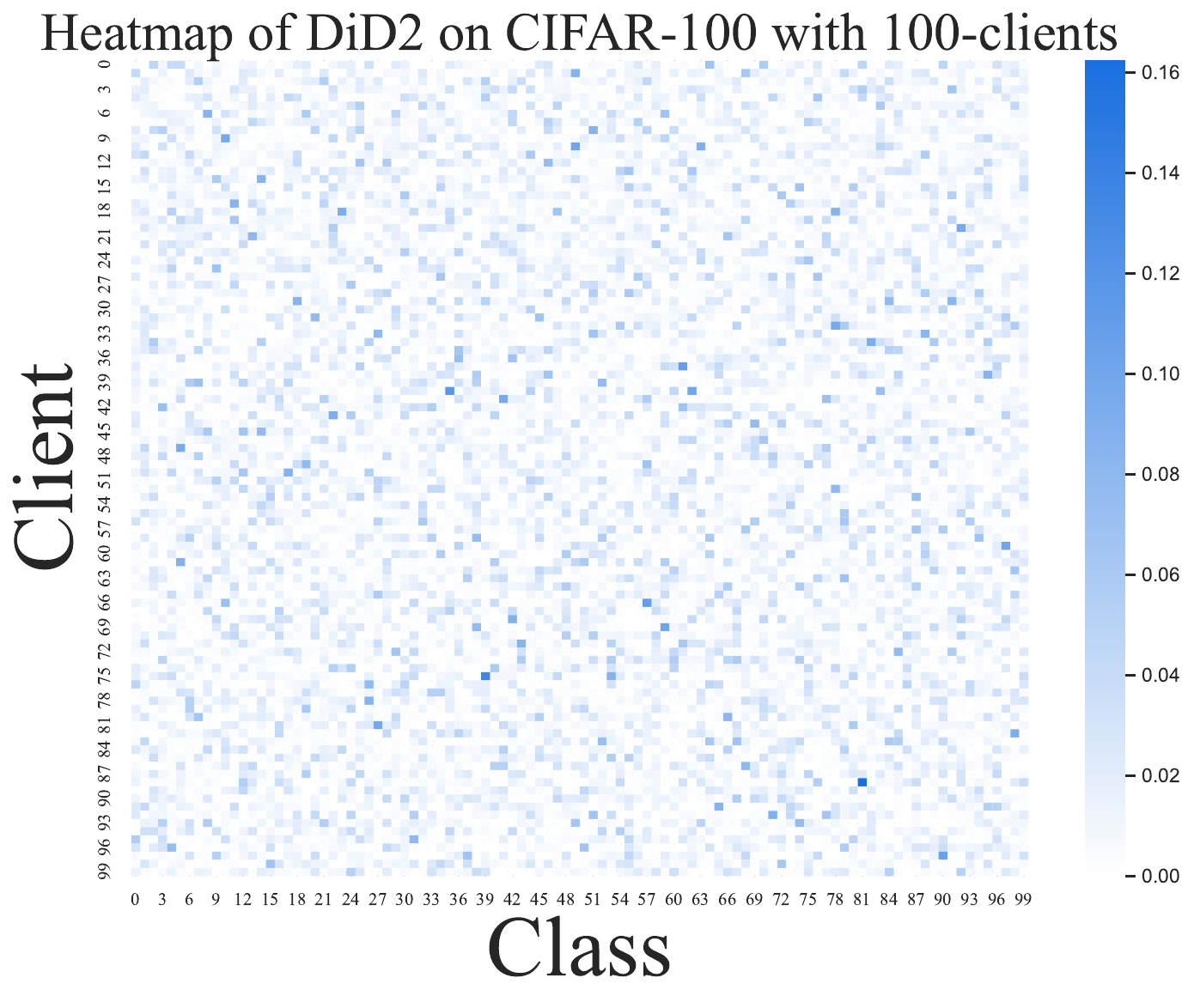}} \hfill
    \caption{The heatmap of Dirichlet distribution setting in experiments.}
    \label{fig:heatmap}
\end{figure}

\subsection{Experiments Analysis}

In the section, we will specifically discuss and analyze FedSSG's performances with accuracy, convergence time and hyper-parameter's comparison on each datasets. Table \ref{tab:performancesA}-\ref{tab:speedupA} shows the performances on MNIST and EMNIST-L; Figure \ref{fig:MP15}-\ref{fig:C100P2} shows the learning curves on each datasets; and Figure \ref{fig:MP15H}-\ref{fig:C100P2H} shows the hyper-parameter's comparison on each datasets. All experiments' settings follow as the text to be consistent. Specifically, we use "mean$\pm$standard" to compare the performance; since we took the target convergence frequency of FedAvg as the benchmark for speed up, we set the target accuracy to a level barely attainable by FedAvg to ensure the effectiveness and fairness of our comparison of convergence time; the learning curves show the tendency of FedSSG and baselines; and learning curves and loss curves display difference of hyper-parameters. We will complete the discussion in the order of MNIST, EMNIST-L, CIFAR-10, CIFAR-100. 


\subsubsection{Discussion on MNIST and EMNIST-L}

\begin{table*}[h!]
	\centering
    			\caption{\label{tab:performancesA}Test accuracy (\%) comparison among baselines and our proposed method on the MNIST/EMNIST-L dataset from last 50 tests, presented in the format of "mean$\pm$standard". The bolded text highlights the best performance in the corresponding experimental setting. All data is accurate to two decimal places.}
	\setlength{\tabcolsep}{1mm}
	\small
			\begin{tabular}{ccccccccc}
				\toprule
				Settings & FedAvg & FedProx & SCAFFOLD & FedDyn & FedDC & FedVRA & FedDisco & FedSSG\\
				\midrule
				\multicolumn{9}{c}{100-clients-15\%}\\
				\midrule
				MNIST\&IID   & $98.08_{\pm0.03}$ & $98.06_{\pm0.02}$ & $98.39_{\pm0.02}$ & $98.38_{\pm0.02}$ & $98.38_{\pm0.02}$ & $98.28_{\pm0.04}$ & $98.23_{\pm0.02}$ & $\mathbf{98.43_{\pm0.04}}$	 \\
				MNIST\&DiD1  & $97.70_{\pm0.03}$ & $97.70_{\pm0.04}$ & $98.32_{\pm0.03}$ & $98.10_{\pm0.04}$ & $98.29_{\pm0.03}$ & $98.10_{\pm0.02}$ & $98.08_{\pm0.02}$ & $\mathbf{98.42_{\pm0.04}}$ \\
				MNIST\&DiD2  & $97.84_{\pm0.02}$ & $97.81_{\pm0.02}$ & $98.41_{\pm0.02}$ & $98.12_{\pm0.03}$ & $98.40_{\pm0.02}$ & $98.17_{\pm0.02}$ & $98.13_{\pm0.03}$ & $\mathbf{98.46_{\pm0.02}}$ \\
				EMNIST-L\&IID   & $94.77_{\pm0.05}$ & $94.73_{\pm0.07}$ & $\mathbf{95.44_{\pm0.05}}$ & $95.15_{\pm0.08}$ & $95.32_{\pm0.07}$ & $94.81_{\pm0.06}$ & $95.05_{\pm0.08}$ & $95.32_{\pm0.06}$	 \\
				EMNIST-L\&DiD1  & $93.94_{\pm0.07}$ & $93.73_{\pm0.23}$ & $94.85_{\pm0.06}$ & $94.57_{\pm0.06}$ & $95.01_{\pm0.07}$ & $94.66_{\pm0.05}$ & $94.65_{\pm0.06}$ & $\mathbf{95.05_{\pm0.08}}$	 \\
				EMNIST-L\&DiD2  & $94.38_{\pm0.04}$ & $94.47_{\pm0.07}$ & $95.23_{\pm0.07}$ & $94.71_{\pm0.06}$ & $95.11_{\pm0.09}$ & $94.80_{\pm0.10}$ & $94.87_{\pm0.06}$ & $\mathbf{95.31_{\pm0.07}}$	 \\
				\midrule
				\multicolumn{9}{c}{100-clients-10\%}\\
				\midrule
				MNIST\&IID   & $98.07_{\pm0.03}$ & $98.04_{\pm0.04}$ & $98.32_{\pm0.01}$ & $98.24_{\pm0.04}$ & $98.36_{\pm0.03}$ & $98.26_{\pm0.02}$ & $98.21_{\pm0.02}$ & $\mathbf{98.47_{\pm0.02}}$ \\
				MNIST\&DiD1  & $97.70_{\pm0.03}$ & $97.69_{\pm0.03}$ & $98.31_{\pm0.03}$ & $98.09_{\pm0.03}$ & $\mathbf{98.37_{\pm0.04}}$ & $98.08_{\pm0.01}$ & $98.10_{\pm0.02}$ & $98.34_{\pm0.04}$ \\
				MNIST\&DiD2  & $97.94_{\pm0.05}$ & $97.94_{\pm0.04}$ & $98.34_{\pm0.02}$ & $98.18_{\pm0.02}$ & $98.38_{\pm0.03}$ & $98.14_{\pm0.01}$ & $98.18_{\pm0.02}$ & $\mathbf{98.39_{\pm0.03}}$ \\
				EMNIST-L\&IID   & $94.54_{\pm0.04}$ & $94.65_{\pm0.04}$ & $\mathbf{95.49_{\pm0.07}}$ & $94.79_{\pm0.08}$ & $95.05_{\pm0.08}$ & $94.98_{\pm0.06}$ & $94.74_{\pm0.06}$ & $95.30_{\pm0.07}$	 \\
				EMNIST-L\&DiD1  & $94.01_{\pm0.04}$ & $94.03_{\pm0.10}$ & $\mathbf{94.84_{\pm0.09}}$ & $94.76_{\pm0.10}$ & $94.65_{\pm0.08}$ & $94.34_{\pm0.05}$ & $94.36_{\pm0.08}$ & $94.82_{\pm0.18}$	 \\
				EMNIST-L\&DiD2  & $94.23_{\pm0.09}$ & $94.31_{\pm0.05}$ & $95.14_{\pm0.05}$ & $94.65_{\pm0.10}$ & $\mathbf{95.19_{\pm0.09}}$ & $94.78_{\pm0.07}$ & $94.72_{\pm0.08}$ & $95.18_{\pm0.10}$	 \\
				\midrule
				\multicolumn{9}{c}{500-clients-2\%}\\
				\midrule
				EMNIST-L\&IID   & $93.72_{\pm0.04}$ & $93.78_{\pm0.02}$ & $94.69_{\pm0.04}$ & $93.93_{\pm0.03}$ & $94.25_{\pm0.04}$ & $94.19_{\pm0.04}$ & $94.26_{\pm0.02}$ & $\mathbf{94.78_{\pm0.13}}$	 \\
				EMNIST-L\&DiD1  & $93.20_{\pm0.02}$ & $93.26_{\pm0.03}$ & $94.25_{\pm0.02}$ & $93.52_{\pm0.06}$ & $94.16_{\pm0.04}$ & $93.77_{\pm0.04}$ & $93.70_{\pm0.02}$ & $\mathbf{94.57_{\pm0.05}}$	 \\
				EMNIST-L\&DiD2  & $93.45_{\pm0.04}$ & $93.17_{\pm0.16}$ & $94.55_{\pm0.04}$ & $93.93_{\pm0.06}$ & $93.90_{\pm0.05}$ & $93.89_{\pm0.07}$ & $93.80_{\pm0.07}$ & $\mathbf{94.77_{\pm0.05}}$	 \\
				\bottomrule
			\end{tabular}

\end{table*}

\begin{table*}[h!]
	\centering
        			\caption{\label{tab:speedupA}The communication rounds of different methods to approach the same target accuracy. We denote the communication round of each method to approach the target accuracy as "R$\#$", the corresponding convergence speedup relative to FedAvg as "S$\uparrow$". We use "$>T$" sign to represent that the method could not approach the target accuracy. The bolded text highlights the best performance in the corresponding experimental setting. To make it easier to compare performance, we use underline to highlight the second highest accuracy. All data is accurate to two decimal places.}
	\setlength{\tabcolsep}{1mm}
	\small
			\begin{tabular}{ccccccccccccccccc}
				\toprule
				\multirow{2}{*}{Distribution} & \multicolumn{2}{c}{FedAvg} & \multicolumn{2}{c}{FedProx} & \multicolumn{2}{c}{SCAFFOLD} & \multicolumn{2}{c}{FedDyn} & \multicolumn{2}{c}{FedDC} & \multicolumn{2}{c}{FedVRA} & \multicolumn{2}{c}{FedDisco} & \multicolumn{2}{c}{FedSSG}\\
				& R$\#$ & S$\uparrow$ & R$\#$ & S$\uparrow$ & R$\#$ & S$\uparrow$ & R$\#$ & S$\uparrow$ & R$\#$ & S$\uparrow$ & R$\#$ & S$\uparrow$ & R$\#$ & S$\uparrow$ & R$\#$ & S$\uparrow$\\
				\midrule
				\multicolumn{17}{c}{MNIST\&100-clients-15\%\&Target Accuracy=97\%}\\
				\midrule
				IID  & 45 & $1.00\times$ & 45 & $1.00\times$ & 26 & $1.73\times$ & 28 & $1.61\times$ & 25 & $1.80\times$ & 27 & $1.67\times$ & 29 & $1.55\times$ & \textbf{22} & $\mathbf{2.05\times}$ \\
				DiD1 & 76 & $1.00\times$ & 79 & $0.96\times$ & 35 & $2.17\times$ & 39 & $1.95\times$ & 31 & $2.45\times$ & 35 & $2.17\times$ & 37 & $2.05\times$ & \textbf{24} & $\mathbf{3.17\times}$\\
				DiD2 & 61 & $1.00\times$ & 62 & $0.98\times$ & 34 & $1.79\times$ & 35 & $1.74\times$ & 31 & $1.97\times$ & 33 & $1.97\times$ & 32 & $1.91\times$ & \textbf{23} & $\mathbf{2.65\times}$\\
				\midrule
				\multicolumn{17}{c}{MNIST\&100-clients-10\%\&Target Accuracy=97\%}\\
				\midrule
				IID  & 54 & $1.00\times$ & 55 & $0.98\times$ & 31 & $1.74\times$ & 31 & $1.74\times$ & 28 & $1.93\times$ & 33 & $1.64\times$ & 33 & $1.64\times$ & \textbf{24} & $\mathbf{2.25\times}$ \\
				DiD1 & 83 & $1.00\times$ & 76 & $1.09\times$ & 42 & $1.98\times$ & 43 & $1.93\times$ & 38 & $2.18\times$ & 42 & $1.98\times$ & 40 & $2.08\times$ & \textbf{33} & $\mathbf{2.52\times}$\\
				DiD2 & 67 & $1.00\times$ & 66 & $1.02\times$ & 35 & $1.91\times$ & 38 & $1.76\times$ & 33 & $2.03\times$ & 41 & $1.63\times$ & 40 & $1.68\times$ & \textbf{27} & $\mathbf{2.48\times}$\\
				\midrule
				\multicolumn{17}{c}{EMNIST-L\&100-clients-15\%\&Target Accuracy=93\%}\\
				\midrule
				IID  & 46 & $1.00\times$ & 44 & $1.05\times$ & 27 & $1.70\times$ & 34 & $1.35\times$ & 25 & $1.84\times$ & 36 & $1.28\times$ & 34 & $1.35\times$ & \textbf{24} & $\mathbf{1.92\times}$ \\
				DiD1 & 94 & $1.00\times$ & 93 & $1.01\times$ & 43 & $2.19\times$ & 48 & $1.96\times$ & 37 & $2.54\times$ & 49 & $1.92\times$ & 44 & $2.14\times$ & \textbf{31} & $\mathbf{3.03\times}$\\
				DiD2 & 66 & $1.00\times$ & 71 & $0.93\times$ & 35 & $1.89\times$ & 39 & $1.69\times$ & 32 & $2.06\times$ & 37 & $1.78\times$ & 38 & $1.74\times$ & \textbf{29} & $\mathbf{2.28\times}$\\
				\midrule
				\multicolumn{17}{c}{EMNIST-L\&100-clients-10\%\&Target Accuracy=93\%}\\
				\midrule
				IID  & 49 & $1.00\times$ & 54 & $0.91\times$ & 32 & $1.53\times$ & 36 & $1.36\times$ & 30 & $1.63\times$ & 37 & $1.32\times$ & 37 & $1.32\times$ & \textbf{27} & $\mathbf{1.81\times}$ \\
				DiD1 & 104 & $1.00\times$ & 91 & $1.14\times$ & 51 & $2.24\times$ & 57 & $1.82\times$ & 48 & $2.17\times$ & 48 & $2.17\times$ & 51 & $2.04\times$ & \textbf{39} & $\mathbf{2.67\times}$\\
				DiD2 & 74 & $1.00\times$ & 72 & $1.03\times$ & 40 & $1.85\times$ & 41 & $1.80\times$ & 34 & $2.18\times$ & 42 & $1.76\times$ & 41 & $1.80\times$ & \textbf{32} & $\mathbf{2.31\times}$\\
				\midrule
				\multicolumn{17}{c}{EMNIST-L\&500-clients-2\%\&Target Accuracy=93\%}\\
				\midrule
				IID  & 318 & $1.00\times$ & 334 & $0.95\times$ & 185 & $1.72\times$ & 217 & $1.47\times$ & \textbf{176} & $\mathbf{1.81\times}$ & 212 & $1.50\times$ & 198 & $1.61\times$ & 188 & $1.69\times$ \\
				DiD1 & 507 & $1.00\times$ & 498 & $1.02\times$ & 245 & $2.07\times$ & 274 & $1.85\times$ & 228 & $2.22\times$ & 279 & $1.82\times$ & 261 & $1.94\times$ & \textbf{198} & $\mathbf{2.56\times}$\\
				DiD2 & 438 & $1.00\times$ & 448 & $0.98\times$ & 218 & $2.01\times$ & 257 & $1.70\times$ & \textbf{176} & $\mathbf{2.49\times}$ & 274 & $1.60\times$ & 261 & $1.68\times$ & 186 & $2.35\times$\\
				\bottomrule
			\end{tabular}

\end{table*}

As Table \ref{tab:performancesA} shown, FedSSG's performs better than baselines except with 100-clients-10\% on DiD1. Even under unfavorable conditions, FedSSG is also used as the top-2 method, and the gap is only 0.03 points compared with the top-1 method FedDC. FedSSG's top-1 score was still an average of 0.05 points higher than the second-highest baselines on MNIST (which rounded to 2 decimal places). Meanwhile, as Table \ref{tab:speedupA} presented, FedSSG converge to target accuracy faster than all the baselines on MNIST. Although MNIST is a relatively simple dataset where all baselines can achieve good results and converge faster than other datasets, the convergence time's improvement of FedSSG over FedAvg is still more than twice in all cases.

Figure \ref{fig:MP15}-\ref{fig:MP10} display the learning curves of FedSSG and baselines with different settings on MNIST. It is easy to find the existence of inflection pints, shows the effect of phase strategy.
They block the tendency of FedSSG to converge rapidly but also finally make FedSSG achieve outperformance. And Figure \ref{fig:MP15H}-\ref{fig:MP10H}, searching the hyper-parameter space including $\{0.01,0.02,0.03,0.04,0.05,0.06\}$, indicate that there is no obvious normal distribution relationship between the performance of FedSSG and the choice of $\alpha$ on MNIST, and a small $\alpha$ will cause slow convergence and poor effect of FedSSG. The $\alpha$ of FedSSG are sensitive on MNIST, so we have to choose a compact search space. This will guide us adjust FedSSG in relatively simple practical applications like MNIST.

It is worth noting that, FedSSG and FedDC cannot converge to a stationary point on MNIST with 500-clients-2\% setting. Here are two simple observations: (1) From Table \ref{tab:performancesA} we can find that the "mean$\pm$standard" of FedSSG and FedDC has more intense fluctuation degree; (2) The objection functions of FedSSG and FedDC are relatively more complex than other baselines. Hence, we suppose that the model being overfitted is the reason of this terrible phenomenon.

And as can be seen in Table \ref{tab:performancesA}, although FedSSG has increased significantly as a Top-2 method in EMNIST-L, it performs than baselines over 0.07 points on average (which is rounded to 2 decimal places). 
Meanwhile, as Table \ref{tab:speedupA} shown, FedSSG is the fastest method in the experiments on EMNIST-L. Although the top-1 accuracy of FedSSG showed a slight decline on MNIST and EMNIST-L, it was always superior in the convergence speed.

Figure \ref{fig:EP15}-\ref{fig:EP2} display the learning curves of FedSSG on EMNIST-L. These figures give a significant observation that how fast the FedSSG's convergence is. In the meantime, Memo-corrected points distinctly appear in 500-clients-2\% experiments. FedSSG quickly converges to a higher level and then enters a slower convergence state, and achieves the best performance of convergence in the middle and late stages. And as Figure \ref{fig:EP15H}-\ref{fig:EP2H} shown, FedSSG is still sensitive to $\alpha$. We search the hyper-parameter space including $\{0.01,0.02,0.05,0.06,0.1\}$, and find that there are evident normal distribution between performance and $\alpha$ with 100-clients. This phenomenon is not obvious at a setting of 500-clients-2\%. Meanwhile, with $\alpha$ increasing, FedSSG's performance becomes better but the convergence time becomes longer.



In general, on the MNIST and EMNIST-L datasets, which have low classification difficulty, FedSSG performed better on average but did not dominate on CIFAR-10 and CIFAR-100. We believe this is mainly due to model overfitting. However, when analyzing the optimal performance of the method, FedSSG may still be the only method that can achieve the global optimal performance point under the corresponding setting, despite its relatively poor performance in Table \ref{tab:performancesA}. Additionally, we observed that SCAFFOLD performed very well on these simpler datasets. However, considering the convergence speed associated with SCAFFOLD, FedSSG remains a better choice.

\subsubsection{Discussion on CIFAR-10 and CIFAR-100}


Although we have thoroughly analyzed the performance on CIFAR-10 and CIFAR-100 in the main body, there are still some details worth adding in the appendix. These details demonstrate the powerful generality of FedSSG and its potential to enhance FL performance in real-world situations. As mentioned in the text and shown in Figures \ref{fig:C10P2} and \ref{fig:C100P2}, FedSSG performs extraordinarily well compared to all baselines in the 500-clients-2\% setting. This indicates strong performance against non-iid data that closely resembles real-world scenarios. Figures \ref{fig:C10P2H} and \ref{fig:C100P2H} show that when $\alpha$ is set to 0.2, FedSSG's performance is significantly better than with other $\alpha$ values. Considering the search space for $\alpha$ as ${0.01, 0.02, 0.05, 0.1, 0.2}$, it is evident that $\alpha=0.2$ represents the boundary hyperparameter of this search space. This suggests that even in this more extreme setting, FedSSG has substantial room for further improvement and development. Moreover, compared to other client count and engagement settings, we have more search space in this scenario, indicating that FedSSG performs better in more complex and challenging non-IID environments.


\subsubsection{Massive Clients with Low Sample Ratio} Setting experiments with massive clients and low sample ratios makes our experiments more reflective of real-world scenarios. 
To assess the effectiveness of FedSSG in such real-world scenarios, we define a simple equation as follows:
\begin{equation}
	\begin{aligned}
		ACC_{diff}=(AC&C_{Top-1}-ACC_{FedAvg})\\
		&-(ACC_{Top-2}-ACC_{FedAvg})
		\label{eqn:dv}
	\end{aligned}
\end{equation}
where $ACC$ is the highest test accuracy of each method. Using Eqn.\ref{eqn:dv}, we calculate the test accuracy difference between FedSSG (which consistently performs as Top-1) and the Top-2 baselines. We chart these differences for all non-iid settings to verify FedSSG's strong performance with massive clients and low sample ratios. As illustrated in Figure \ref{fig:dv}, FedSSG exhibits superior performance in the more heterogeneous 500-clients-2\% setting compared to both 100-client configurations, even when evaluated under horizontal comparison. This observation underscores the significant potential of FedSSG for real-world applications. Moreover, the pie chart reveals that as client participation heterogeneity intensifies—manifested by a declining participation rate and an increasing number of clients—FedSSG achieves a notable improvement in performance metrics. These results demonstrate that FedSSG is robust and well-suited for scenarios involving large-scale client populations with low sample ratios.

\begin{figure}[t!]
	\centering
	\includegraphics[width=0.45\textwidth]{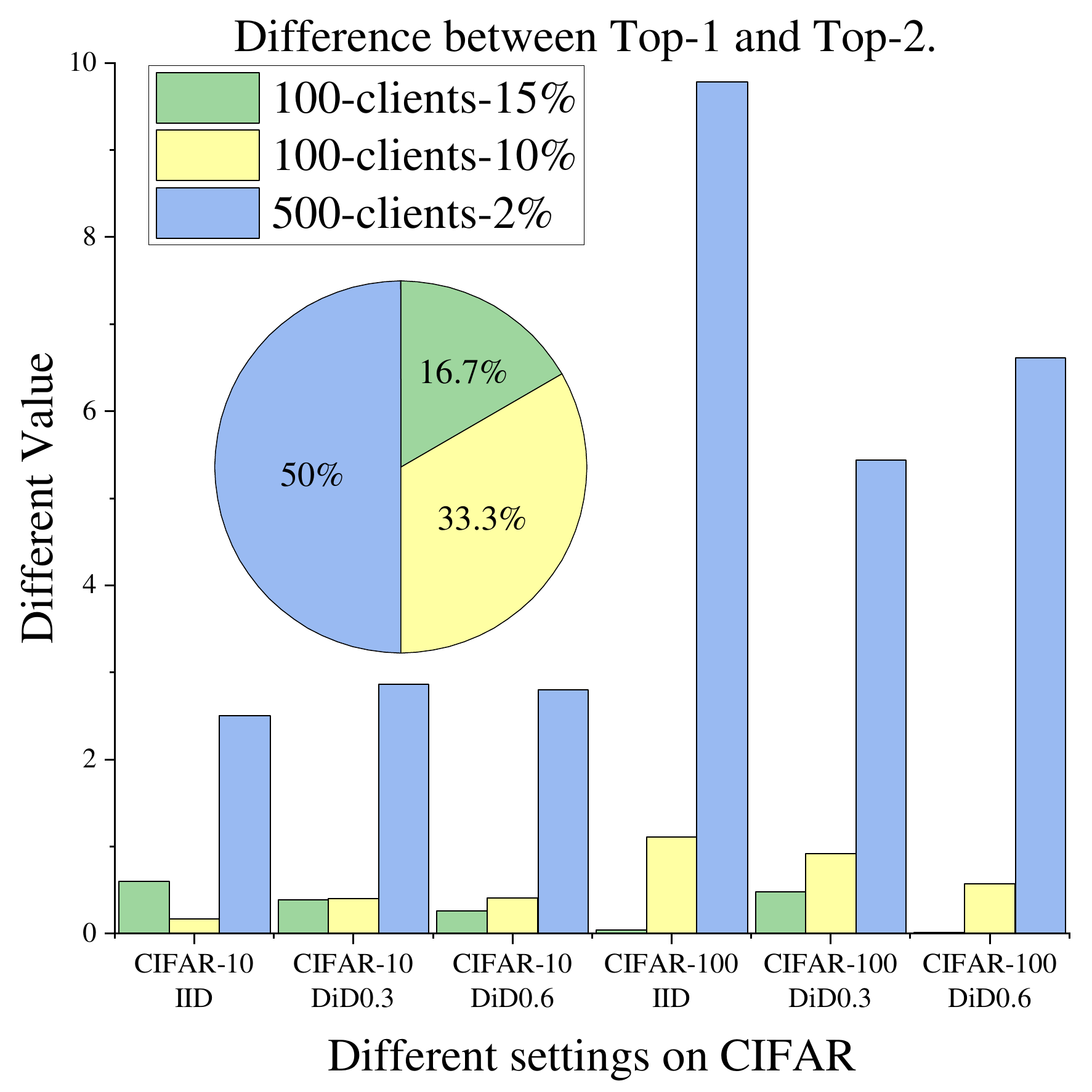}
	\caption{The difference of highest accuracy $ACC_{diff}$ (\%) between FedSSG (always Top-1 method) and Top-2 baselines on different settings in our experiments. Pie chart shows that with the decrease of the participation rate and the increase of the number of clients, the performance of FedSSG increase significantly. FedSSG performs extremely well with massive clients with low sample ratio.}
	\label{fig:dv}
\end{figure}

\begin{table*}[t]
	\centering
	\small
	\begin{tabular}{cccccc}
		\toprule
		\multicolumn{6}{c}{CIFAR-100/500-clients-2\%/DiD1}\\
		Empirical Loss & Penalization & Gradient Correction & Memo-correction Term $\xi_i$ & Accuracy ($E=1$) & Accuracy ($E=5$) \\
		\midrule
		\checkmark &  &  & \checkmark & $36.20_{\pm0.18}$ & $38.61_{\pm0.21}$ \\
		\checkmark & \checkmark &  & \checkmark & $40.74_{\pm0.10}$  & $40.91_{\pm0.25}$ \\
		\checkmark &  & \checkmark & \checkmark & $36.72_{\pm0.13}$ & $37.16_{\pm0.33}$ \\
		\checkmark & \checkmark & \checkmark &  & $36.59_{\pm0.14}$ & $41.88_{\pm0.39}$ \\
		\checkmark & \checkmark & \checkmark & \checkmark & $\mathbf{38.87_{\pm0.17}}$ & $\mathbf{44.15_{\pm0.32}}$ \\
		\midrule
		\multicolumn{4}{c}{FedAvg (Standard)} & $29.65_{\pm0.11}$ & $31.11_{\pm0.04}$ \\
		\multicolumn{4}{c}{FedDyn (Top-2 Performance)} & $37.03_{\pm0.13}$ & $39.57_{\pm0.15}$ \\
		\bottomrule
	\end{tabular}
	\caption{\label{tab:objective}To verify the effectiveness of memo-modification which designed for solving the issues of model drift and knowledge forgetting, comparison explore the effectiveness for each part of FedSSG. Where E means local epoch times.}
\end{table*}	

\subsubsection{Ablation Tests}

To address client drift and mitigate the impact of the server's stochastic sampling, we developed the memo-modification, which includes the stochastic sampling-guided term $\xi_i$ and adjustments to the objective function. In this section, we validate the effectiveness and robustness of these modifications. As shown in Table \ref{tab:objective}, we compared different components of FedSSG's phase strategy, referred to as FedSSGle, FedSSGlp, FedSSGlg, and FedSSGxi. For $E=5$, FedSSGle outperforms FedAvg, and FedSSGxi surpasses the Top-2 baseline, FedDyn, demonstrating the effectiveness of stochastic sampling-guided local drift. Furthermore, FedSSGlp also exceeds the performance of FedDyn, indicating that the penalized term enhances FedSSG and mitigates client drift. Compared to $E=1$, FedSSG shows significant robustness with increased local training. By integrating all phase strategy correction, FedSSG effectively meets its goal of handling non-IID data.

\subsubsection{Discussion on Inflection Points}

\begin{figure}[t]
    \centering
        \setlength{\belowcaptionskip}{10pt}
    \subfloat[]{\includegraphics[width=0.23\textwidth]{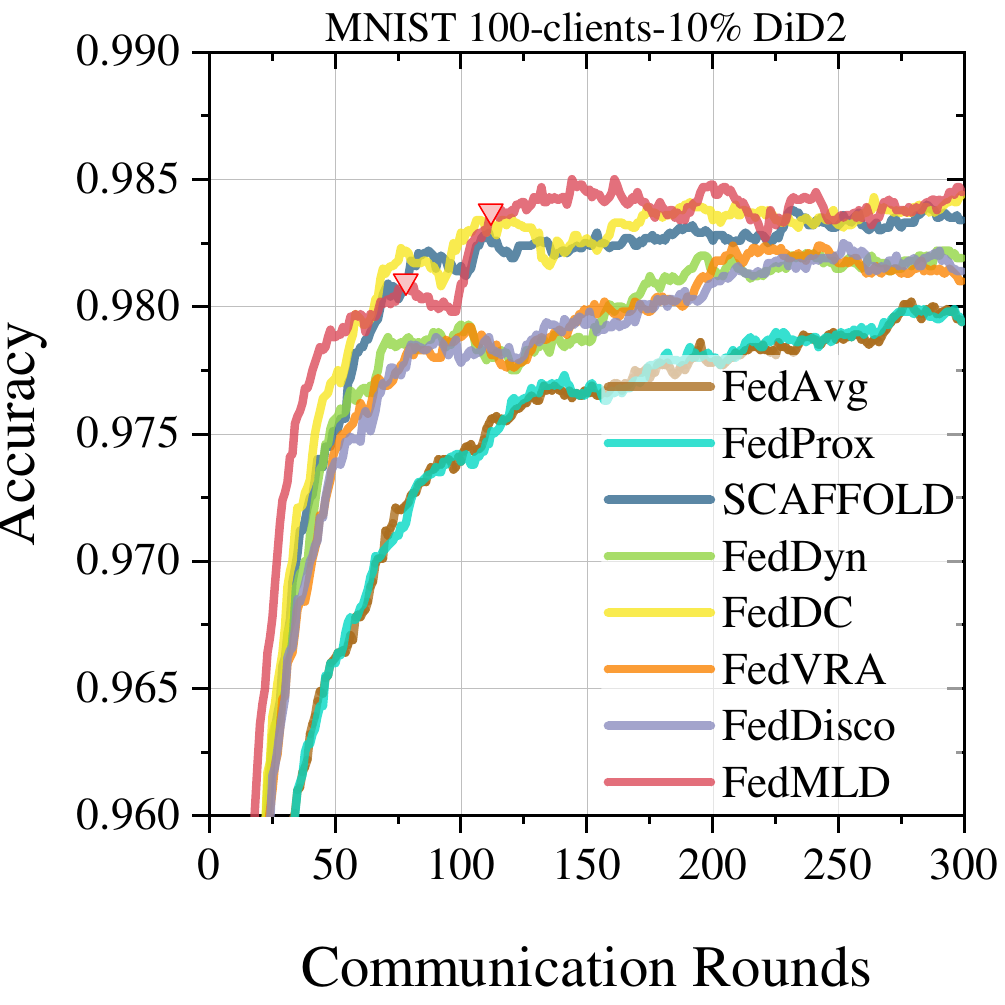}} \hfill
    \subfloat[]{\includegraphics[width=0.23\textwidth]{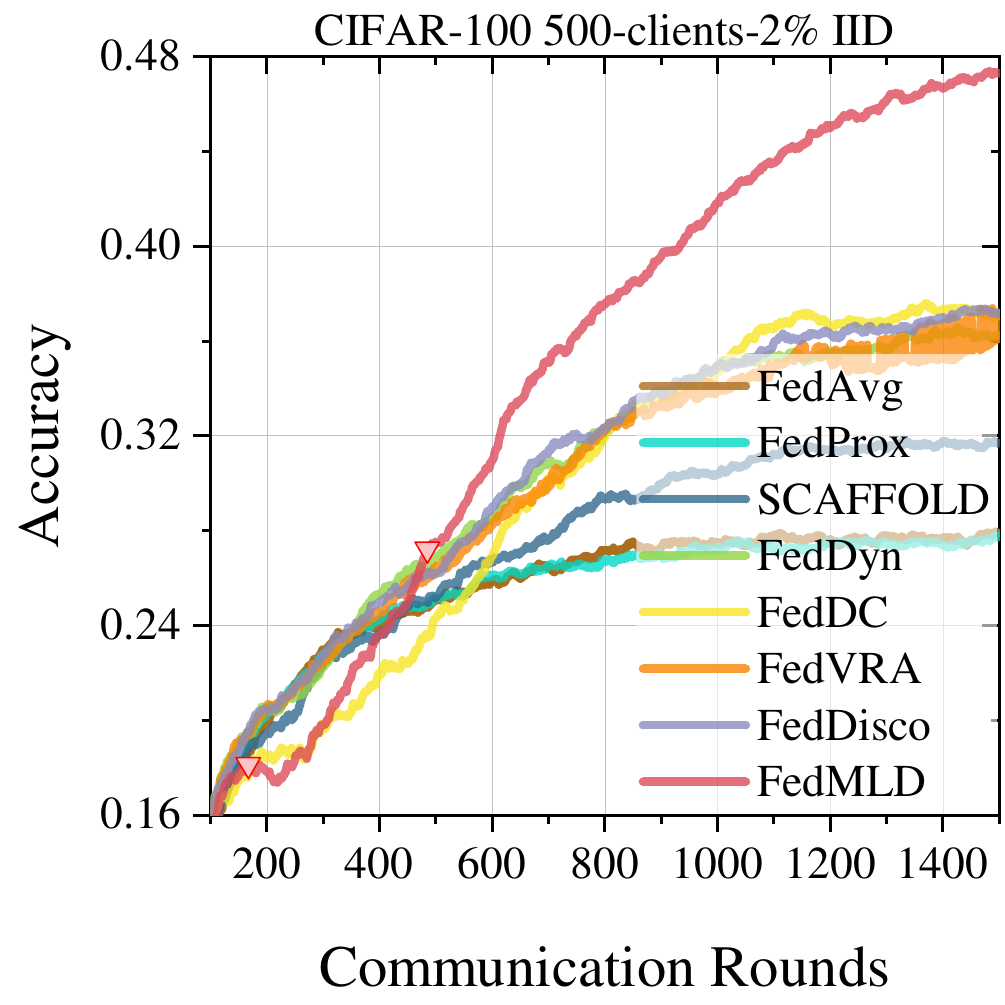}}
    \caption{Triangle sample for inflection points. It shows frequently in experiments.}
    \label{fig:ipoints}
\end{figure}

As can be seen from Table \ref{tab:speedup} and \ref{tab:speedupA}, the FedSSG showed its dominance on convergence speed. Since we took the target convergence frequency of FedAvg as the benchmark for speed up, we set the target accuracy to a level barely attainable by FedAvg to ensure the effectiveness and fairness of our comparison. FedSSG performs powerful fast convergence feature in all experiments except the experiment of 500-clients-2\% on CIFAR-10. 
In the process of exploring the reasons for this exception, we found an interesting phenomenon, which coincides with our previous hypothesis: as Figure \ref{fig:ipoints} shown, there are two inflection points on captured learning curves. We suppose that between inflection points, FedSSG utilize the penalized term and stochastic sampling-guided local drifts to rectify clients' phased knowledge. Specifically, as Eqn.\ref{eqn:8} and Eqn.\ref{eqn:9}, we supposed that the penalized term of FedSSG will correct the objective function in different direction between the initial stage and the final stage in each client. At the same time, stochastic sampling-guided local drifts will accumulate as the number of local communication rounds increases. However, the global communication round is not synchronized with each client. Since the global updates are synthesized from local updates and the server randomly select the clients, we cannot specify where inflection points will arise and how many inflection points will appear. These inflection points affect the convergence speed of FedSSG to a certain target accuracy to some extent, but make its final performance more effective. It also proves that what we think of as memorability does exist. Meanwhile, inflection points appear more obviously in more complicated experimental settings, which also shows that FedSSG has achieved the effect of breaking the performance limit in the real world experimental environment by relying on the characteristics of phase strategy.


\subsubsection{Conclusion}


In terms of experimental results, while FedSSG did not perform as well on the relatively simple MNIST and EMNIST-L datasets compared to the CIFAR-10 and CIFAR-100 datasets, it still outperformed the baselines on average and exhibited a faster convergence rate. For the more complex CIFAR datasets, especially in the most challenging 500-clients-2\% setup, FedSSG performed exceptionally well and shows great potential for real-world FL tasks. We discuss the existence of the inflection point and aim to address this issue through theoretical proof in the future, in order to improve the full-period convergence rate of FedSSG.

\begin{figure*}[h!]
    \centering
    \subfloat[]{\includegraphics[width=0.23\textwidth]{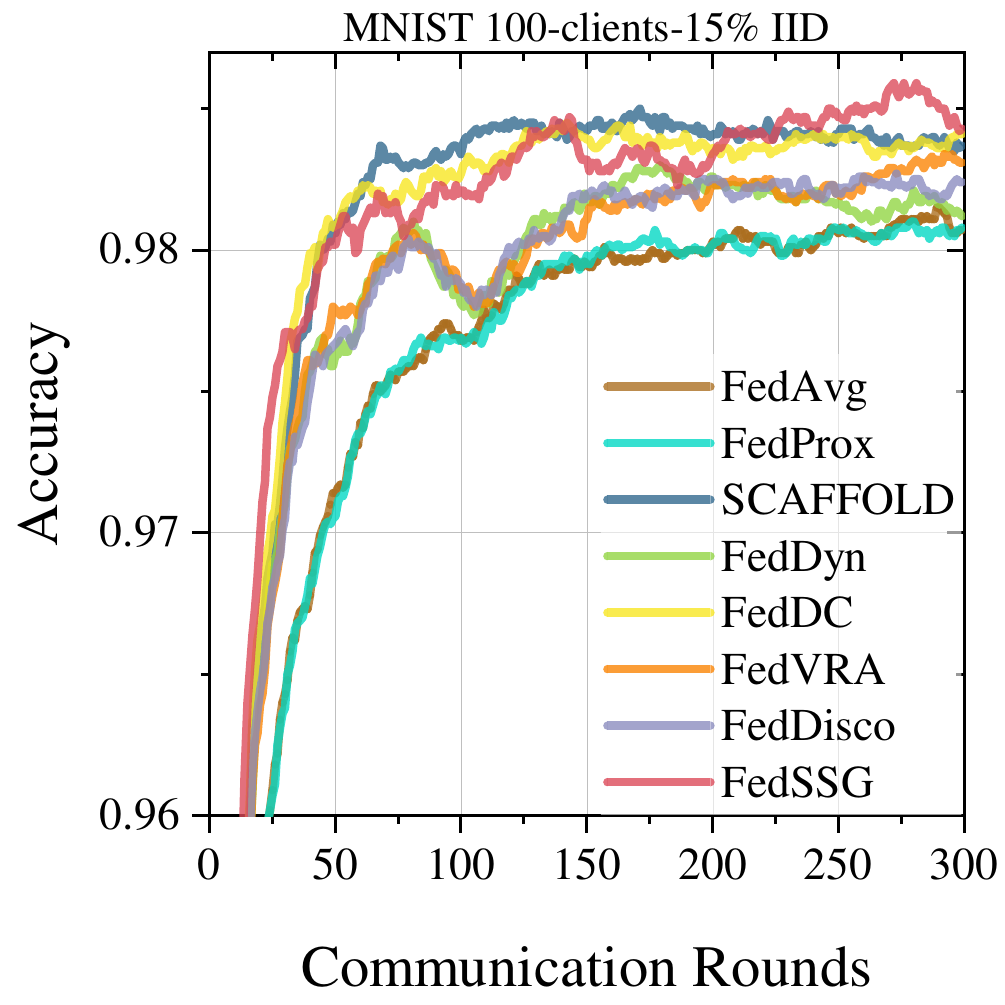}} \hfill
    \subfloat[]{\includegraphics[width=0.23\textwidth]{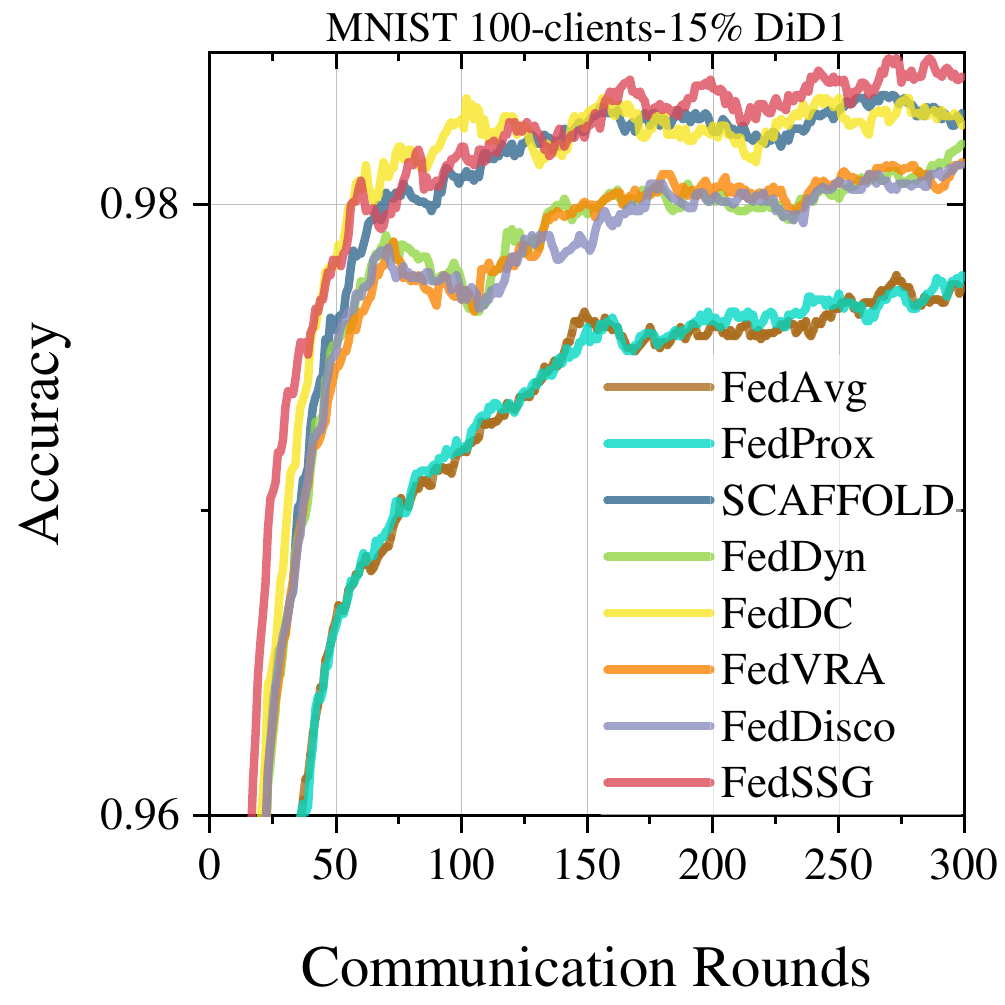}} \hfill
    \subfloat[]{\includegraphics[width=0.23\textwidth]{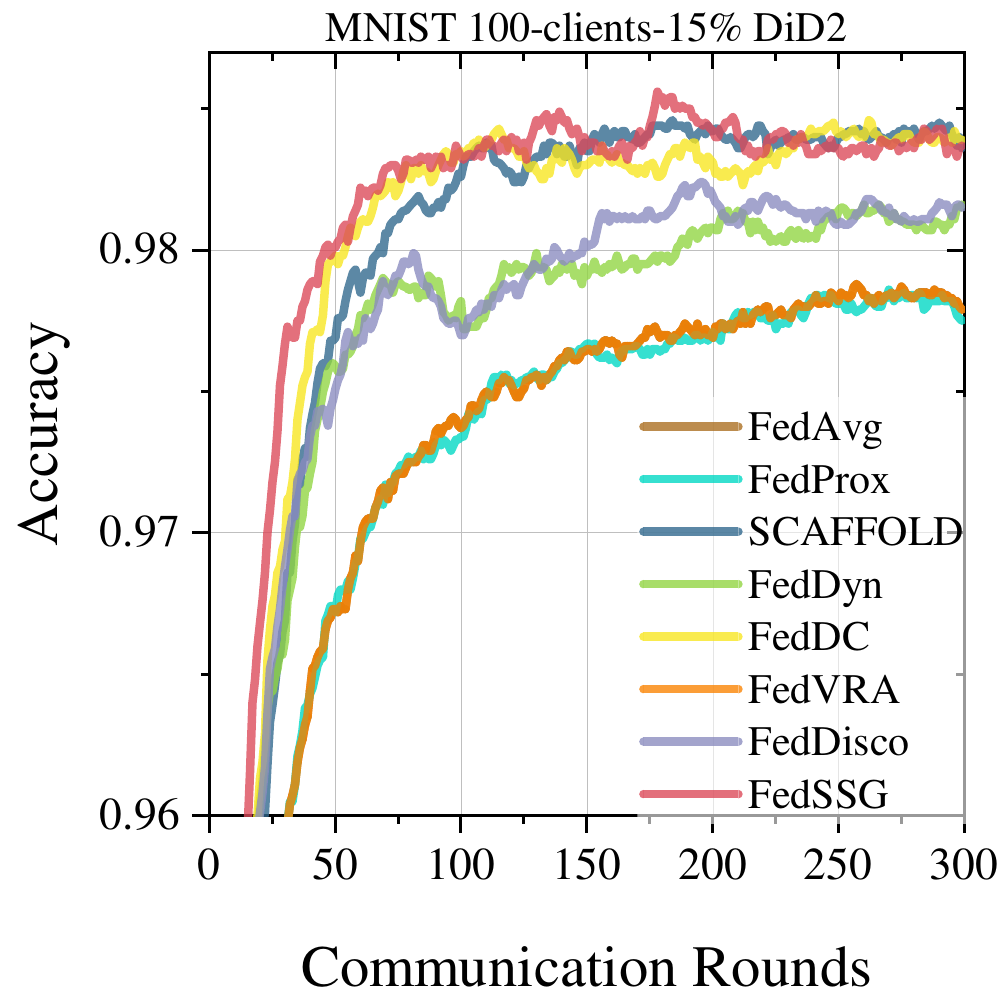}}
    \caption{Learning curves of FedSSG and baselines, with 100-clients-15\% settings on MNIST and on different distribution respectively.}
    \label{fig:MP15}
\end{figure*}

\begin{figure*}[h!]
    \centering
    \subfloat[]{\includegraphics[width=0.23\textwidth]{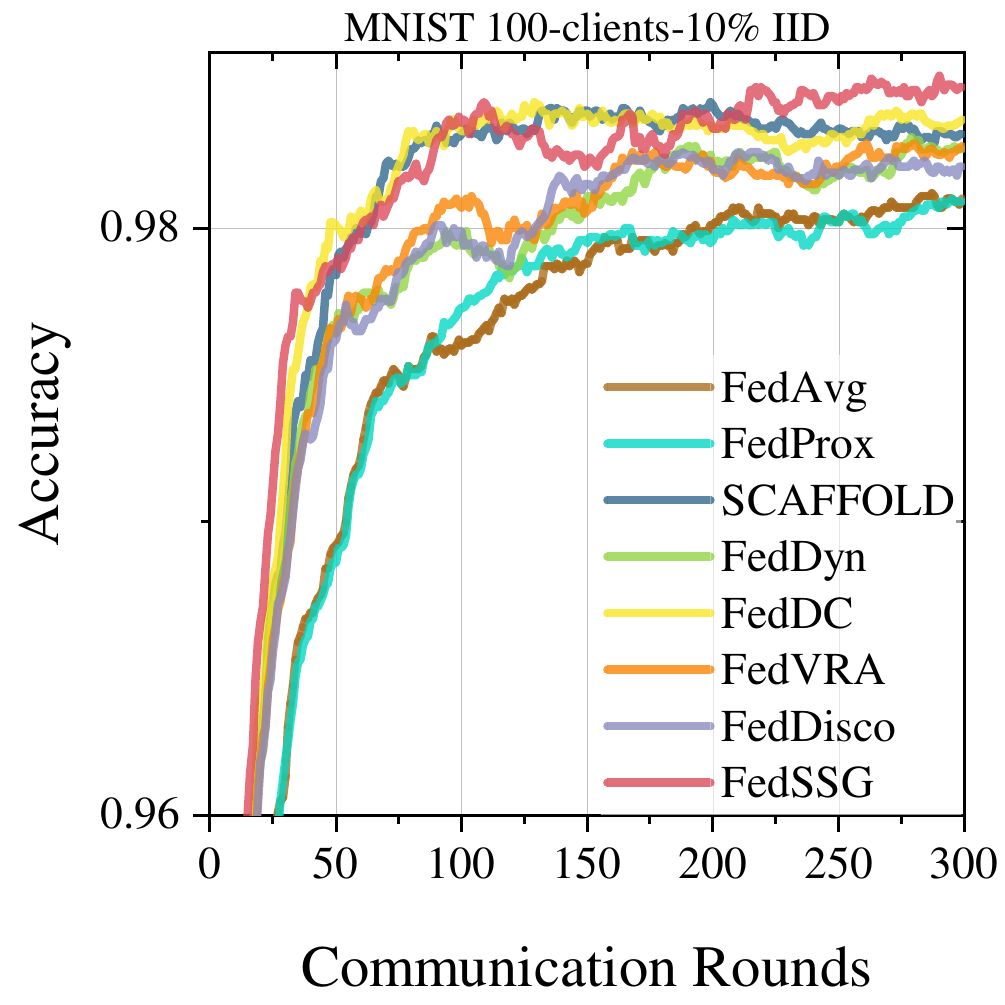}} \hfill
    \subfloat[]{\includegraphics[width=0.23\textwidth]{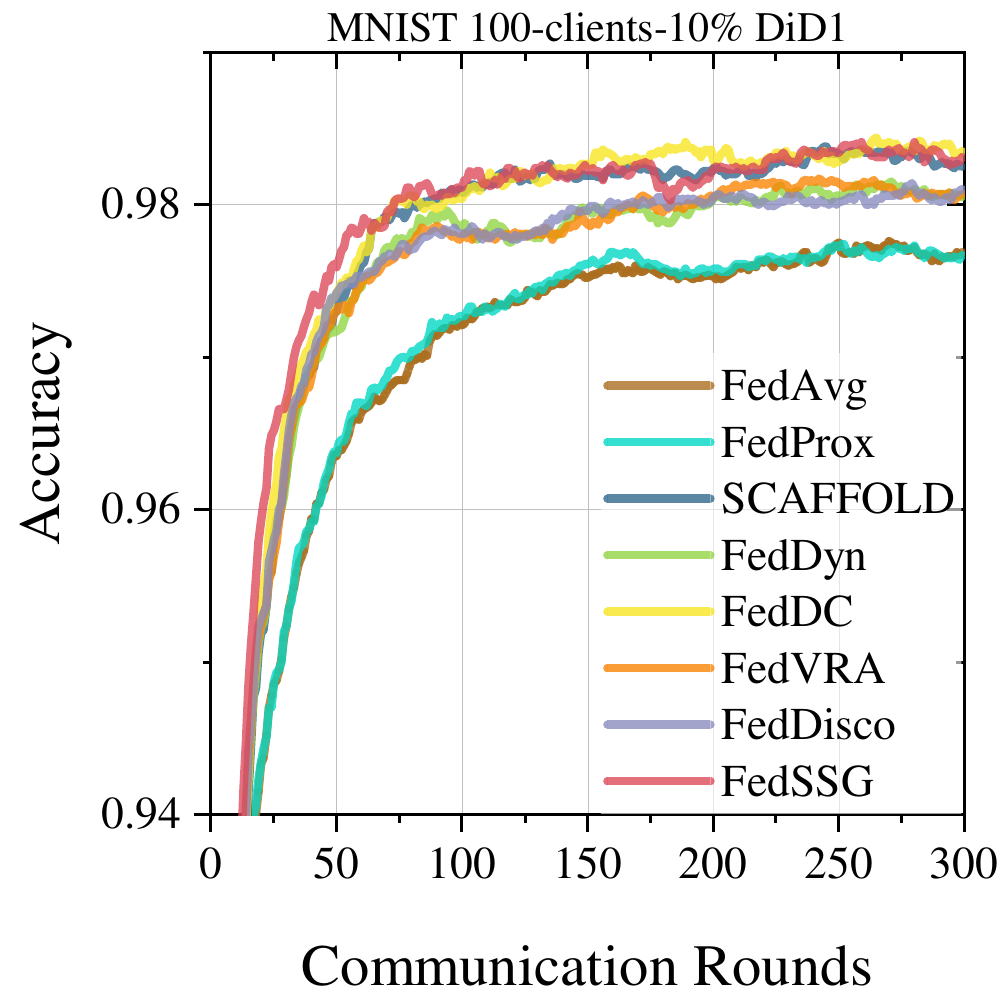}} \hfill
    \subfloat[]{\includegraphics[width=0.23\textwidth]{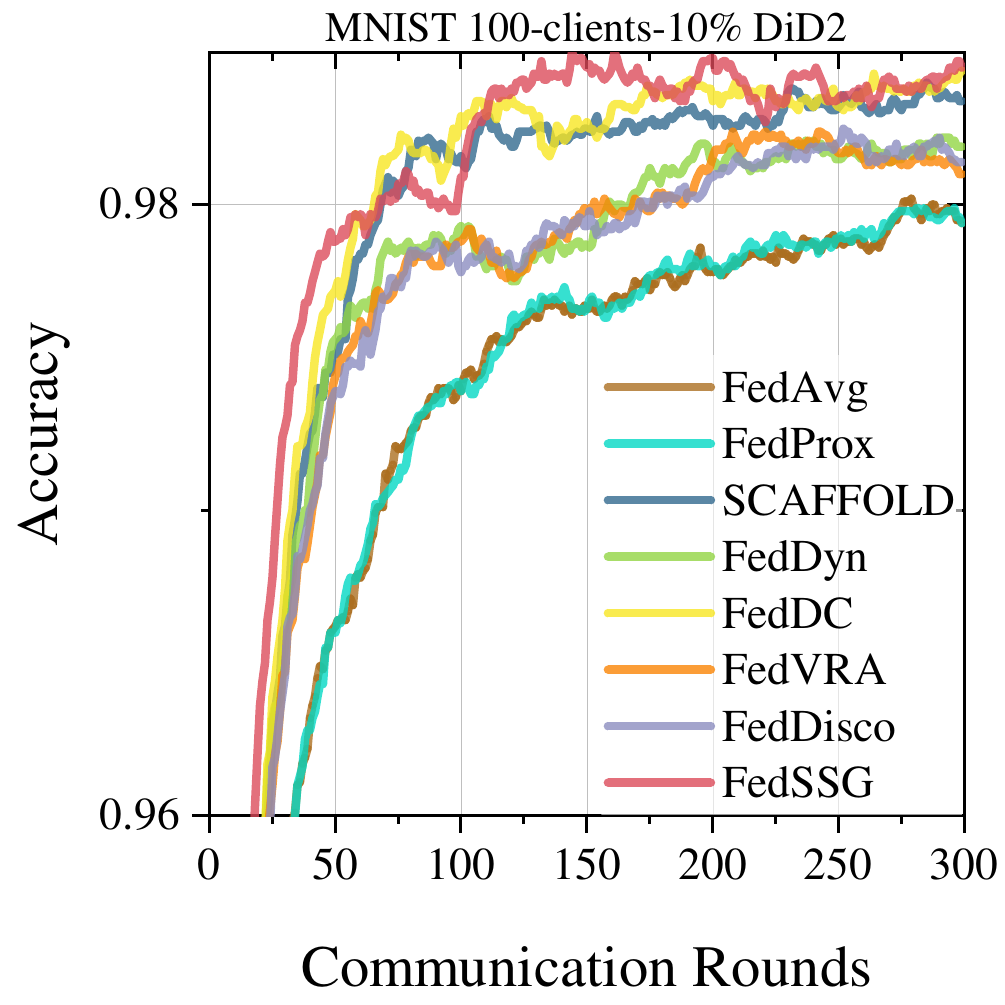}}
    \caption{Learning curves of FedSSG and baselines, with 100-clients-10\% settings on MNIST and on different distribution respectively.}
    \label{fig:MP10}
\end{figure*}

\begin{figure*}[h!]
    \centering
    \subfloat[]{\includegraphics[width=0.23\textwidth]{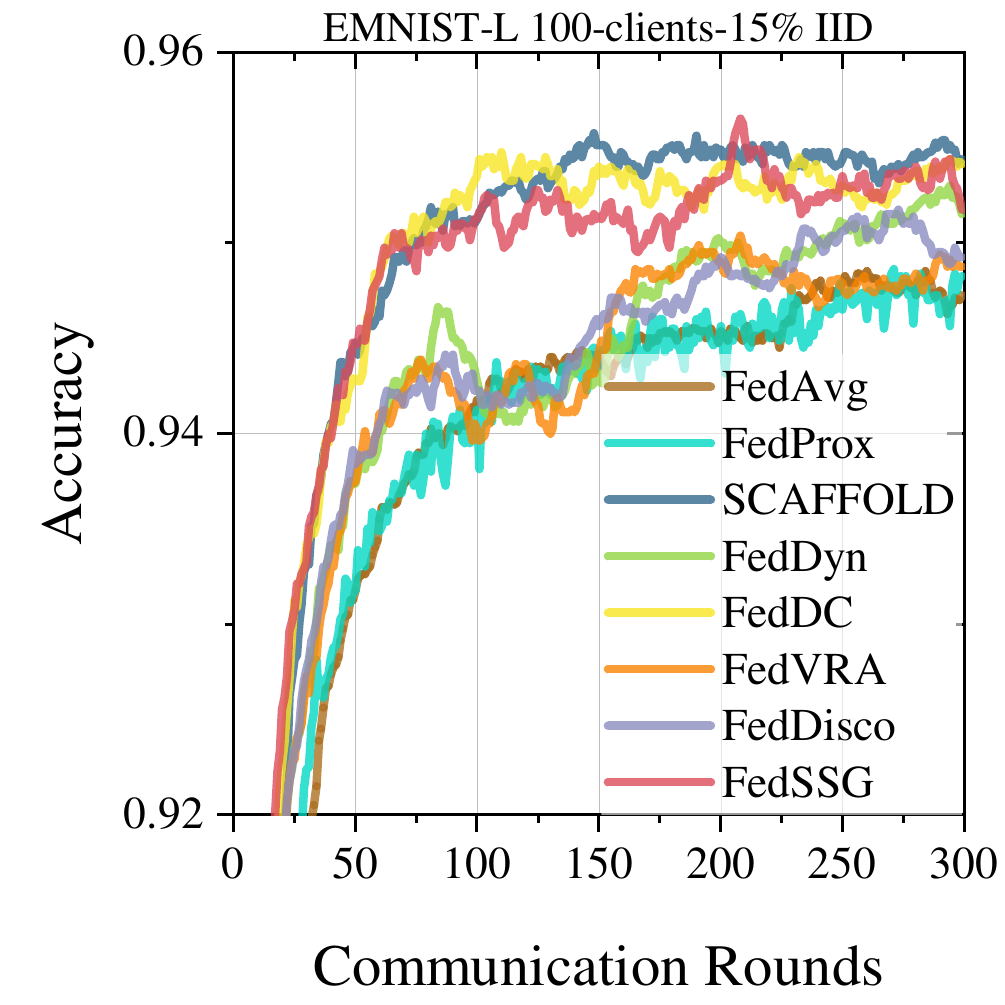}} \hfill
    \subfloat[]{\includegraphics[width=0.23\textwidth]{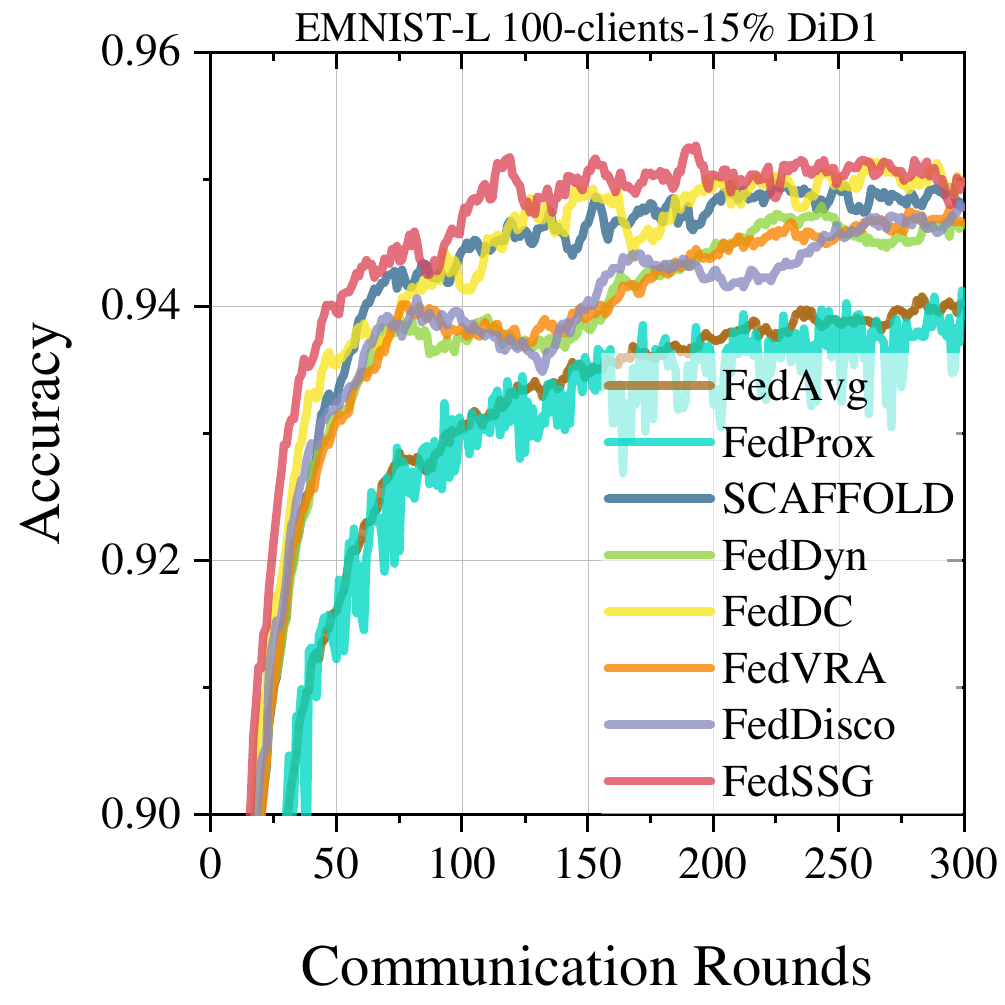}} \hfill
    \subfloat[]{\includegraphics[width=0.23\textwidth]{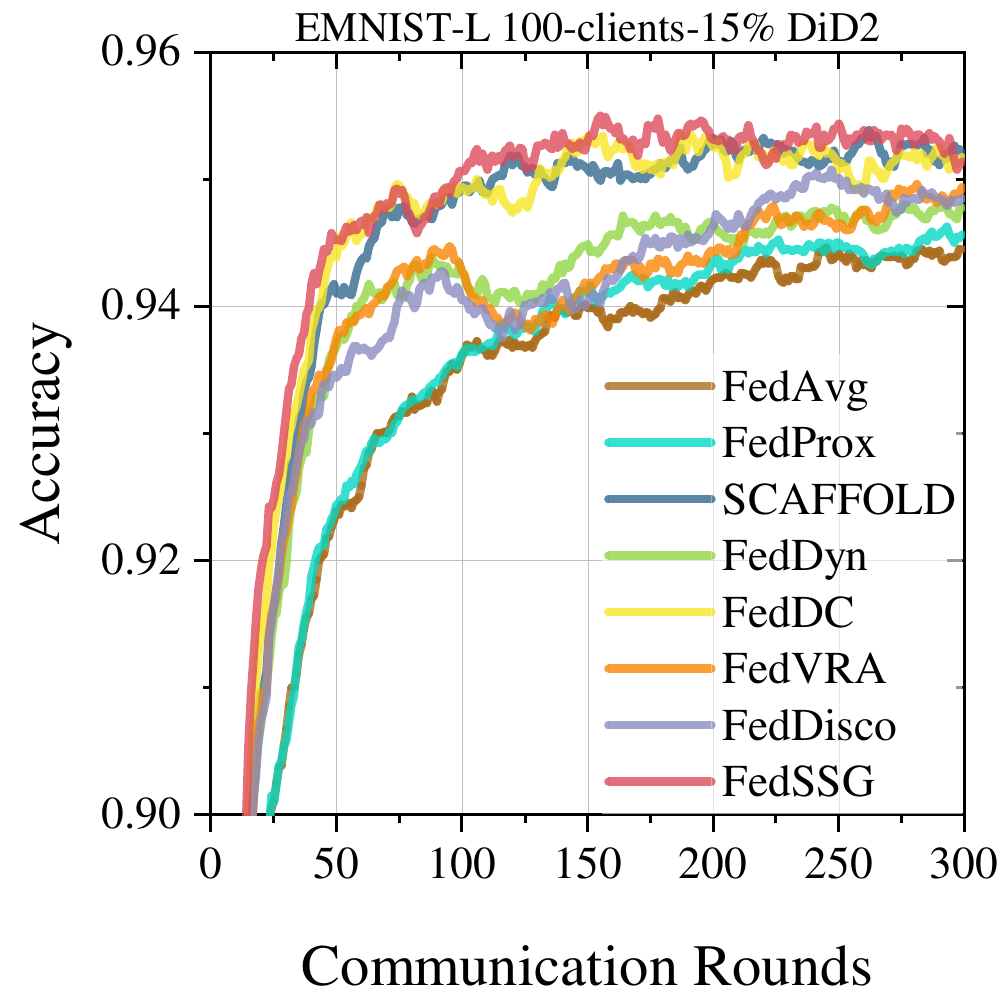}}
    \caption{Learning curves of FedSSG and baselines, with 100-clients-15\% settings on EMNIST-L and on different distribution respectively.}
    \label{fig:EP15}
\end{figure*}
\begin{figure*}[h!]
    \centering
    \subfloat[]{\includegraphics[width=0.23\textwidth]{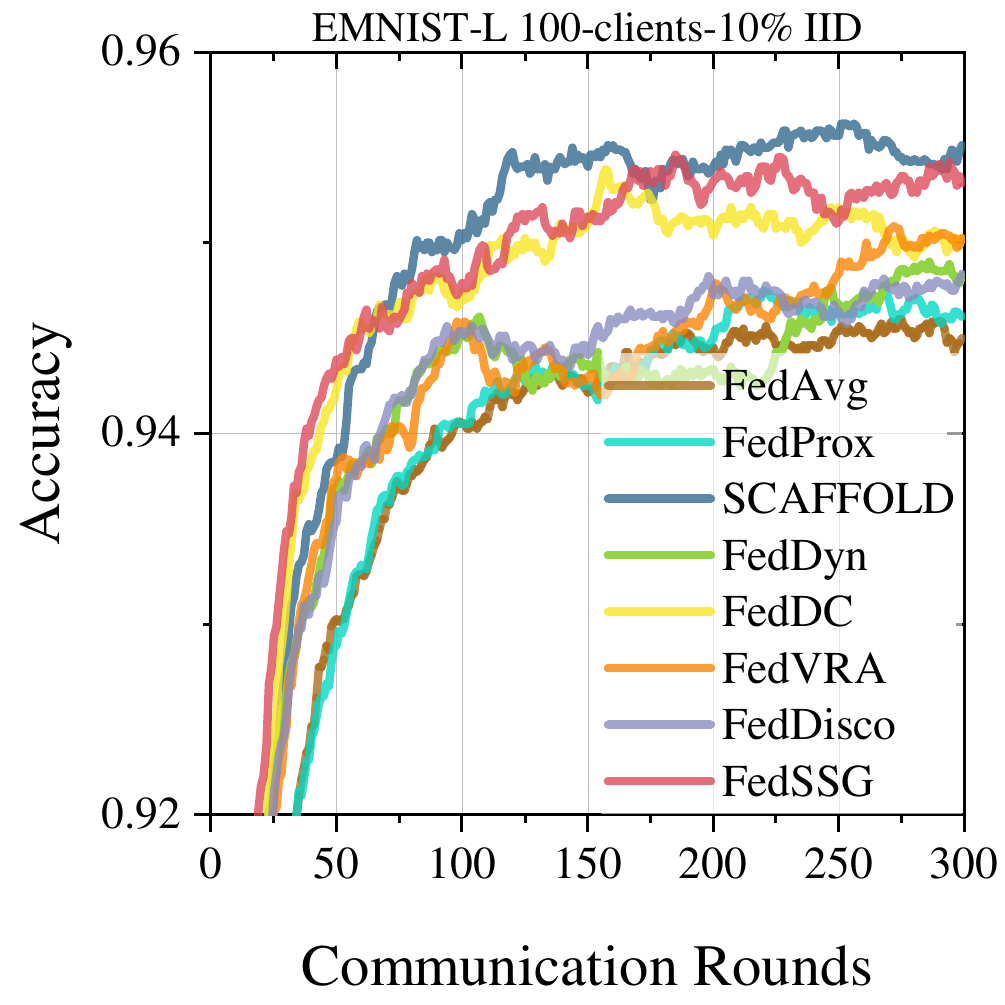}} \hfill
    \subfloat[]{\includegraphics[width=0.23\textwidth]{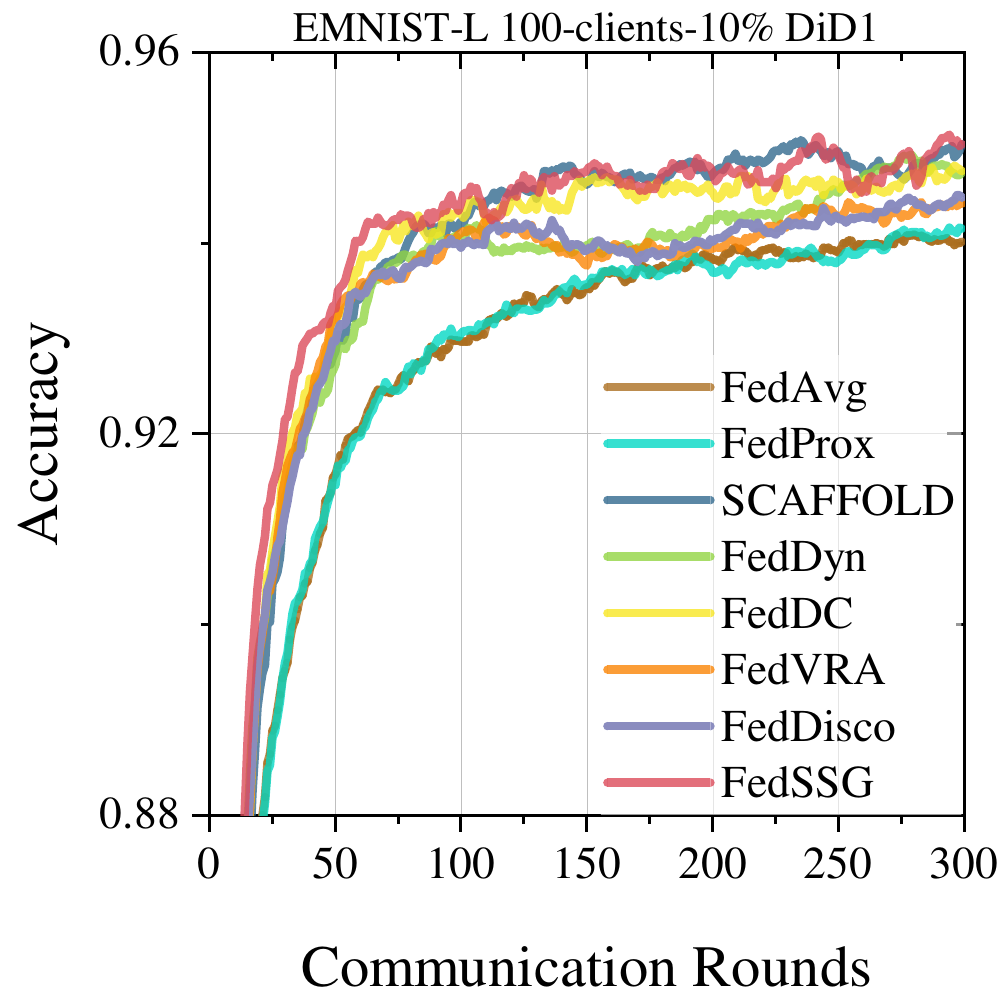}} \hfill
    \subfloat[]{\includegraphics[width=0.23\textwidth]{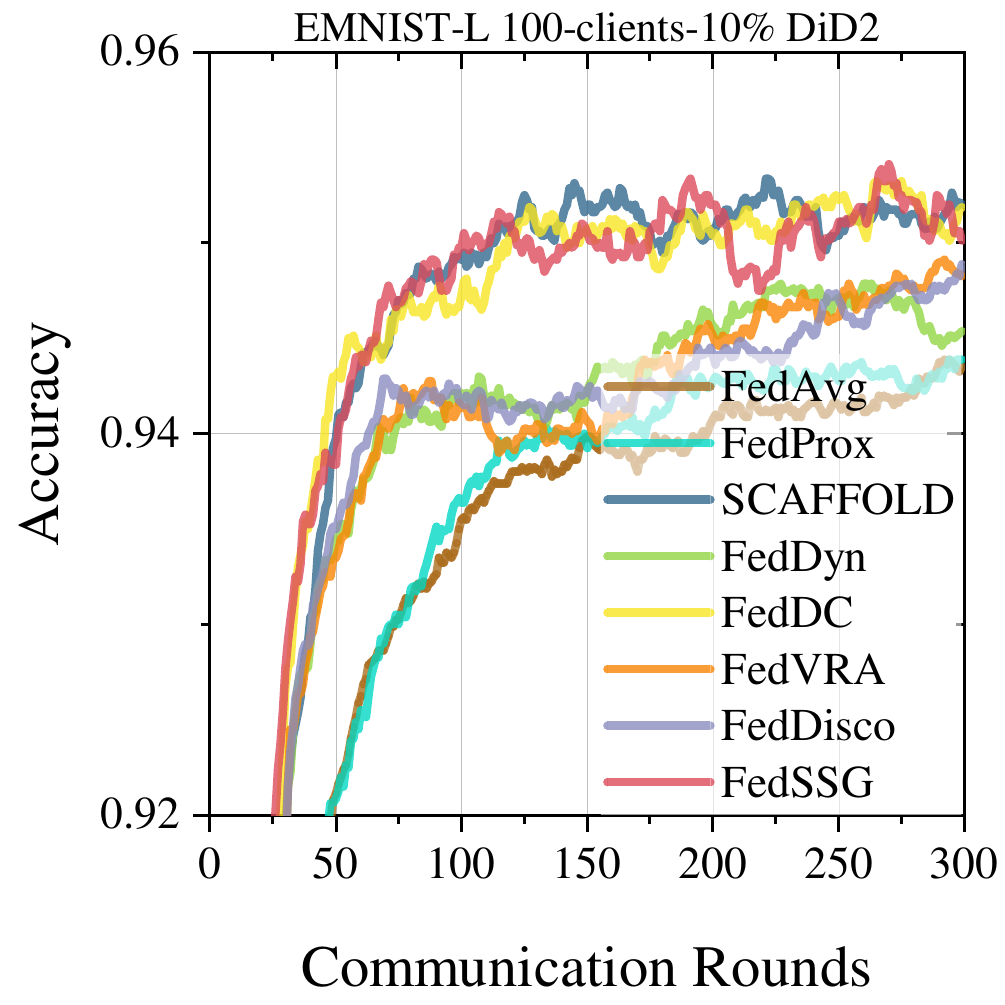}}
    \caption{Learning curves of FedSSG and baselines, with 100-clients-10\% settings on EMNIST-L and on different distribution respectively.}
    \label{fig:EP10}
\end{figure*}

\begin{figure*}[h!]
    \centering
    \subfloat[]{\includegraphics[width=0.23\textwidth]{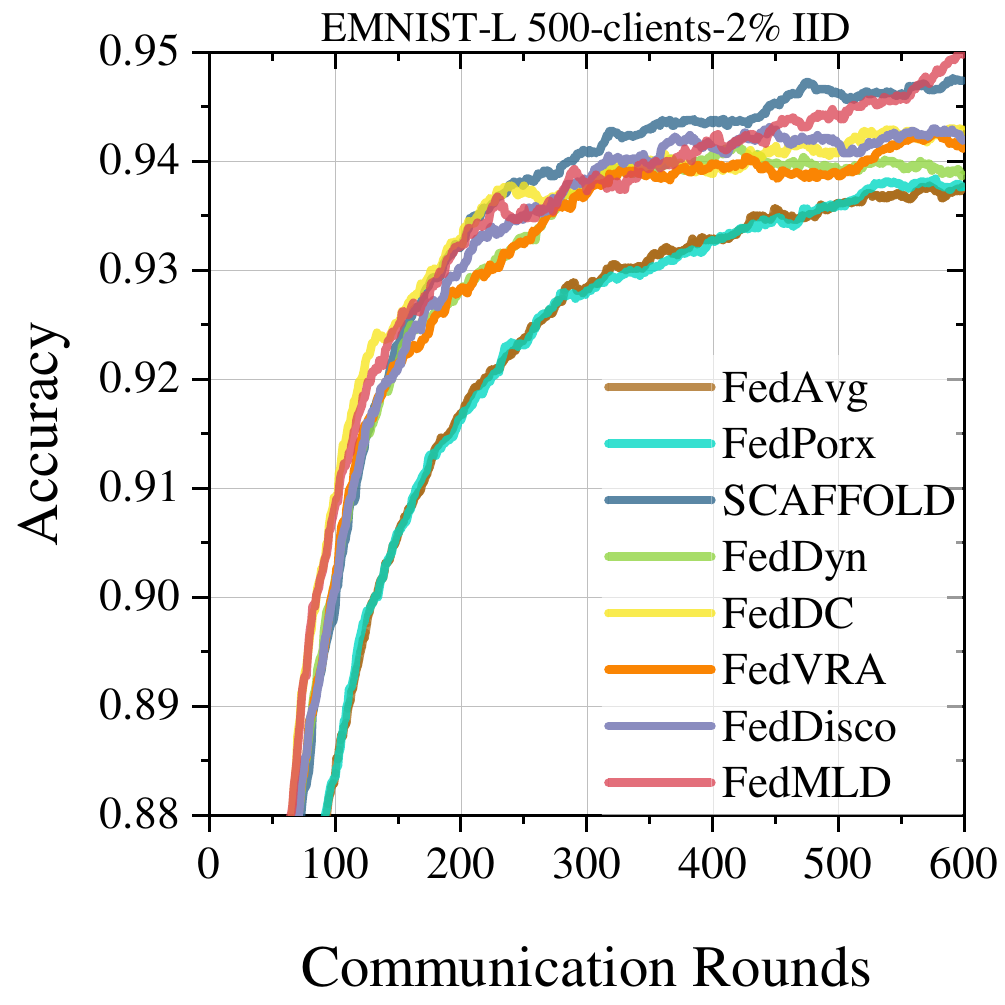}} \hfill
    \subfloat[]{\includegraphics[width=0.23\textwidth]{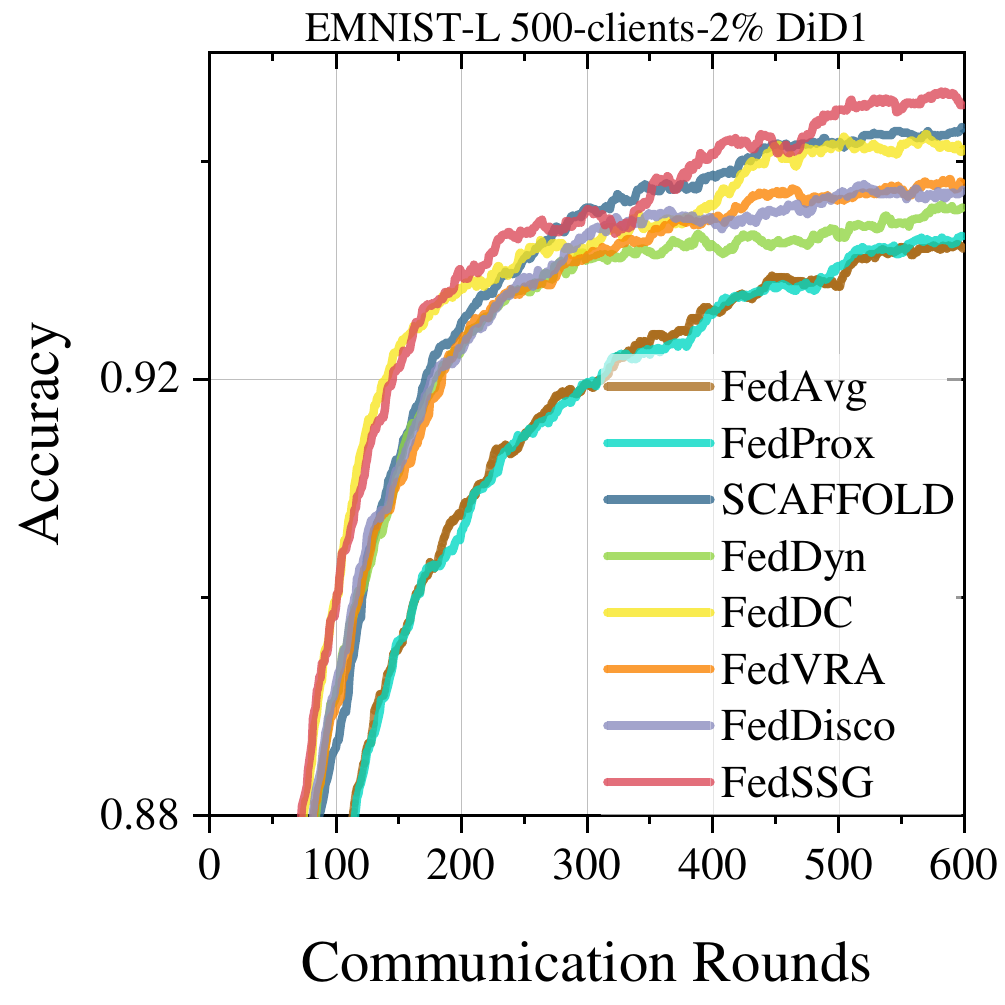}} \hfill
    \subfloat[]{\includegraphics[width=0.23\textwidth]{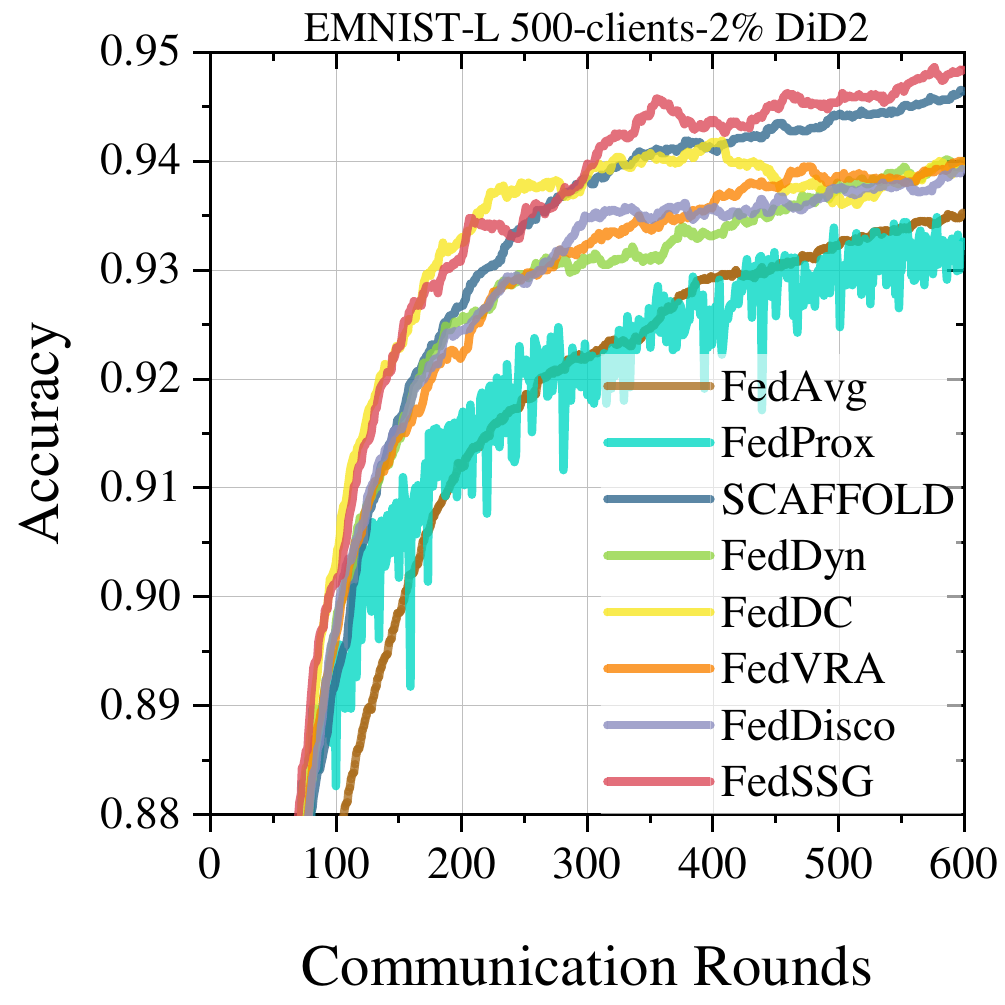}}
    \caption{Learning curves of FedSSG and baselines, with 500-clients-2\% settings on EMNIST-L and on different distribution respectively.}
    \label{fig:EP2}
\end{figure*}

\begin{figure*}[h!]
    \centering
    \subfloat[]{\includegraphics[width=0.23\textwidth]{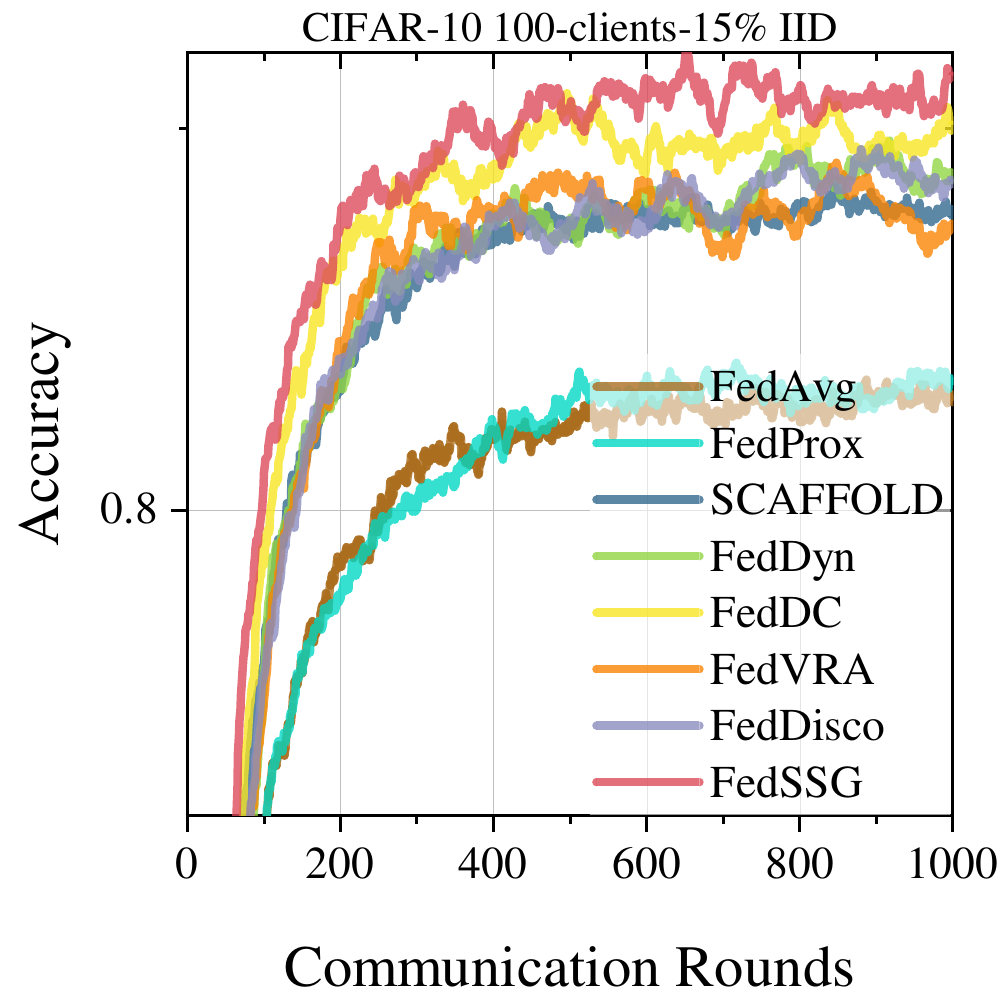}} \hfill
    \subfloat[]{\includegraphics[width=0.23\textwidth]{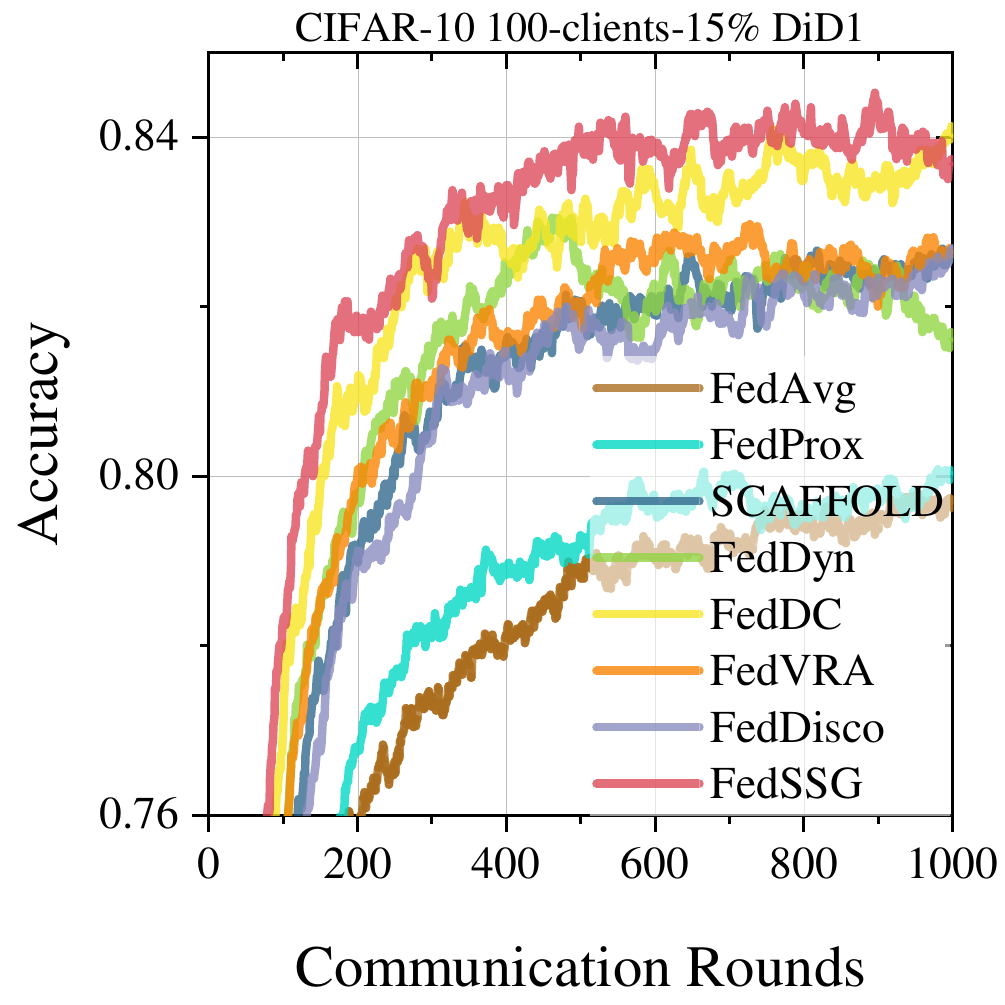}} \hfill
    \subfloat[]{\includegraphics[width=0.23\textwidth]{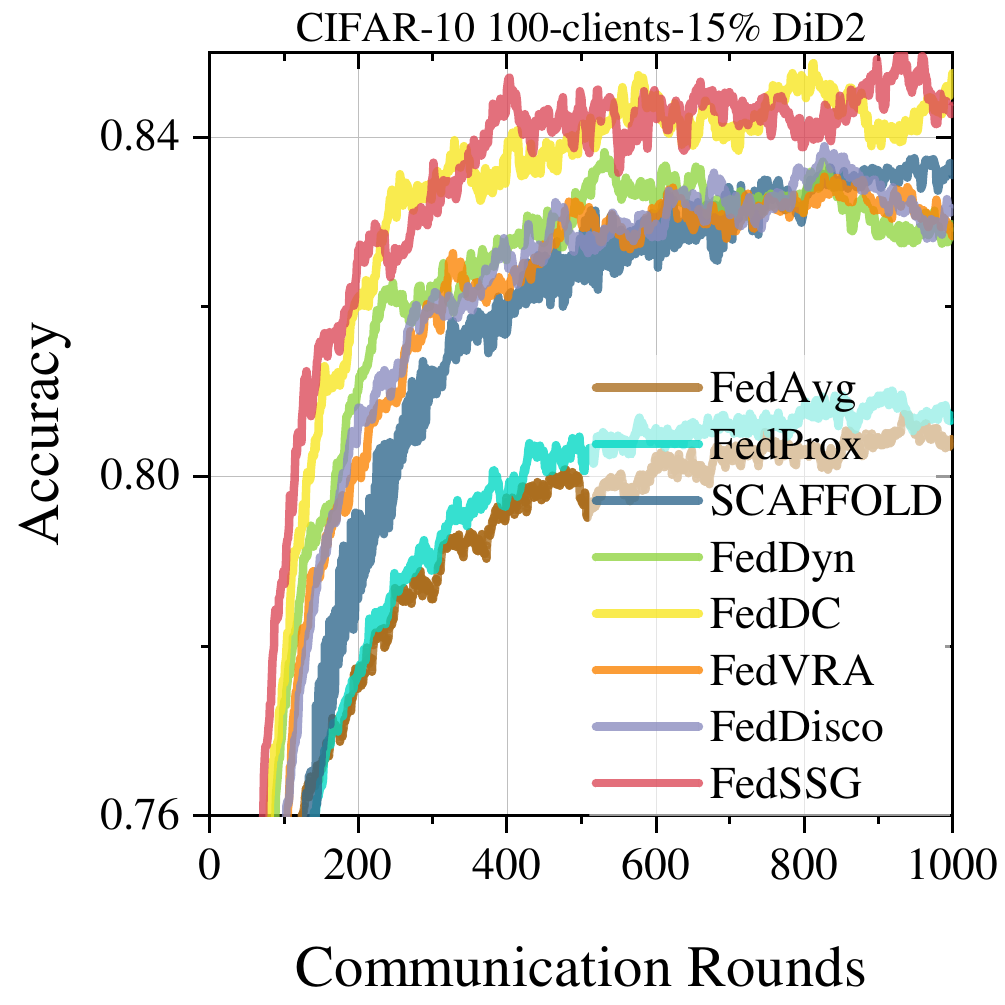}}
    \caption{Learning curves of FedSSG and baselines, with 100-clients-15\% settings on CIFAR-10 and on different distribution respectively.}
    \label{fig:C10P15}
\end{figure*}

\begin{figure*}[h!]
    \centering
    \subfloat[]{\includegraphics[width=0.23\textwidth]{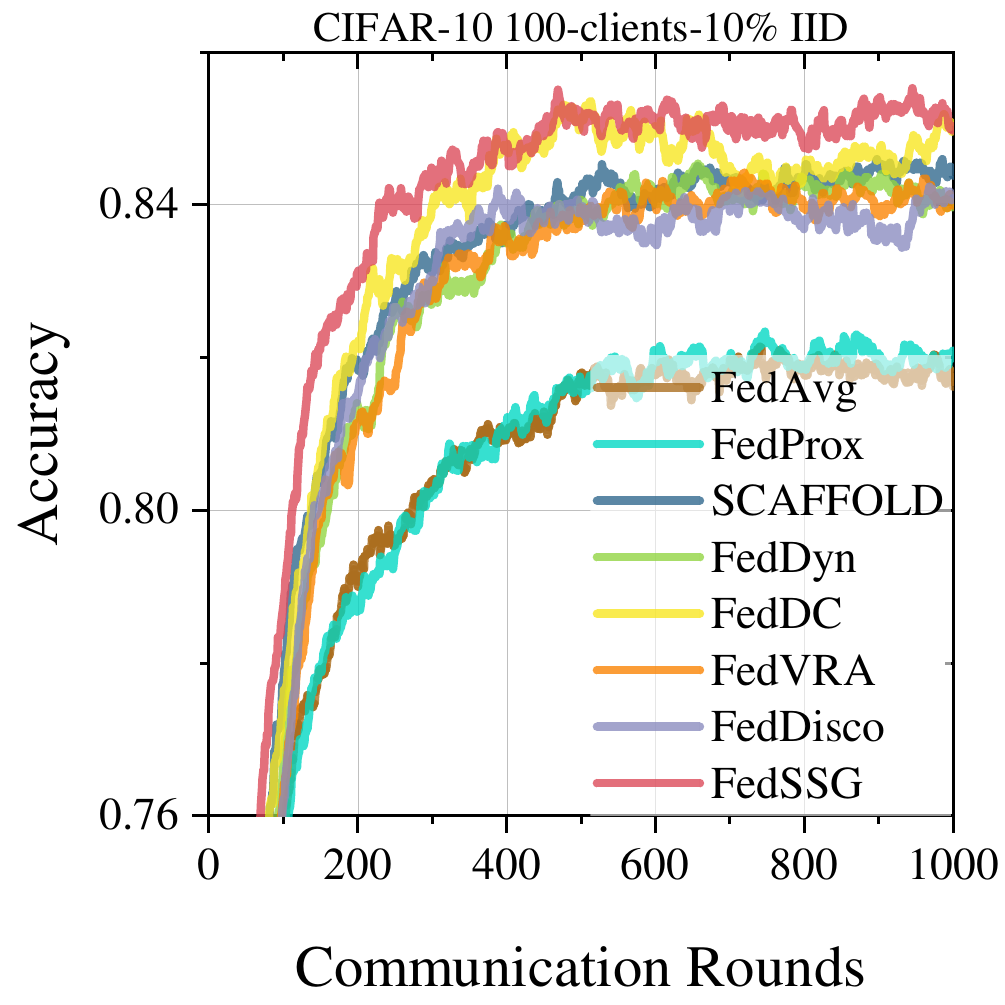}} \hfill
    \subfloat[]{\includegraphics[width=0.23\textwidth]{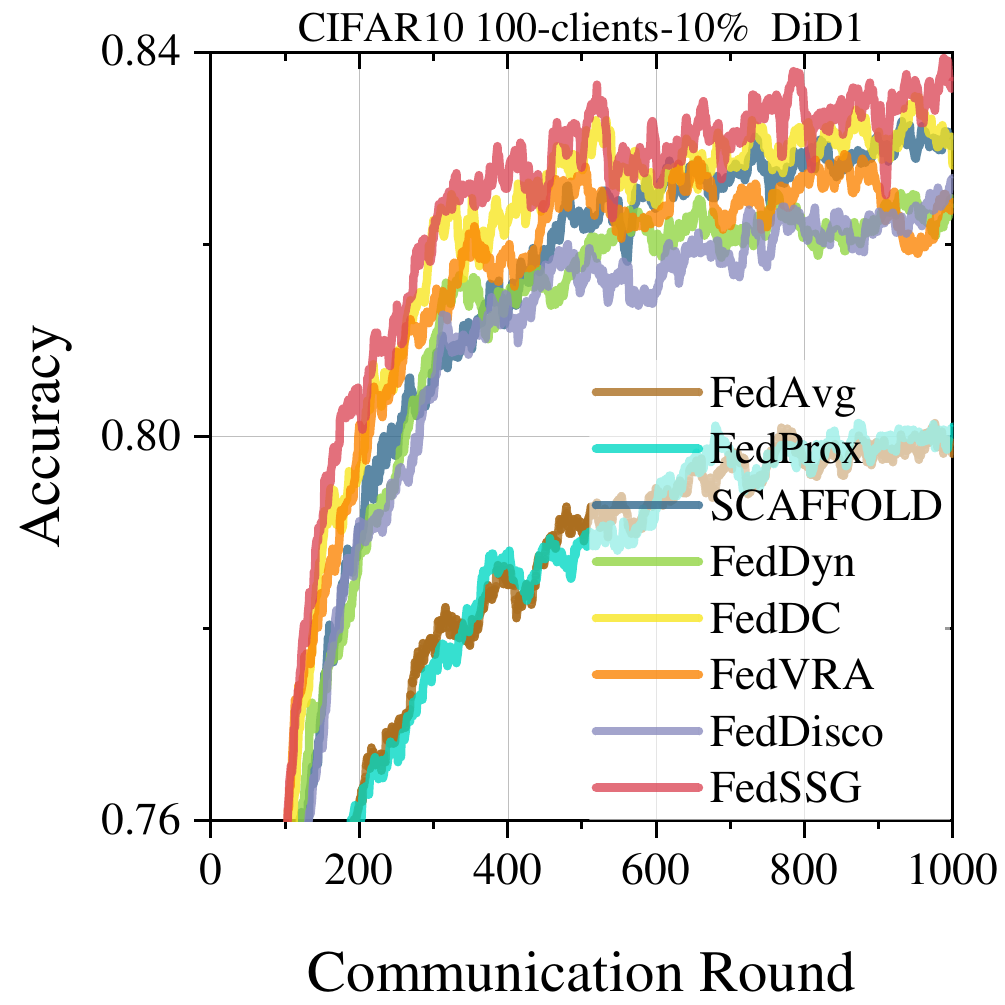}} \hfill
    \subfloat[]{\includegraphics[width=0.23\textwidth]{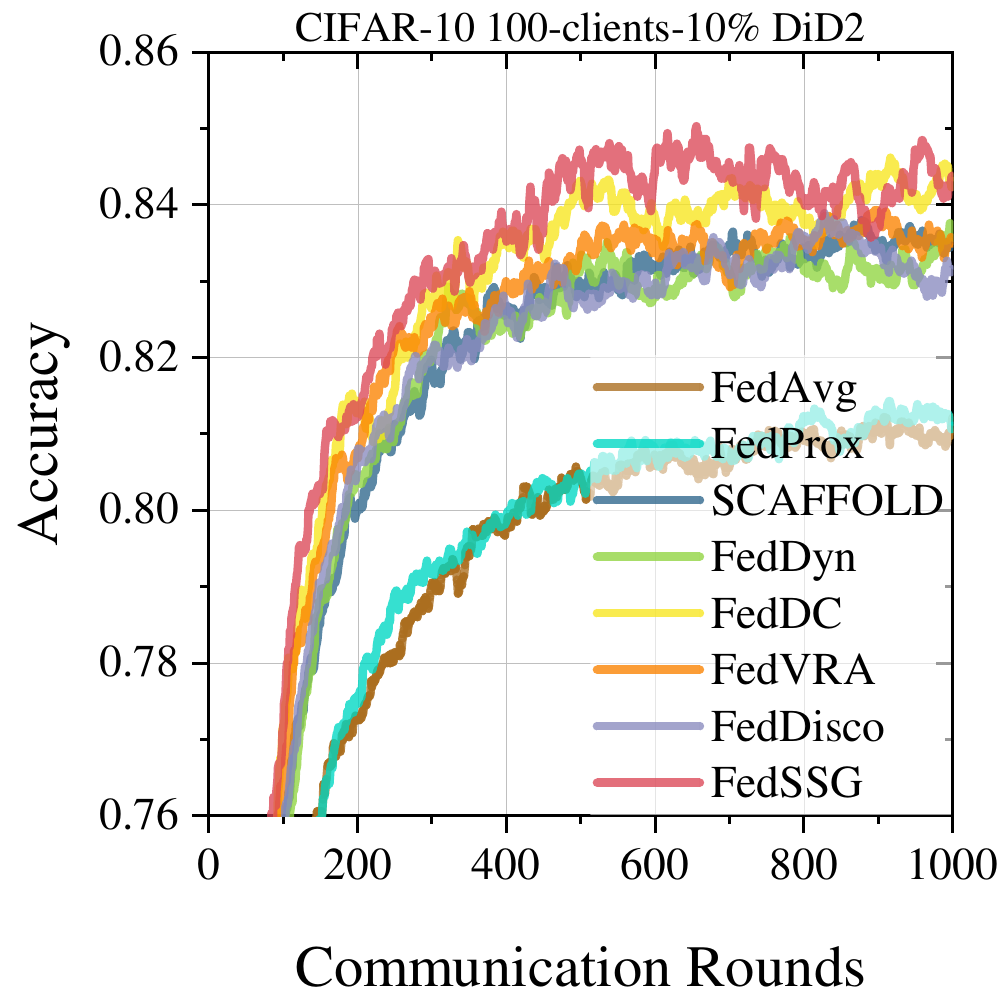}}
    \caption{Learning curves of FedSSG and baselines, with 100-clients-10\% settings on CIFAR-10 and on different distribution respectively.}
    \label{fig:C10P10}
\end{figure*}

\begin{figure*}[h!]
    \centering
    \subfloat[]{\includegraphics[width=0.23\textwidth]{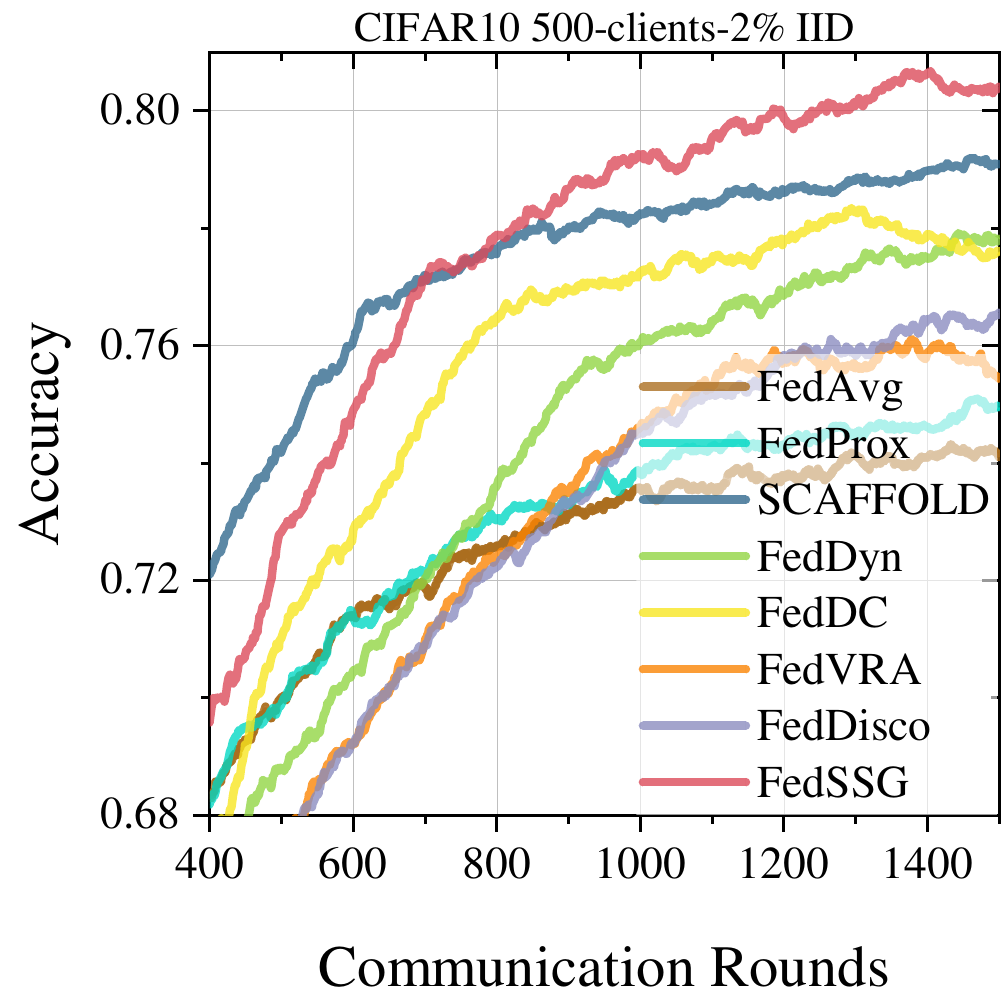}} \hfill
    \subfloat[]{\includegraphics[width=0.23\textwidth]{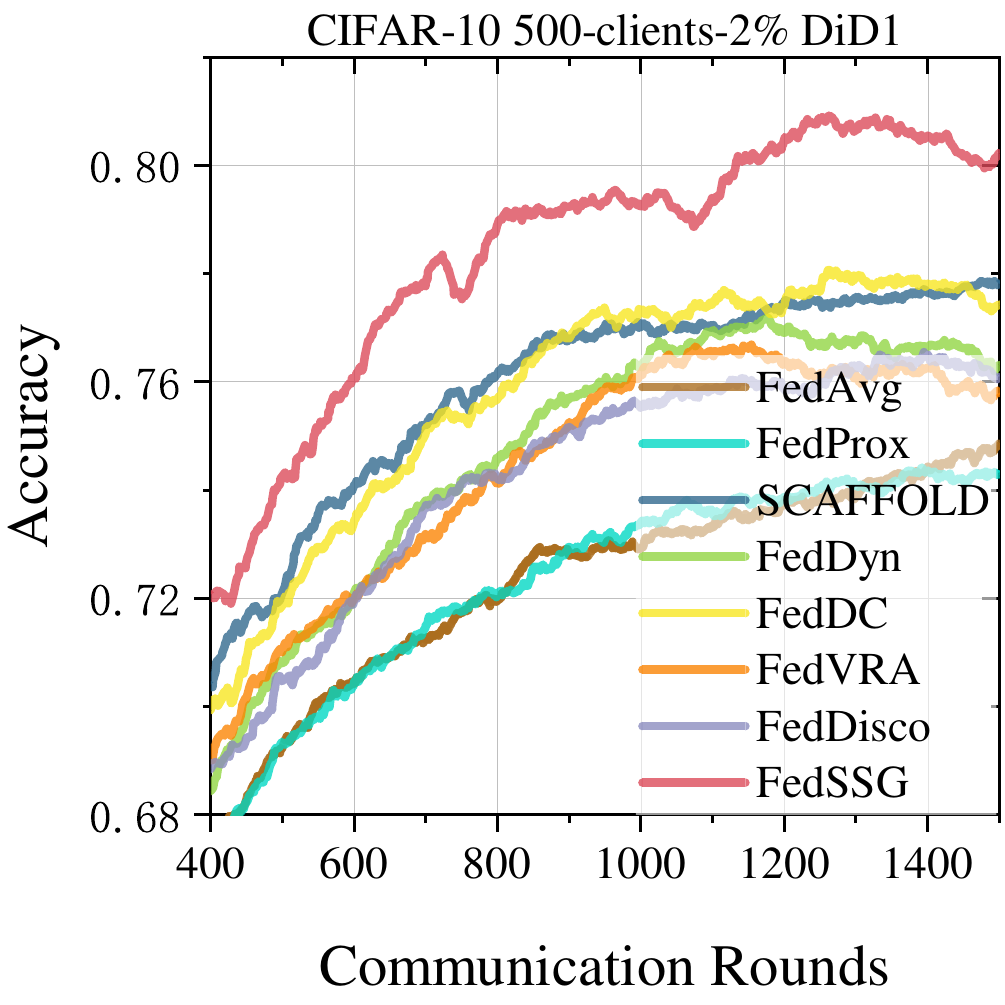}} \hfill
    \subfloat[]{\includegraphics[width=0.23\textwidth]{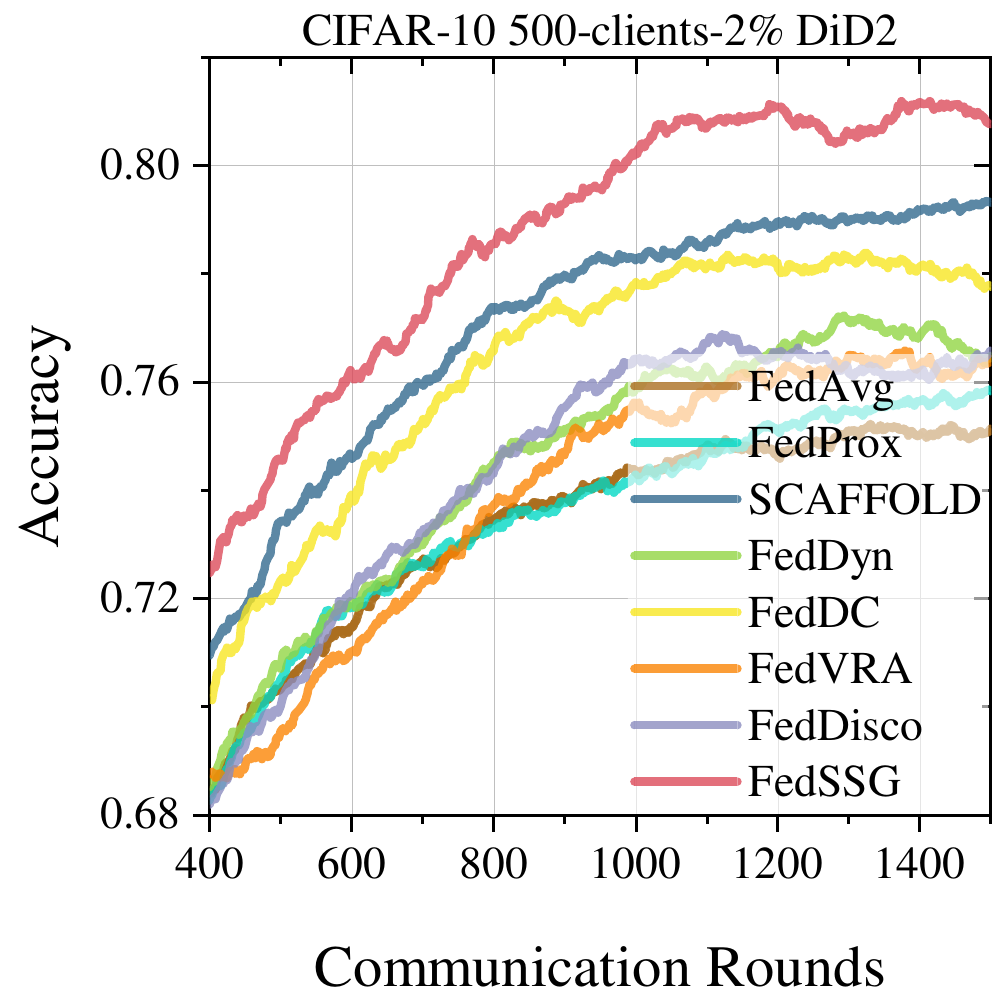}}
    \caption{Learning curves of FedSSG and baselines, with 500-clients-2\% settings on CIFAR-10 and on different distribution respectively.}
    \label{fig:C10P2}
\end{figure*}

\begin{figure*}[h!]
    \centering
    \subfloat[]{\includegraphics[width=0.23\textwidth]{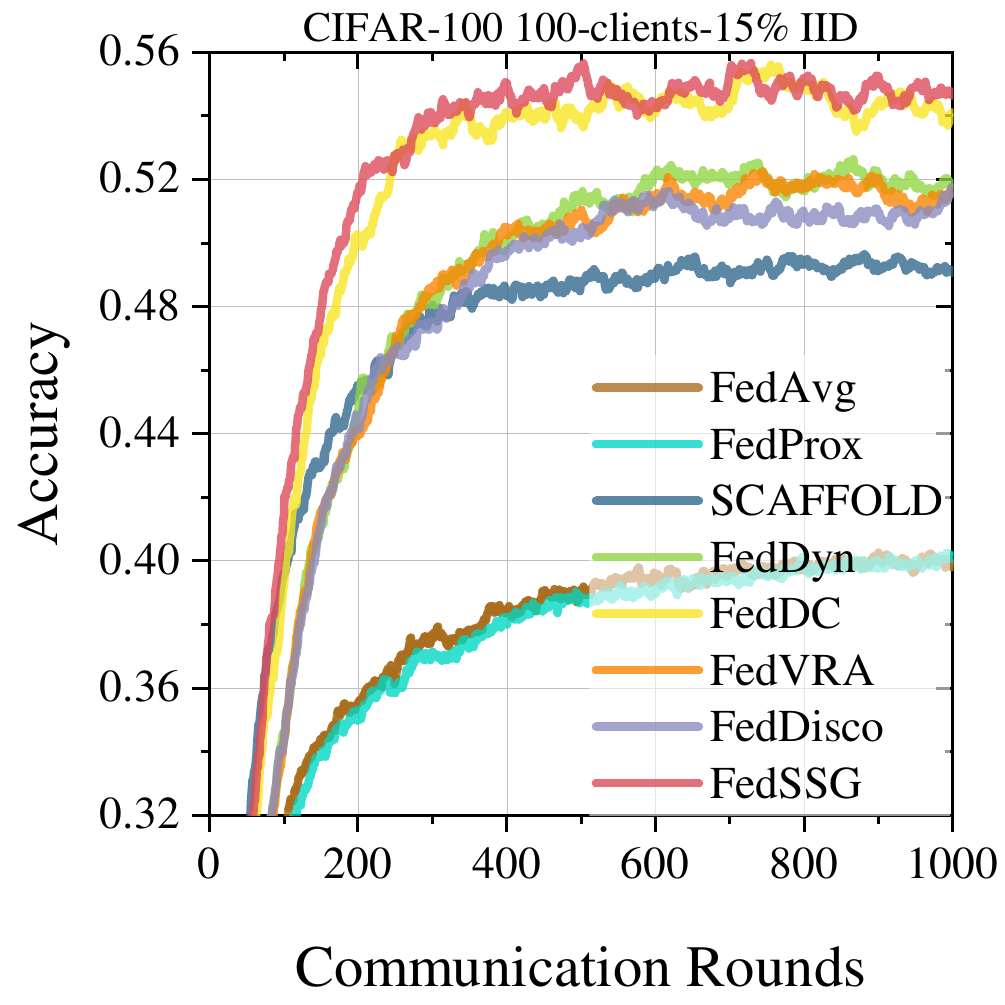}} \hfill
    \subfloat[]{\includegraphics[width=0.23\textwidth]{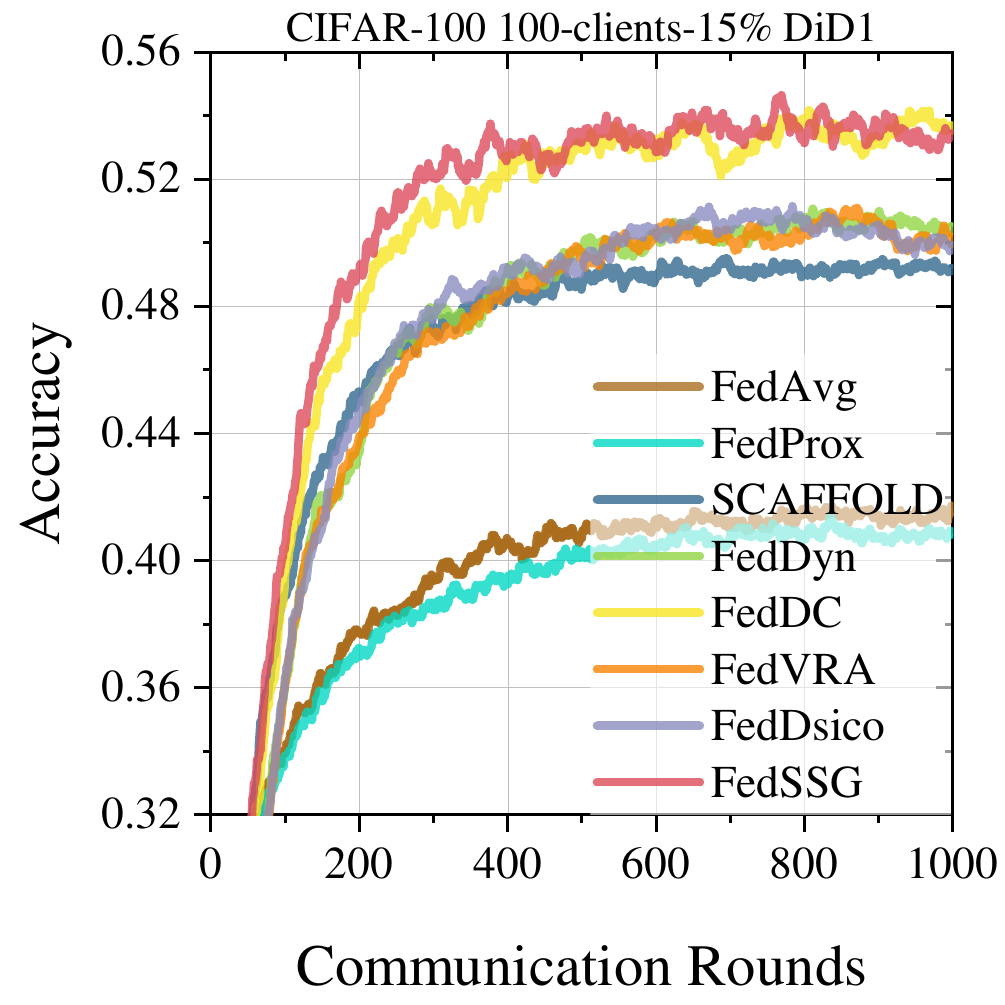}} \hfill
    \subfloat[]{\includegraphics[width=0.23\textwidth]{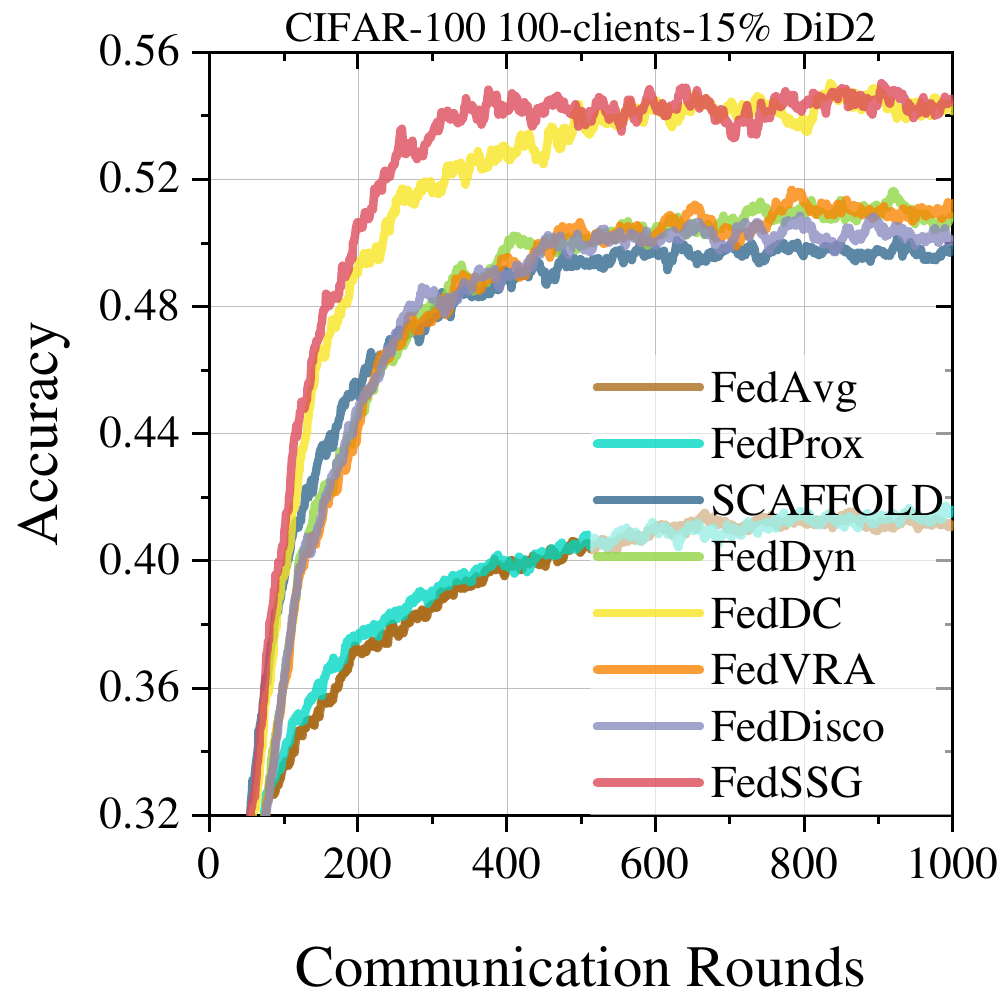}}
    \caption{Learning curves of FedSSG and baselines, with 100-clients-15\% settings on CIFAR-100 and on different distribution respectively.}  
    \label{fig:C100P15}
\end{figure*}

\begin{figure*}[h!]
    \centering
    \subfloat[]{\includegraphics[width=0.23\textwidth]{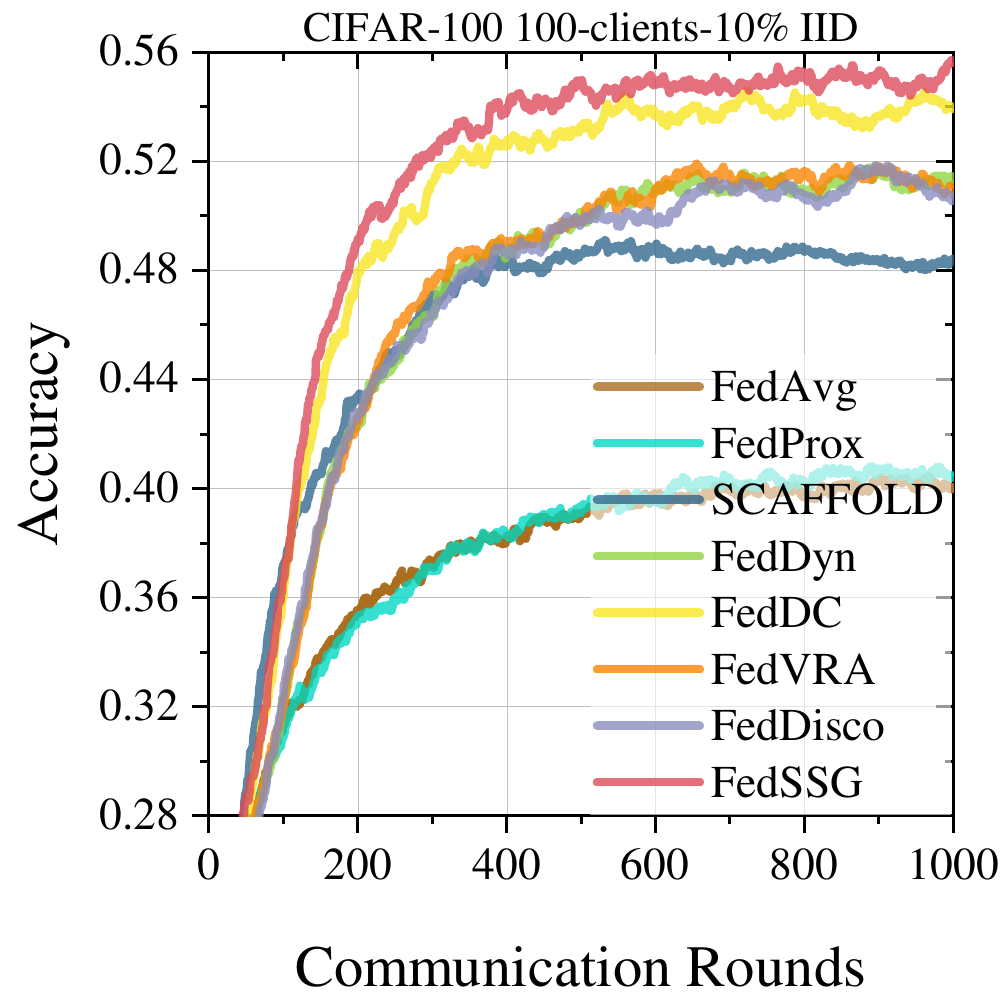}} \hfill
    \subfloat[]{\includegraphics[width=0.23\textwidth]{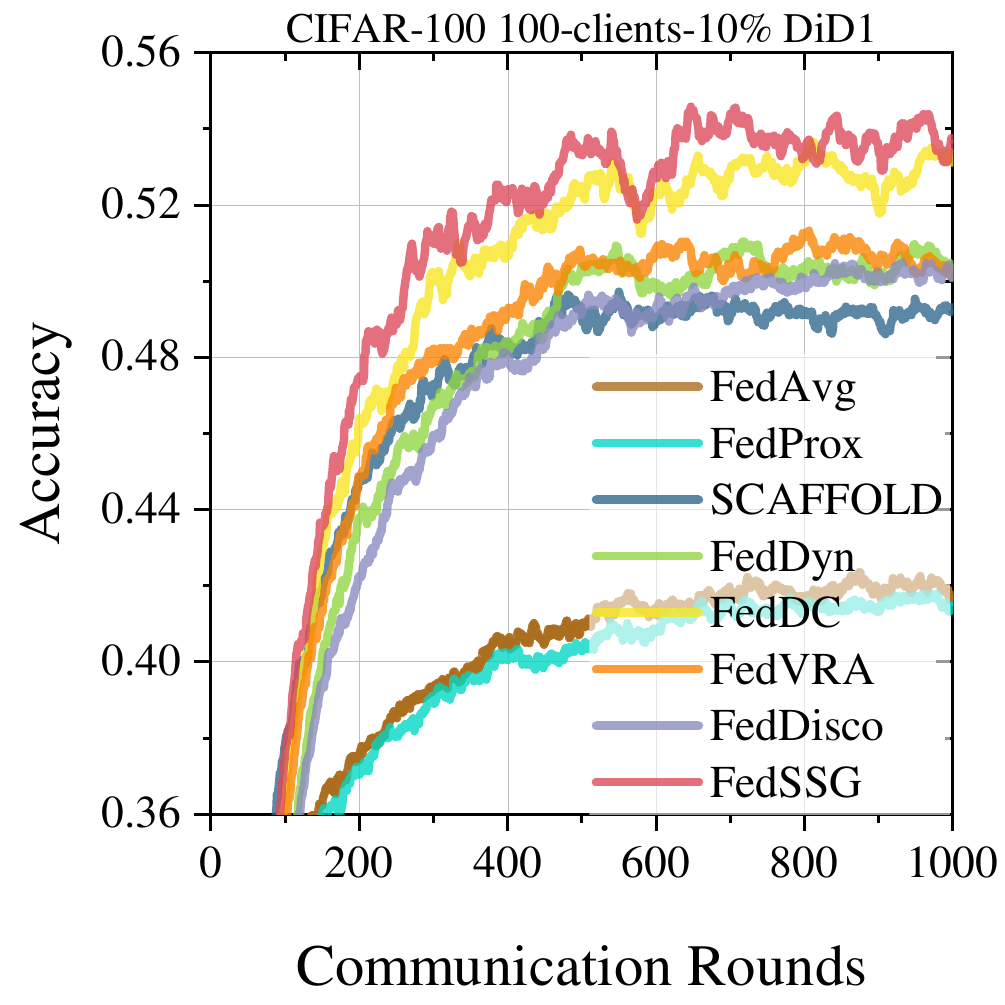}} \hfill
    \subfloat[]{\includegraphics[width=0.23\textwidth]{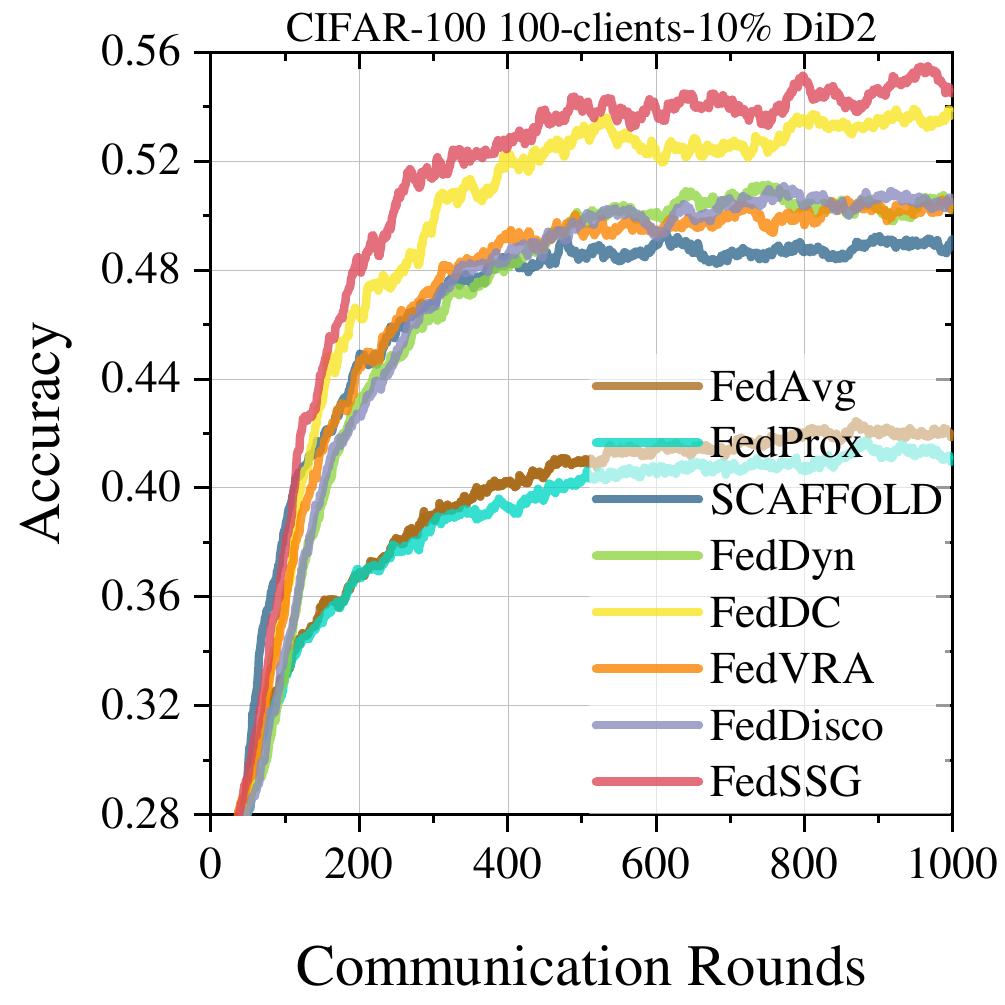}}
    \caption{Learning curves of FedSSG and baselines, with 100-clients-10\% settings on CIFAR-100 and on different distribution respectively.}
    \label{fig:C100P10}
\end{figure*}

\begin{figure*}[h!]
    \centering
    \subfloat[]{\includegraphics[width=0.23\textwidth]{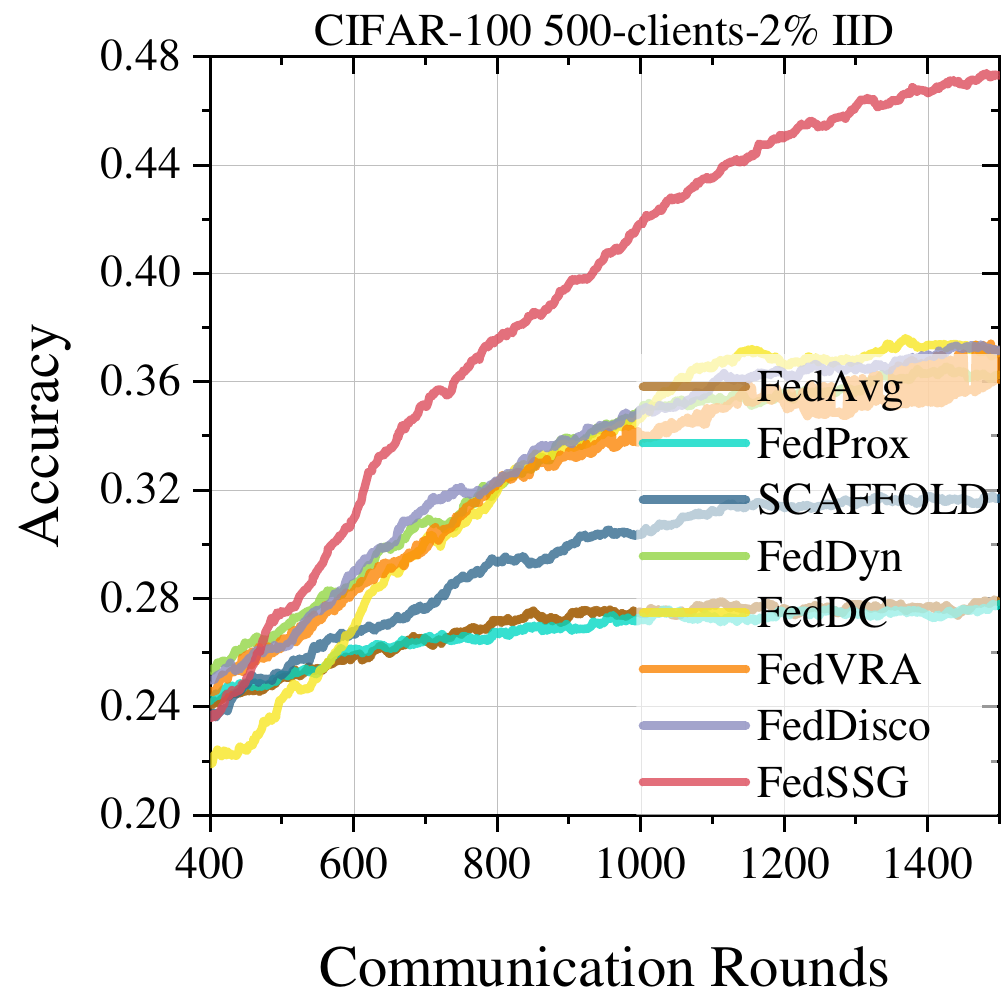}} \hfill
    \subfloat[]{\includegraphics[width=0.23\textwidth]{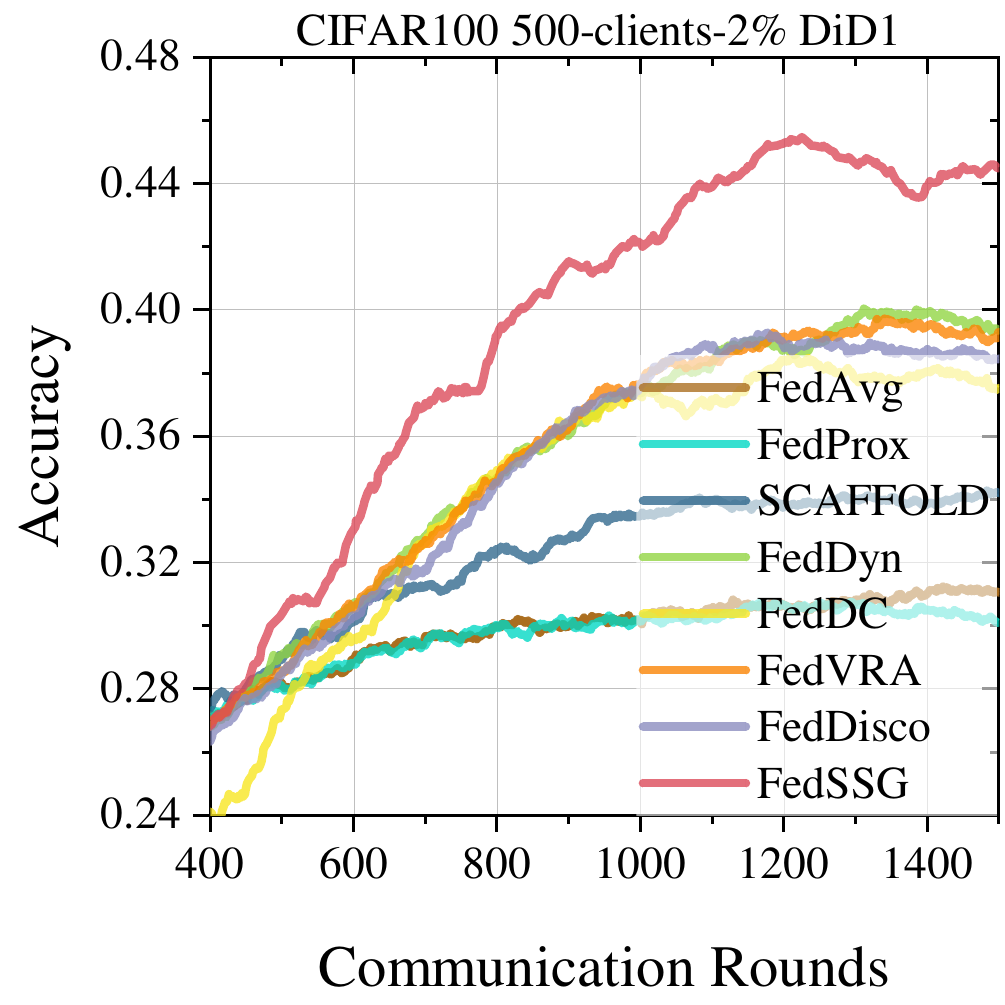}} \hfill
    \subfloat[]{\includegraphics[width=0.23\textwidth]{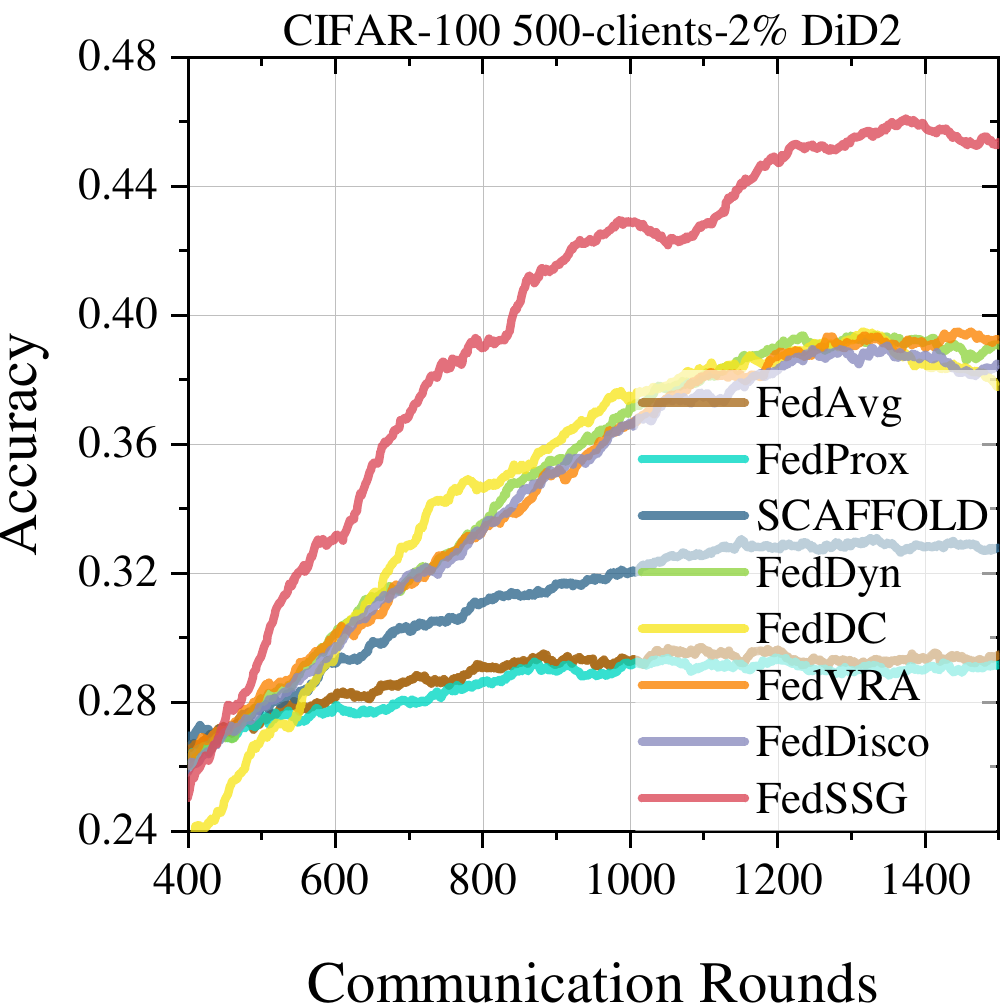}}
    \caption{Learning curves of FedSSG and baselines, with 500-clients-2\% settings on CIFAR-100 and on different distribution respectively.}
    \label{fig:C100P2}
\end{figure*}

\begin{figure*}[h!]
    \centering
    \subfloat[]{\includegraphics[width=0.23\textwidth]{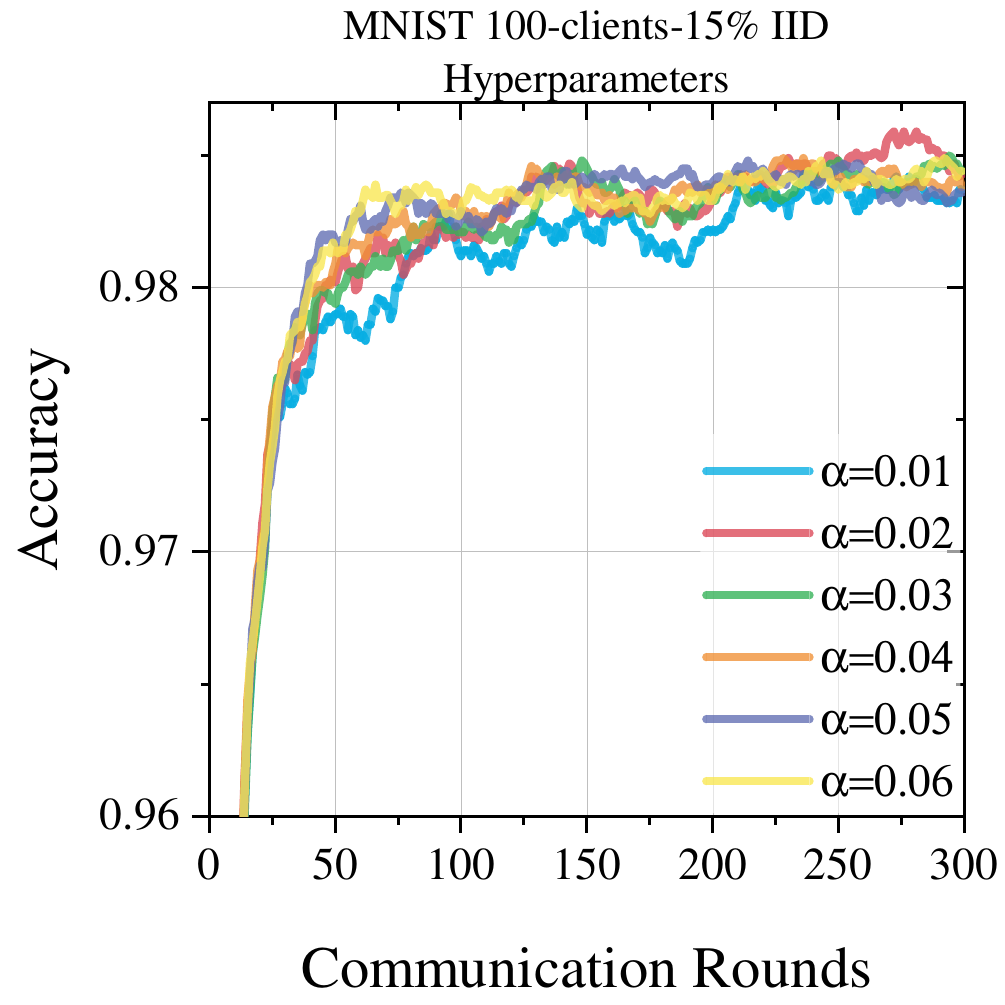}} \hfill
    \subfloat[]{\includegraphics[width=0.23\textwidth]{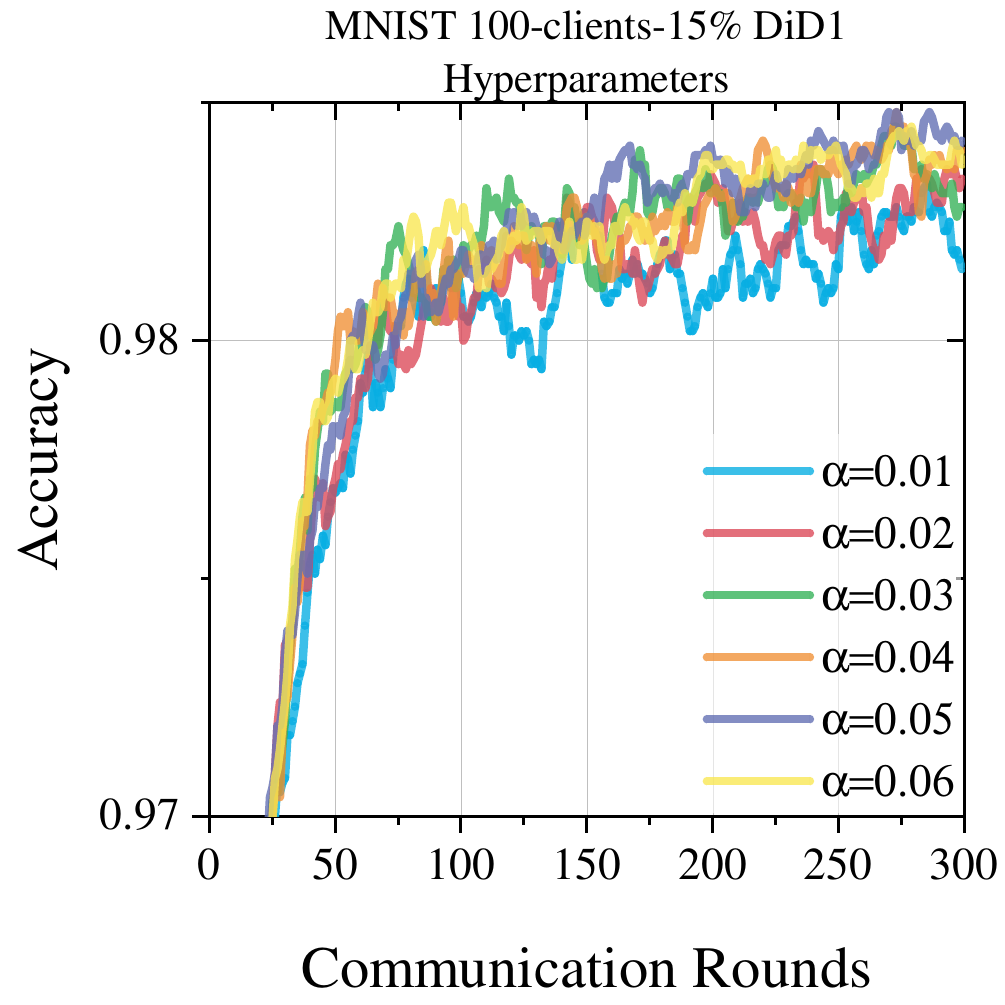}} \hfill
    \subfloat[]{\includegraphics[width=0.23\textwidth]{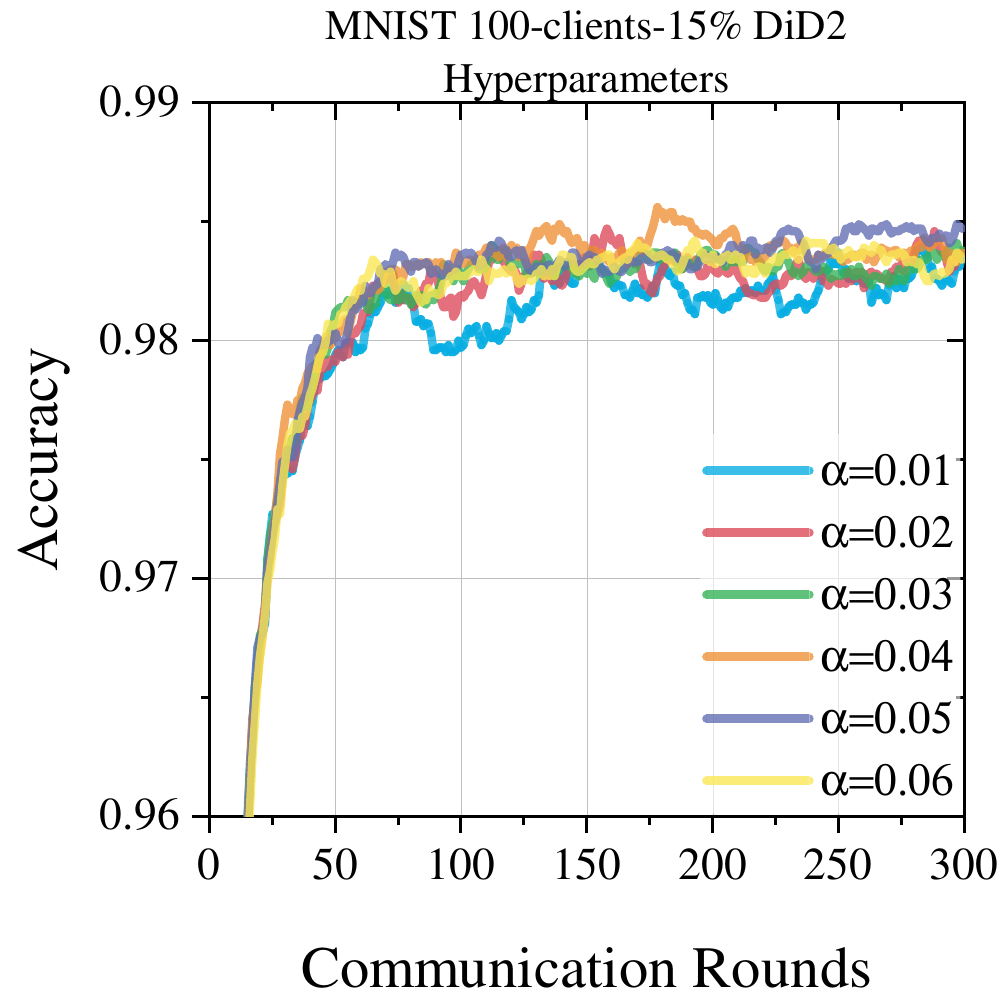}}
    
    
    \caption{Learning curves and loss curve of FedSSG's different hyper-parameters, with 100-clients-15\% settings on MNIST and on different distribution respectively.}
    \label{fig:MP15H}
\end{figure*}

\begin{figure*}[h!]
    \centering
    \subfloat[]{\includegraphics[width=0.23\textwidth]{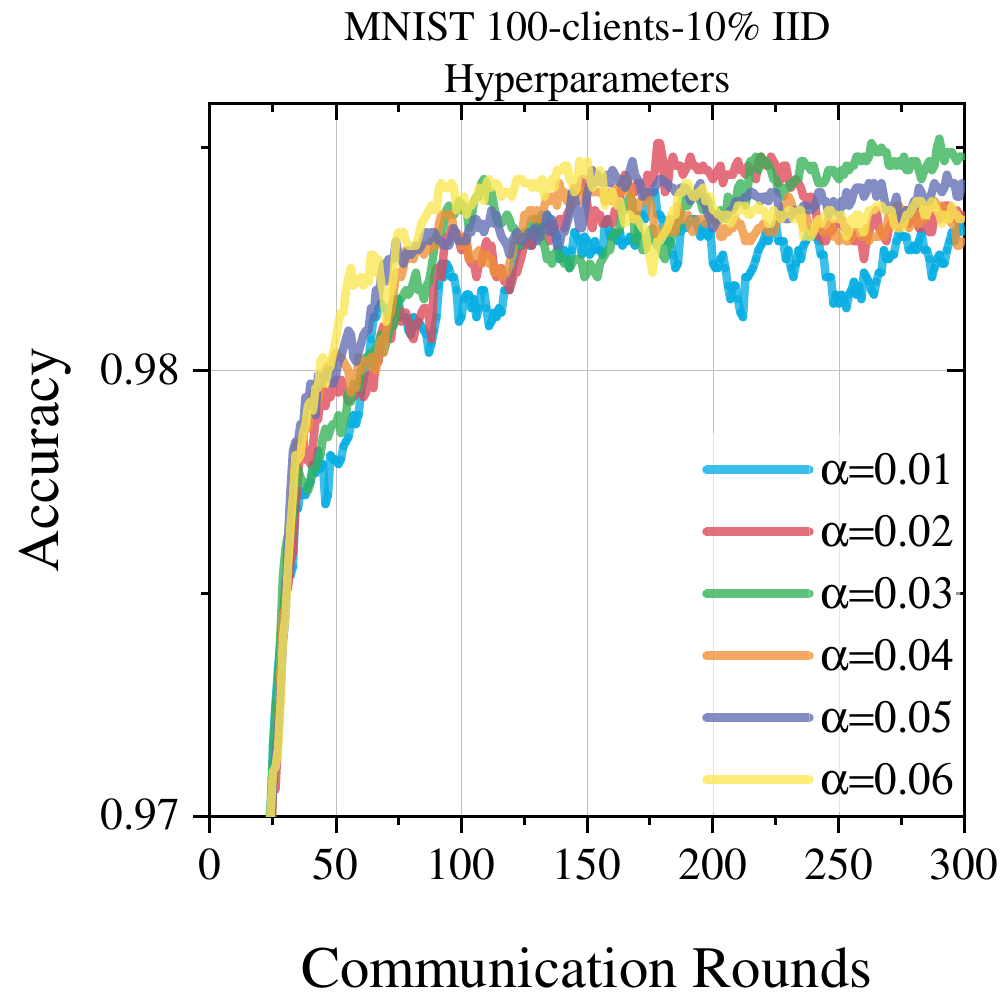}} \hfill
    \subfloat[]{\includegraphics[width=0.23\textwidth]{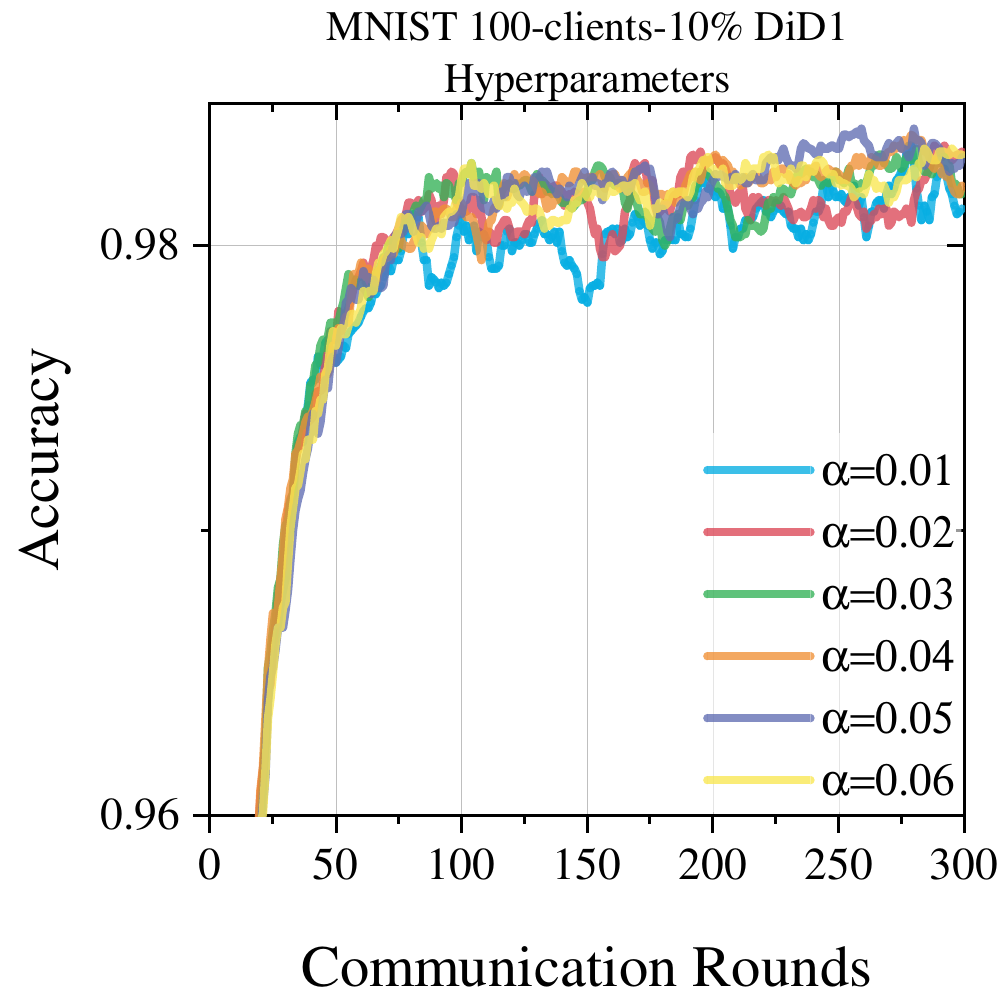}} \hfill
    \subfloat[]{\includegraphics[width=0.23\textwidth]{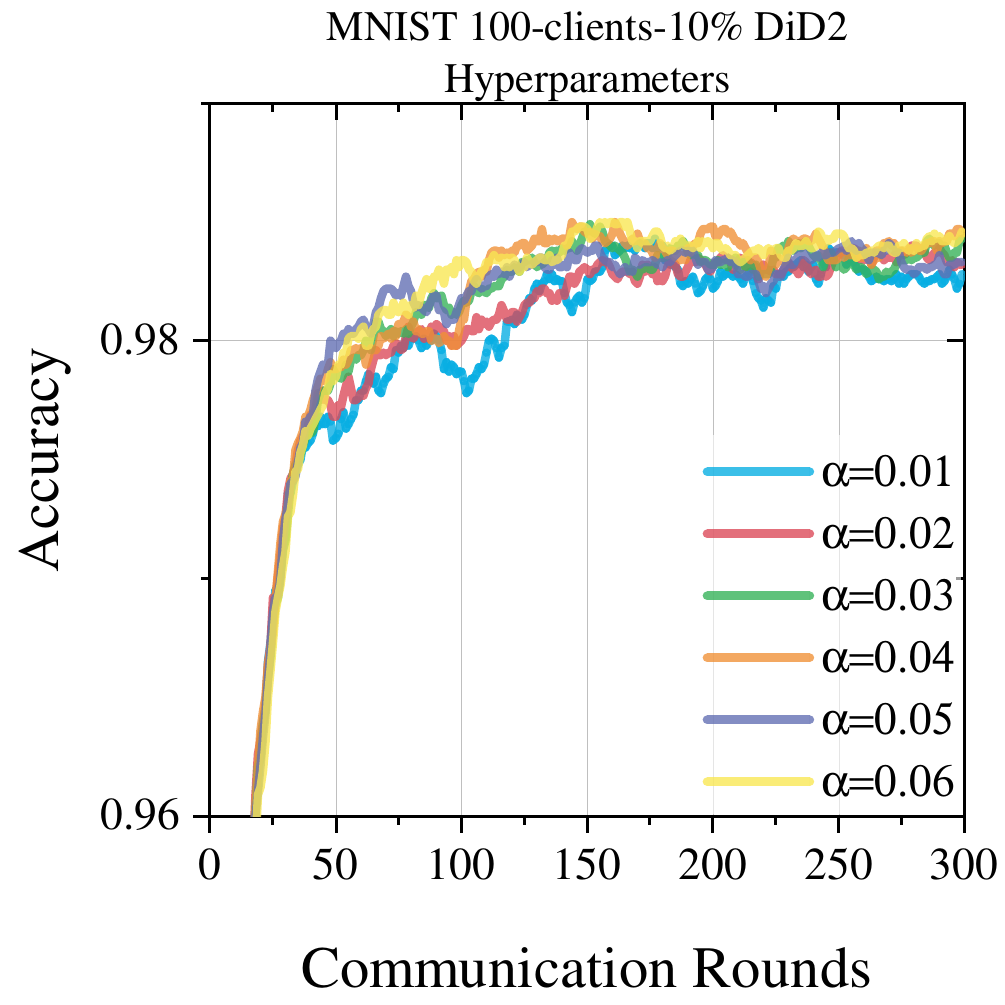}}
    
    
    \caption{Learning curves and loss curve of FedSSG's different hyper-parameters, with 100-clients-10\% settings on MNIST and on different distribution respectively.}
    \label{fig:MP10H}
\end{figure*}

\begin{figure*}[h!]
    \centering
    \subfloat[]{\includegraphics[width=0.23\textwidth]{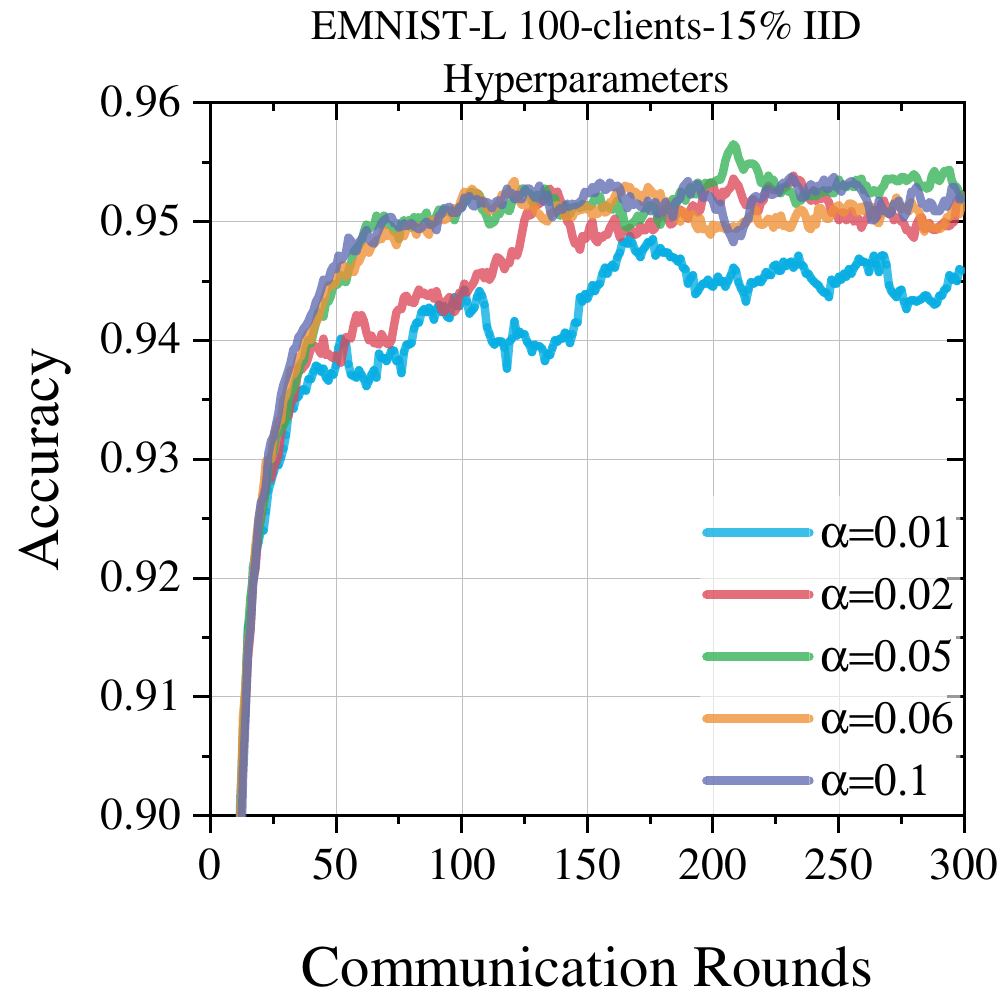}} \hfill
    \subfloat[]{\includegraphics[width=0.23\textwidth]{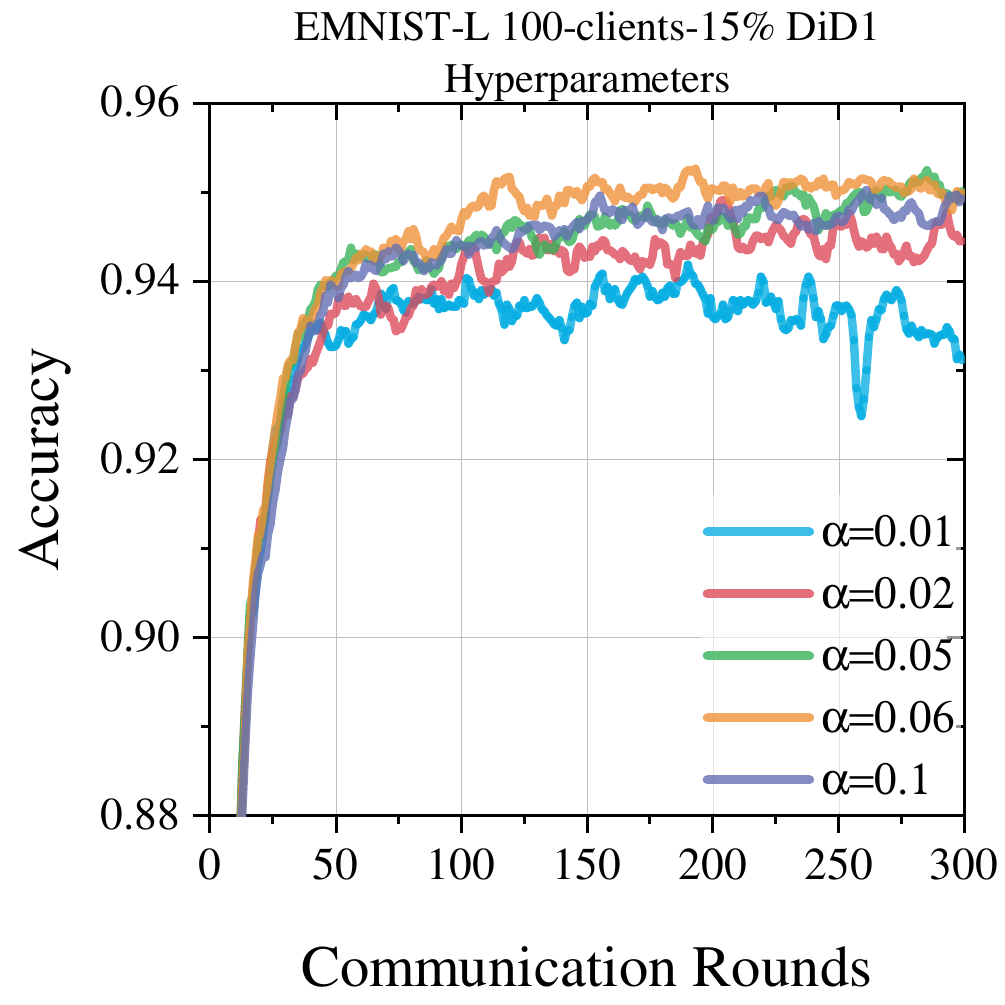}} \hfill
    \subfloat[]{\includegraphics[width=0.23\textwidth]{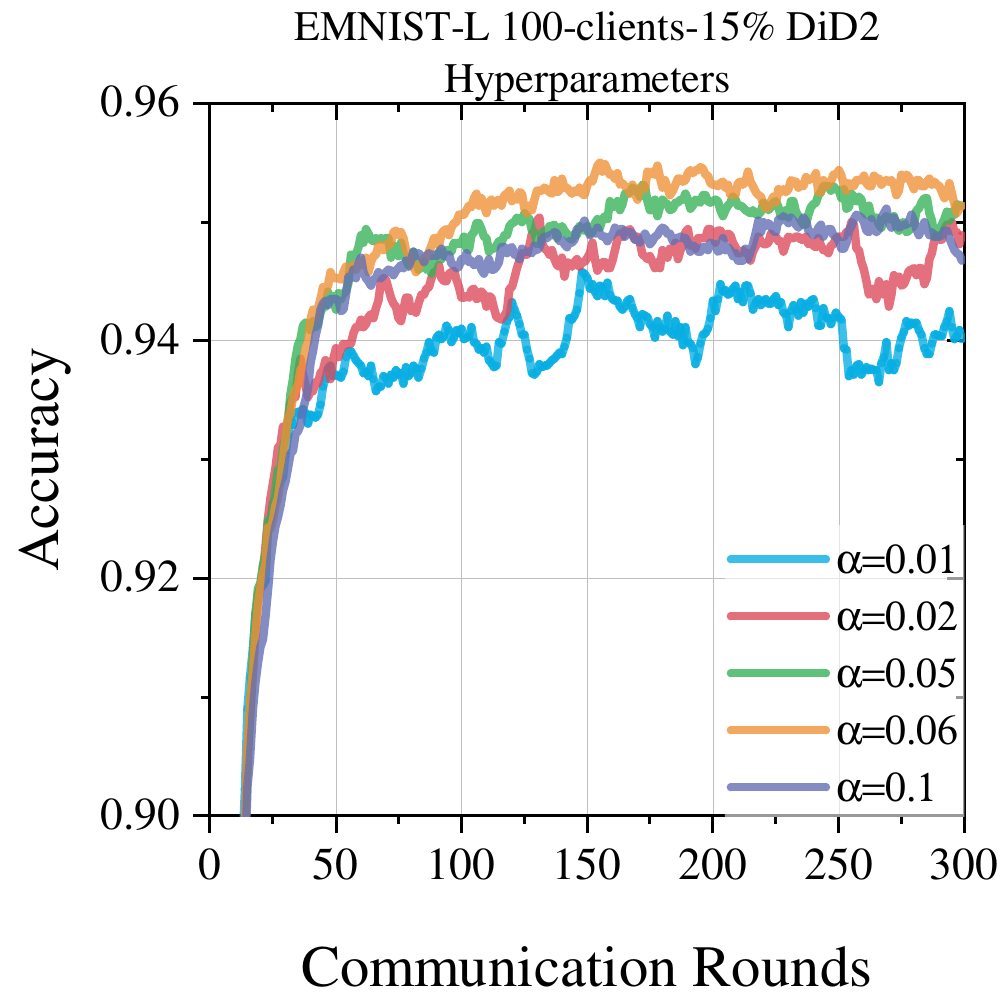}}
    
    
    \caption{Learning curves and loss curve of FedSSG's different hyper-parameters, with 100-clients-15\% settings on EMNIST-L and on different distribution respectively.}
    \label{fig:EP15H}
\end{figure*}

\begin{figure*}[h!]
    \centering
    \subfloat[]{\includegraphics[width=0.23\textwidth]{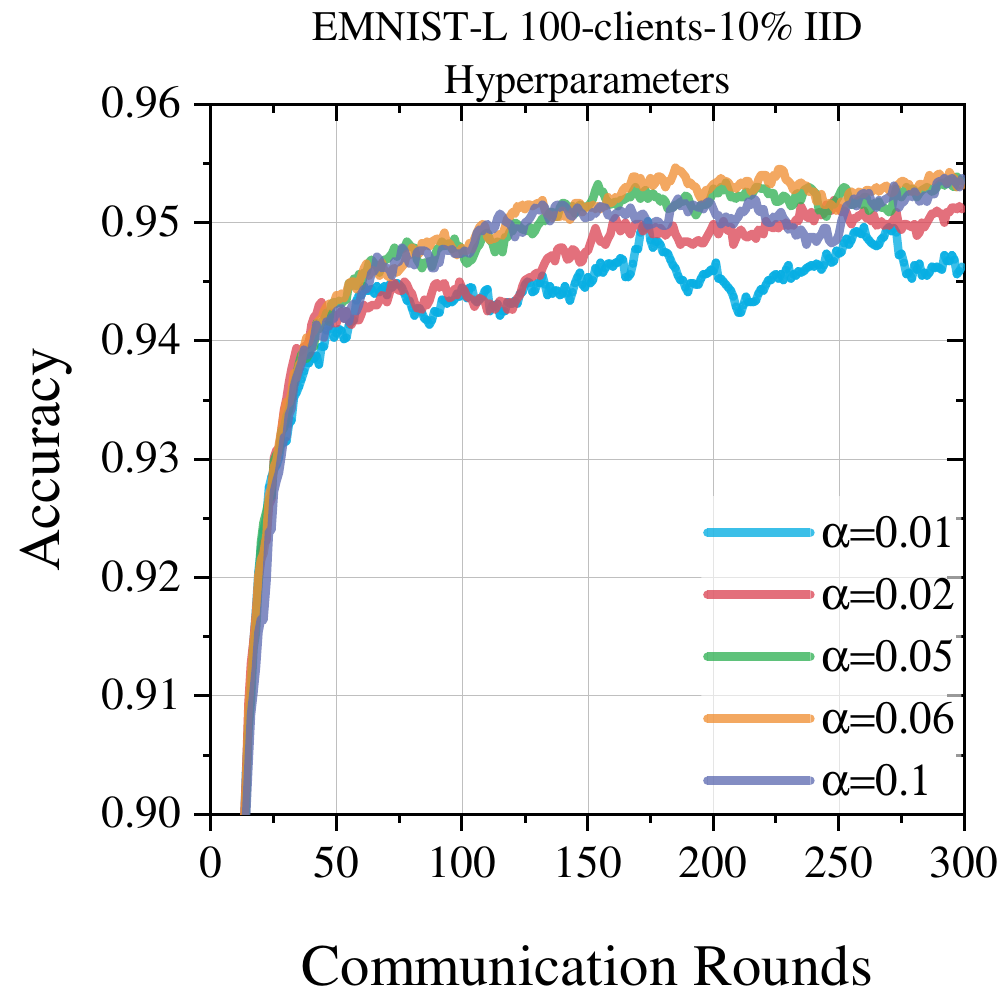}} \hfill
    \subfloat[]{\includegraphics[width=0.23\textwidth]{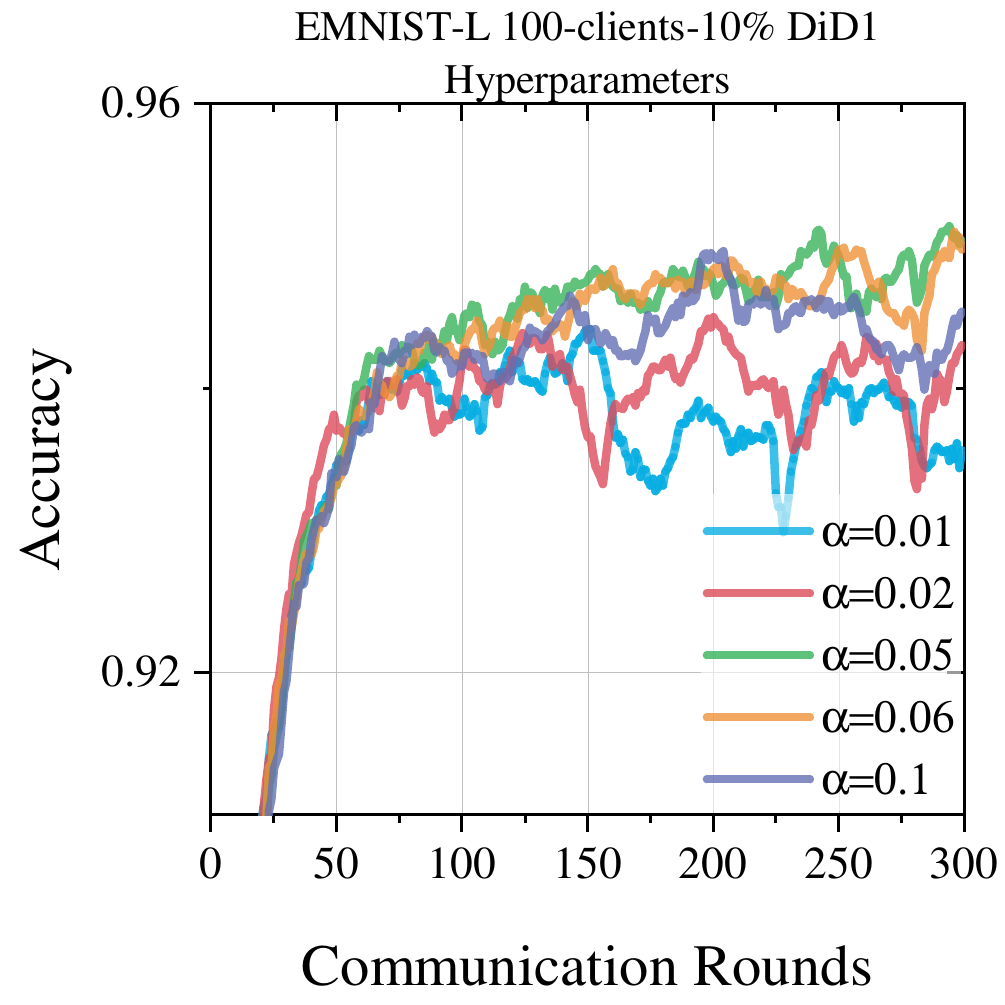}} \hfill
    \subfloat[]{\includegraphics[width=0.23\textwidth]{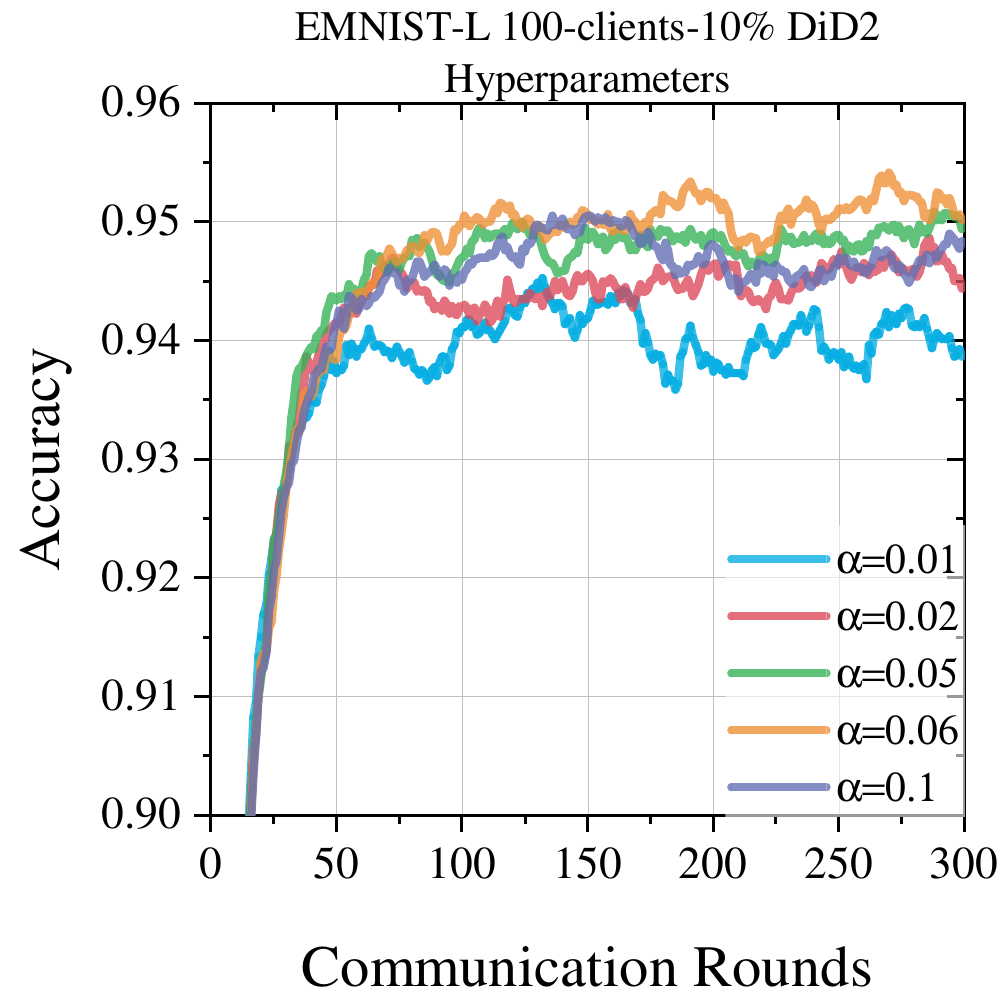}}
    
    
    \caption{Learning curves and loss curve of FedSSG's different hyper-parameters, with 100-clients-10\% settings on EMNIST-L and on different distribution respectively.}
    \label{fig:EP10H}
\end{figure*}

\begin{figure*}[h!]
    \centering
    \subfloat[]{\includegraphics[width=0.23\textwidth]{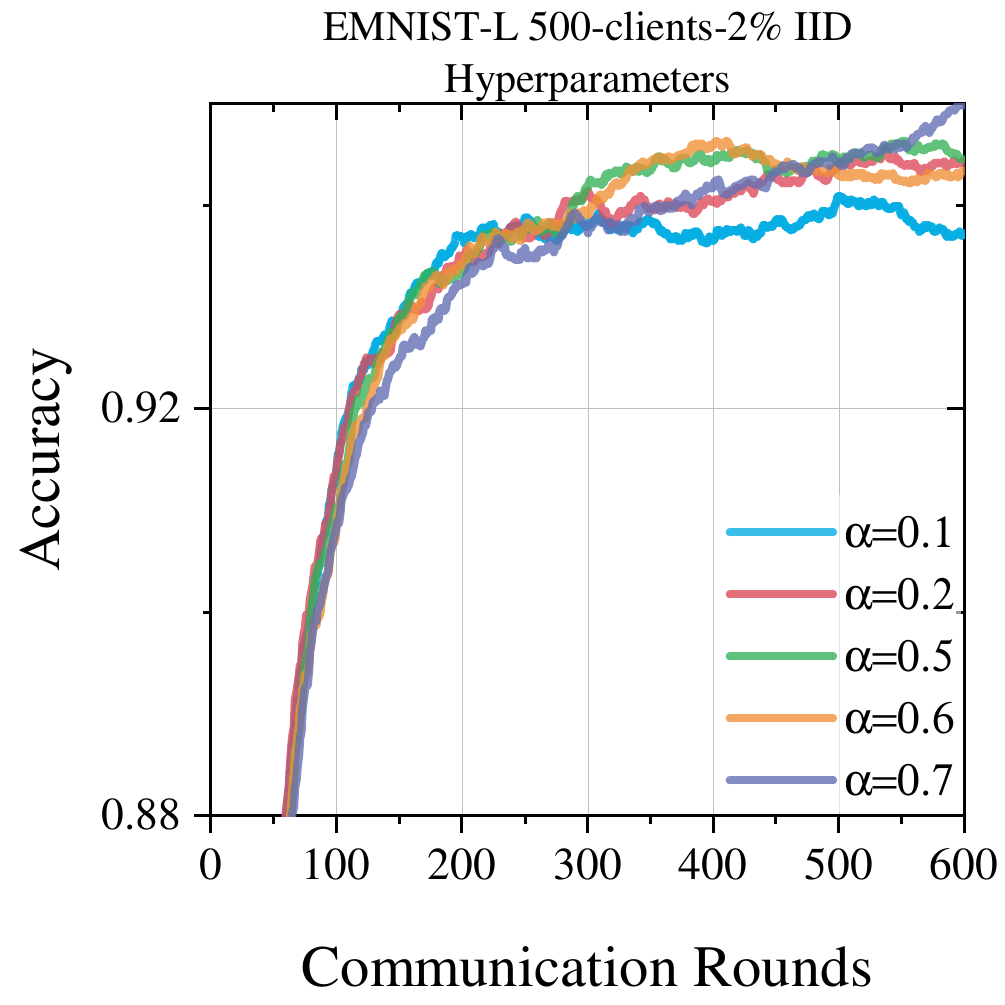}} \hfill
    \subfloat[]{\includegraphics[width=0.23\textwidth]{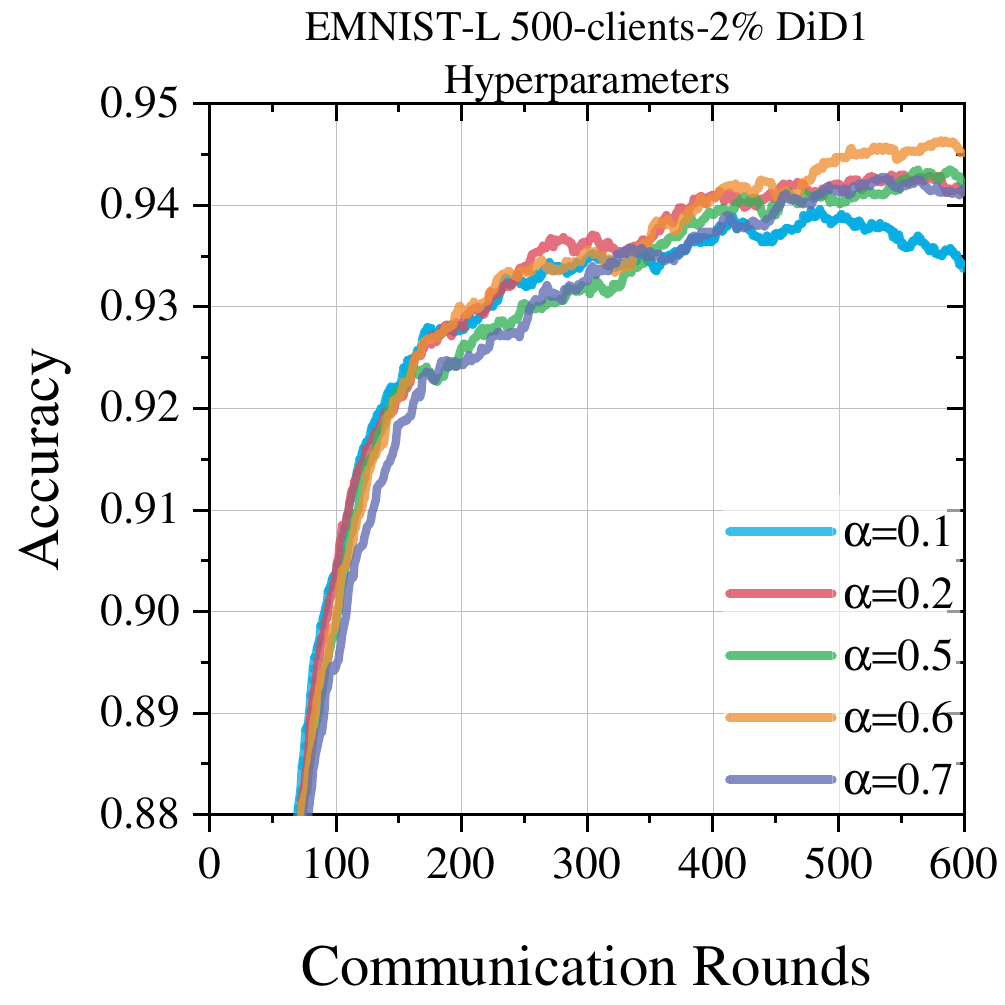}} \hfill
    \subfloat[]{\includegraphics[width=0.23\textwidth]{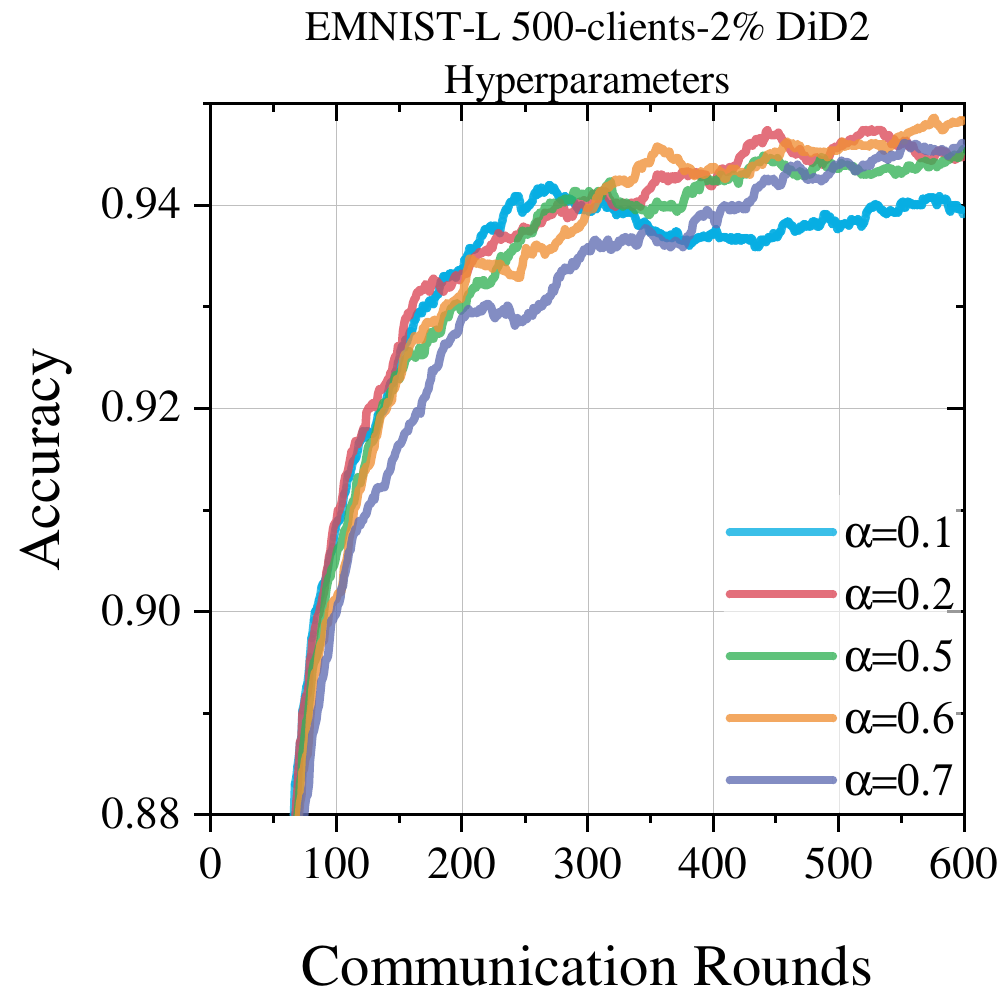}}
    
    
    \caption{Learning curves and loss curve of FedSSG's different hyper-parameters, with 500-clients-2\% settings on EMNIST-L and on different distribution respectively.}
    \label{fig:EP2H}
\end{figure*}

\begin{figure*}[h!]
    \centering
    \subfloat[]{\includegraphics[width=0.23\textwidth]{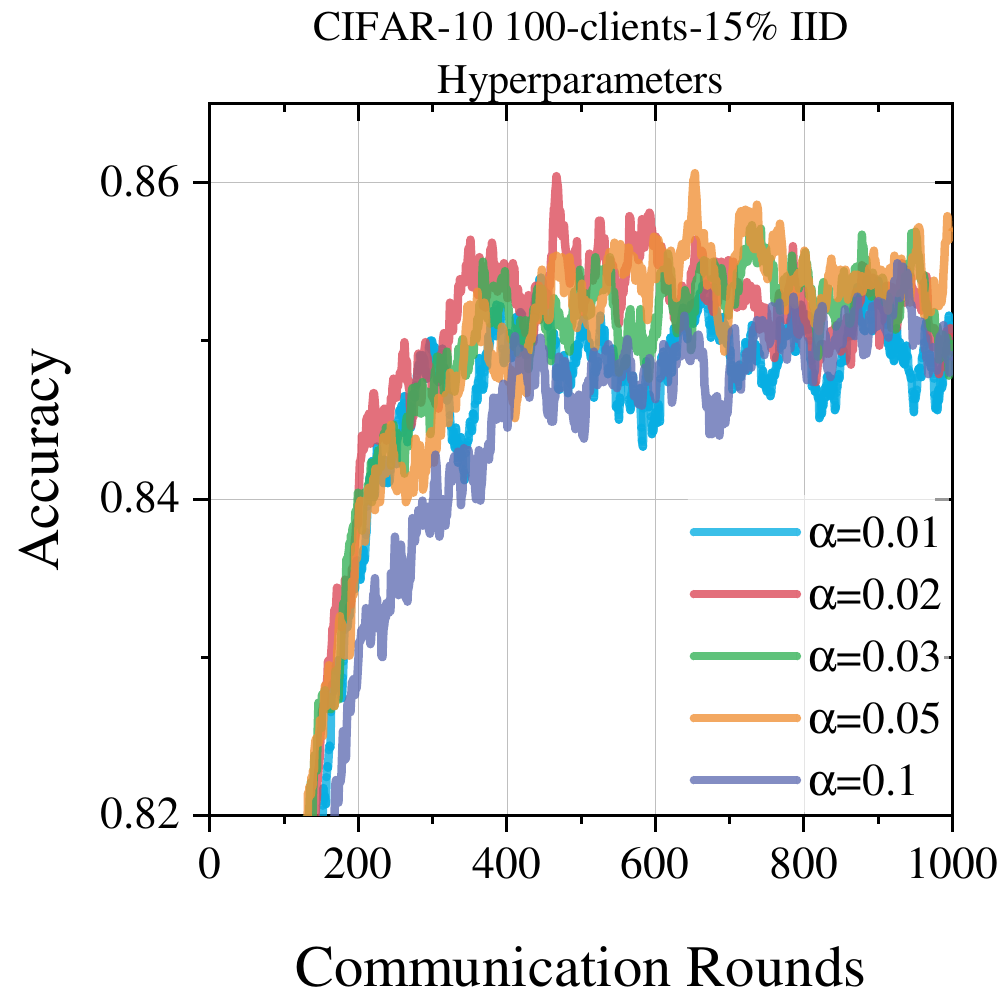}} \hfill
    \subfloat[]{\includegraphics[width=0.23\textwidth]{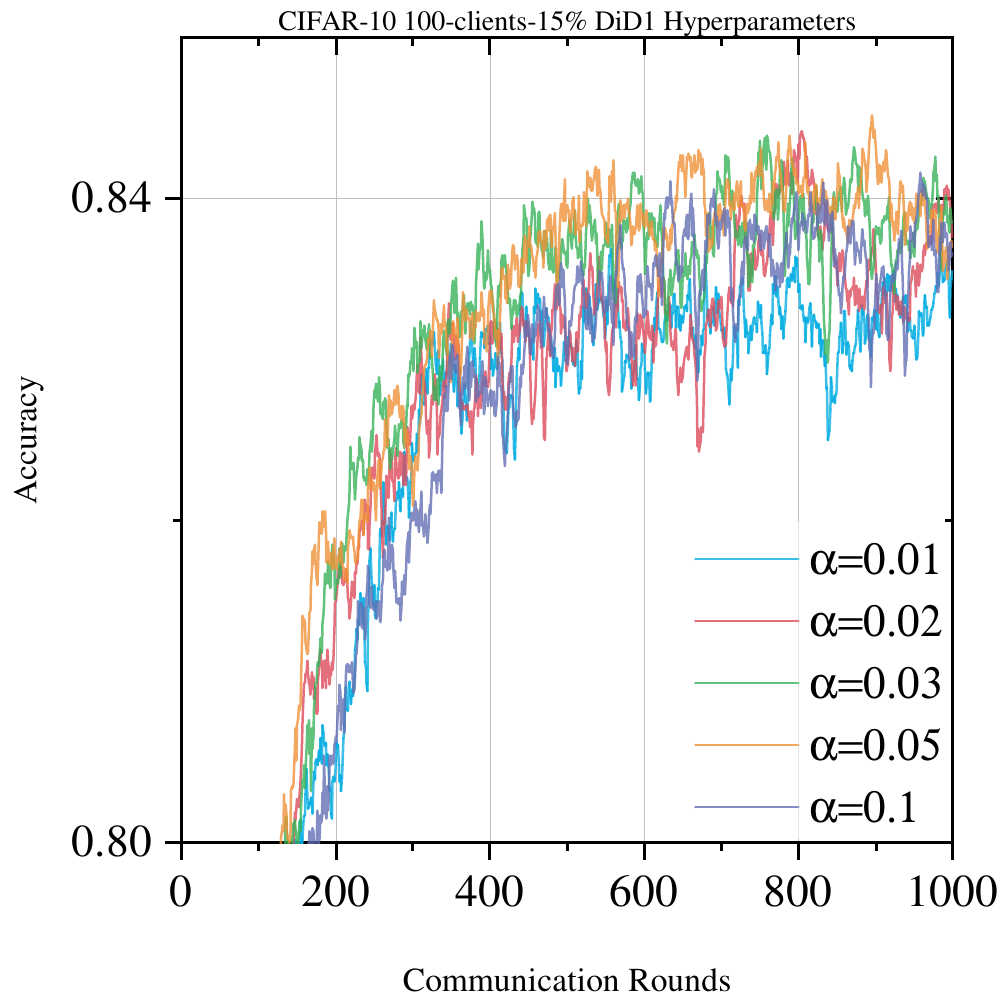}} \hfill
    \subfloat[]{\includegraphics[width=0.23\textwidth]{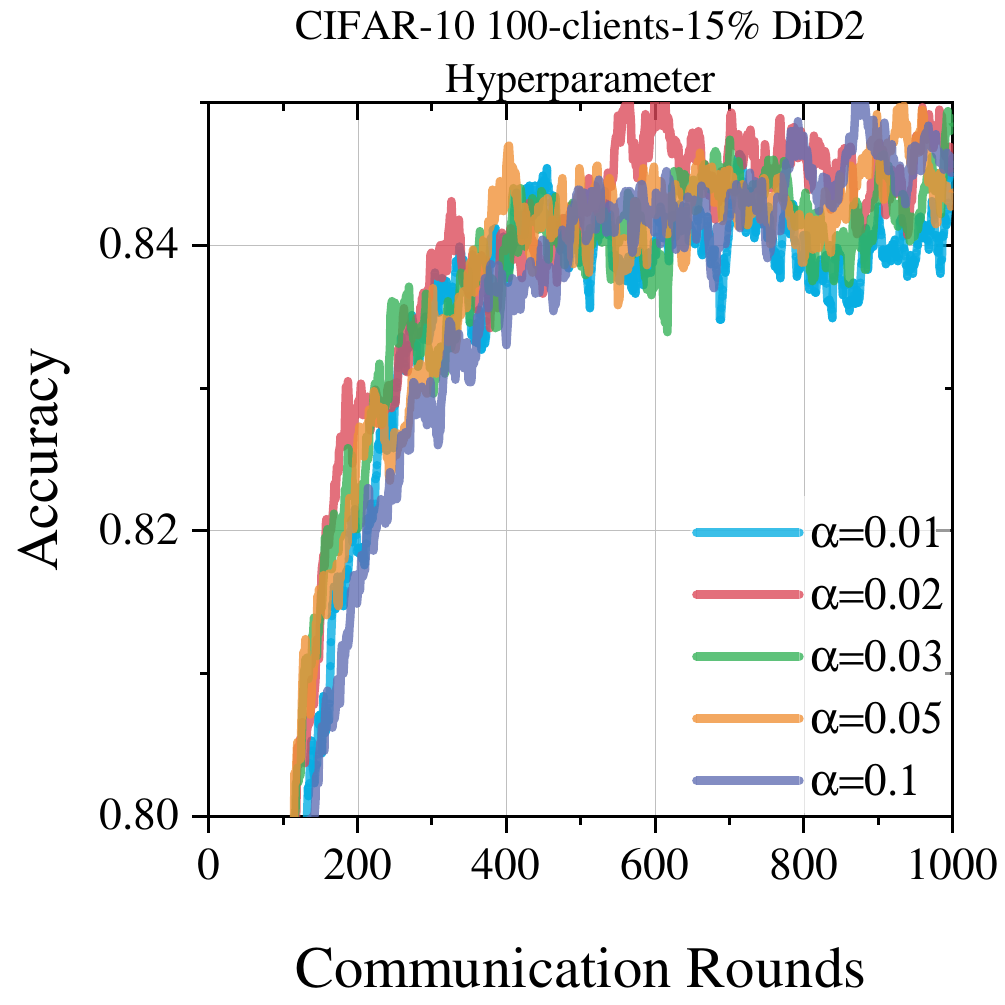}}
    
    
    \caption{Learning curves and loss curve of FedSSG's different hyper-parameters, with 100-clients-15\% settings on CIFAR-10 and on different distribution respectively.}
    \label{fig:C10P15H}
\end{figure*}

\begin{figure*}[h!]
    \centering
    \subfloat[]{\includegraphics[width=0.23\textwidth]{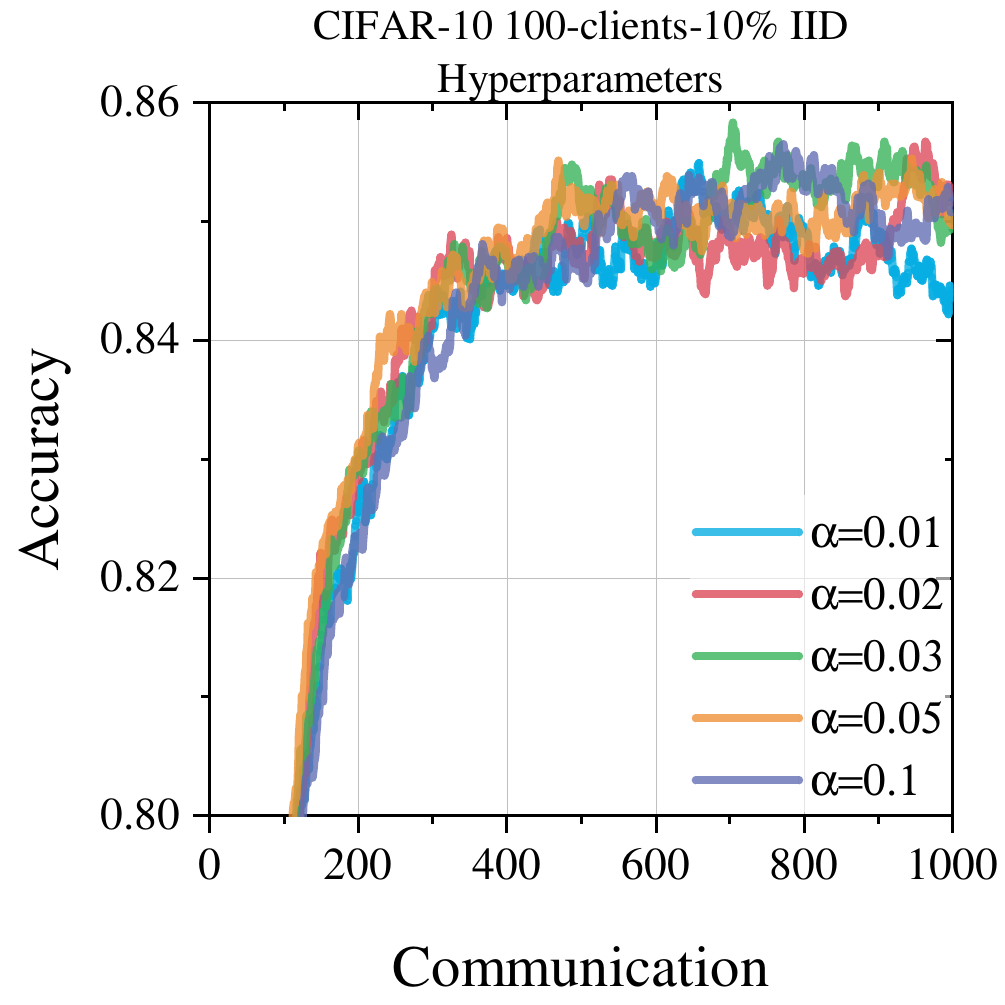}} \hfill
    \subfloat[]{\includegraphics[width=0.23\textwidth]{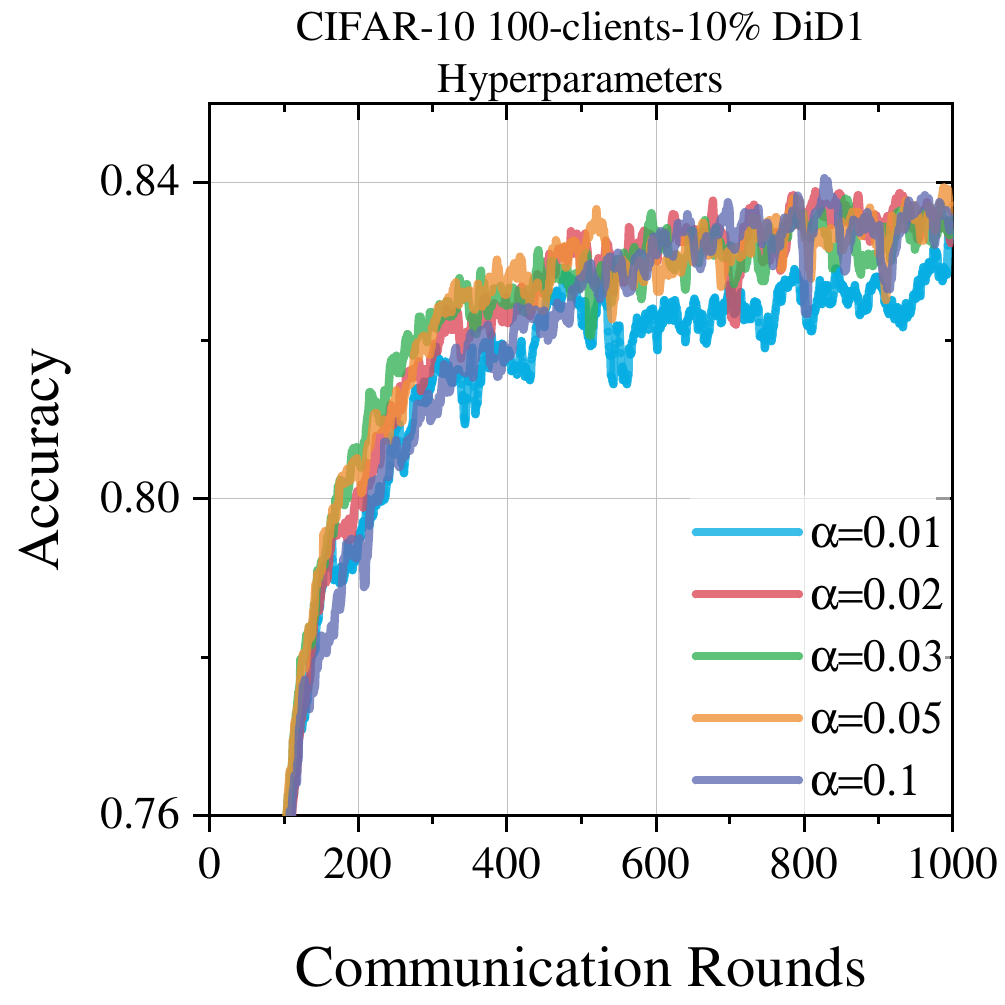}} \hfill
    \subfloat[]{\includegraphics[width=0.23\textwidth]{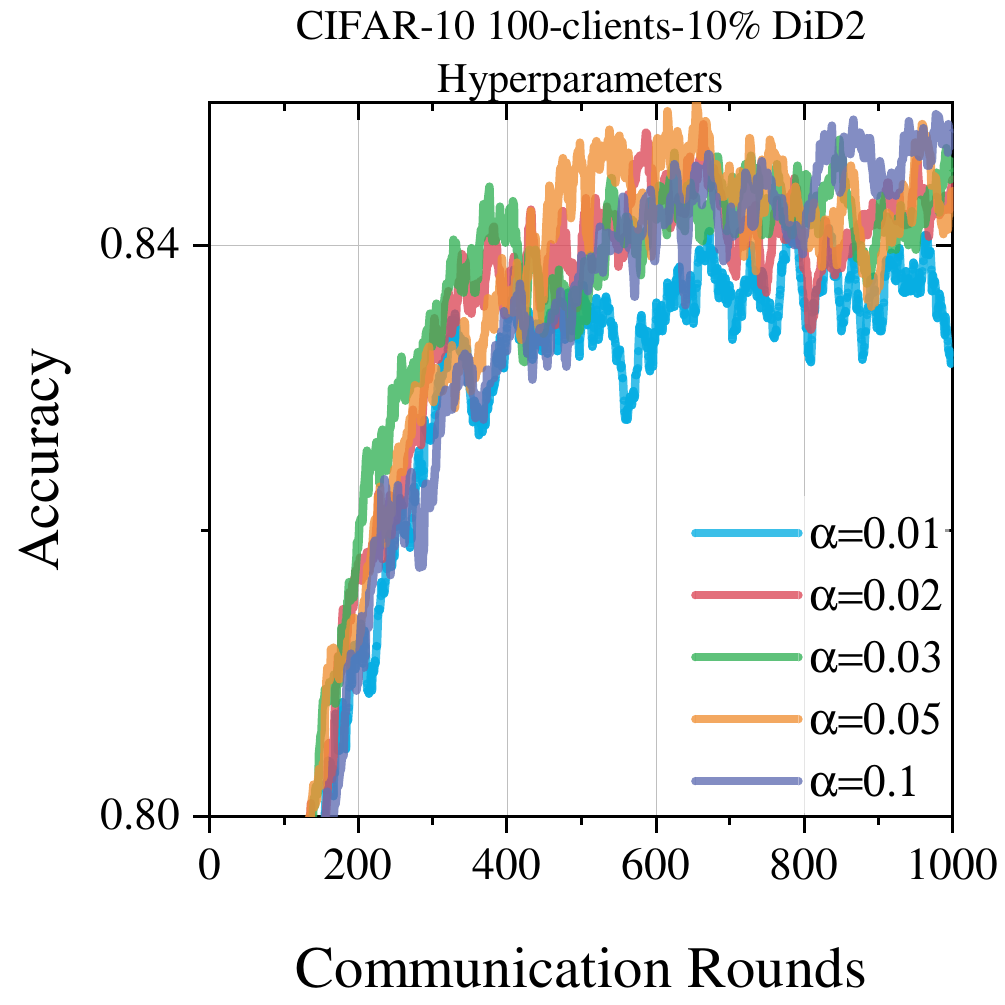}}
    
    
    \caption{Learning curves and loss curve of FedSSG's different hyper-parameters, with 100-clients-10\% settings on CIFAR-10 and on different distribution respectively.}
    \label{fig:C10P10H}
\end{figure*}

\begin{figure*}[h!]
    \centering
    \subfloat[]{\includegraphics[width=0.23\textwidth]{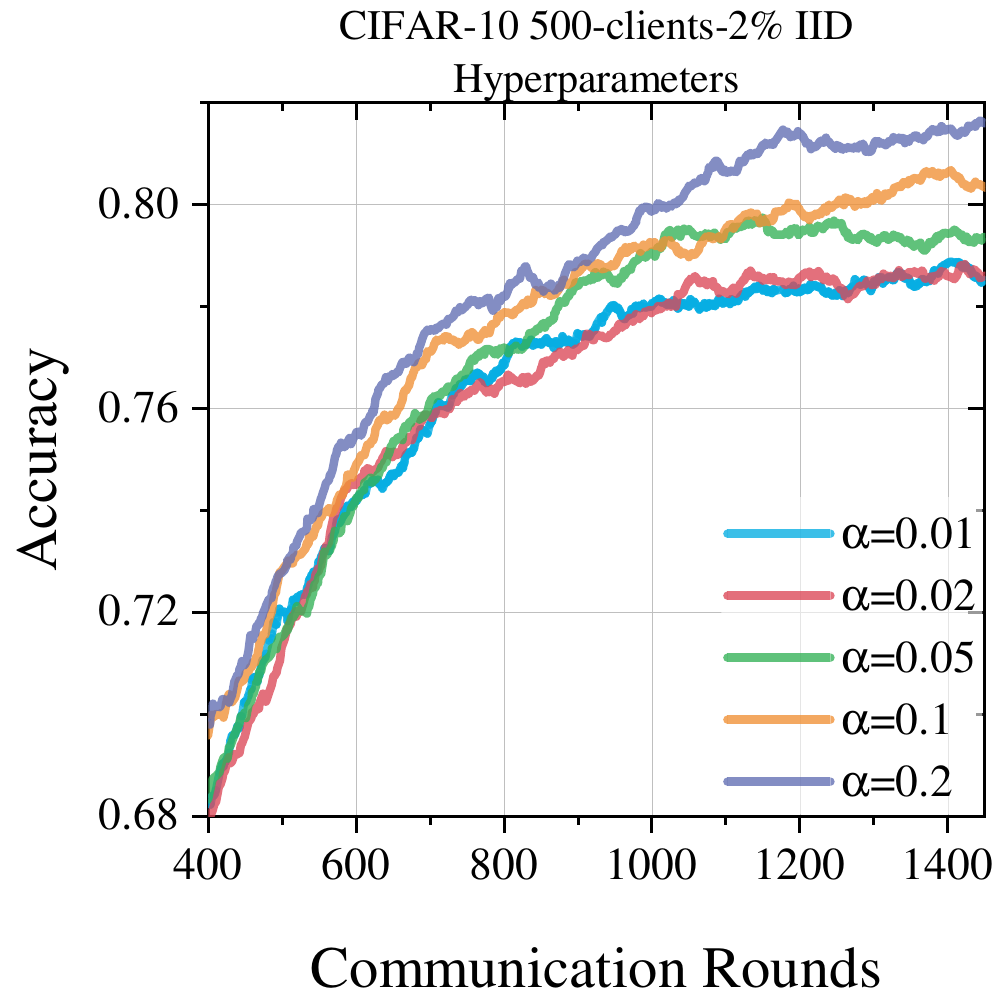}} \hfill
    \subfloat[]{\includegraphics[width=0.23\textwidth]{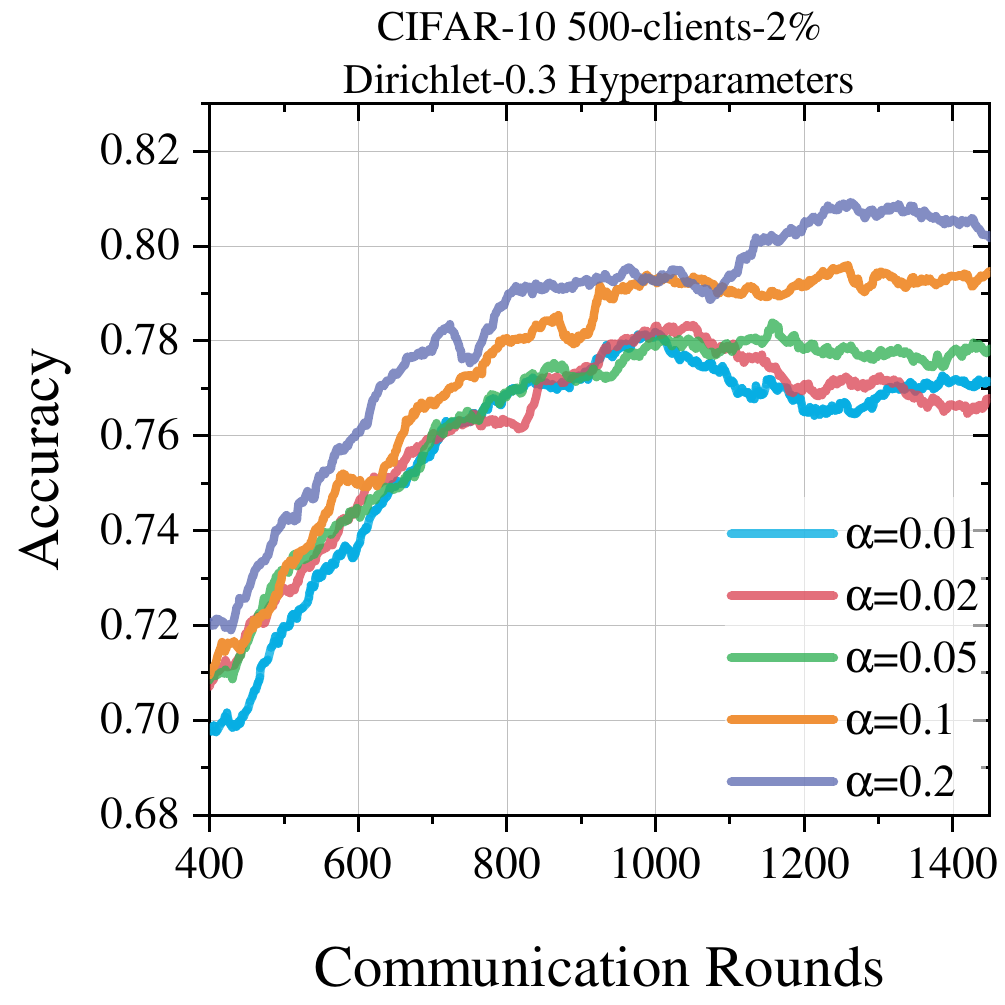}} \hfill
    \subfloat[]{\includegraphics[width=0.23\textwidth]{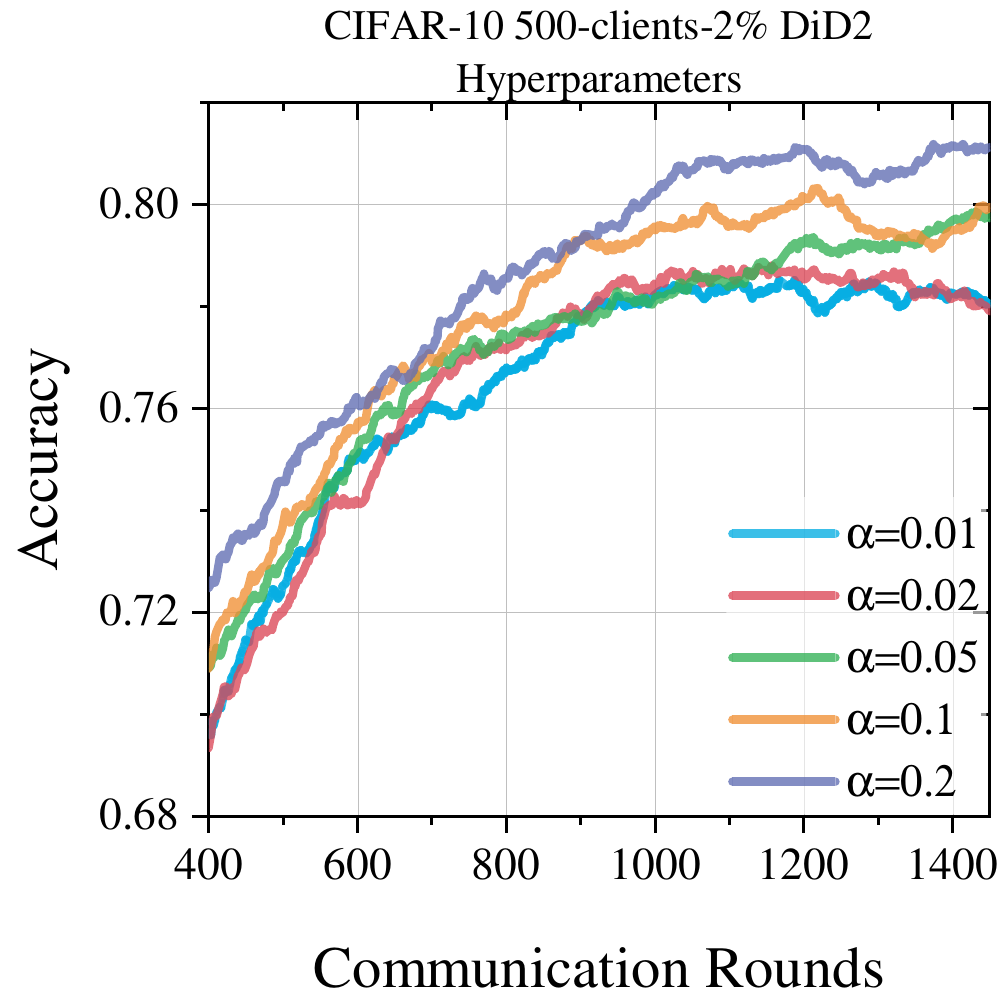}}
    
    
    \caption{Learning curves and loss curve of FedSSG's different hyper-parameters, with 500-clients-2\% settings on CIFAR-10 and on different distribution respectively.}
    \label{fig:C10P2H}
\end{figure*}

\begin{figure*}
    \centering
    \subfloat[]{\includegraphics[width=0.23\textwidth]{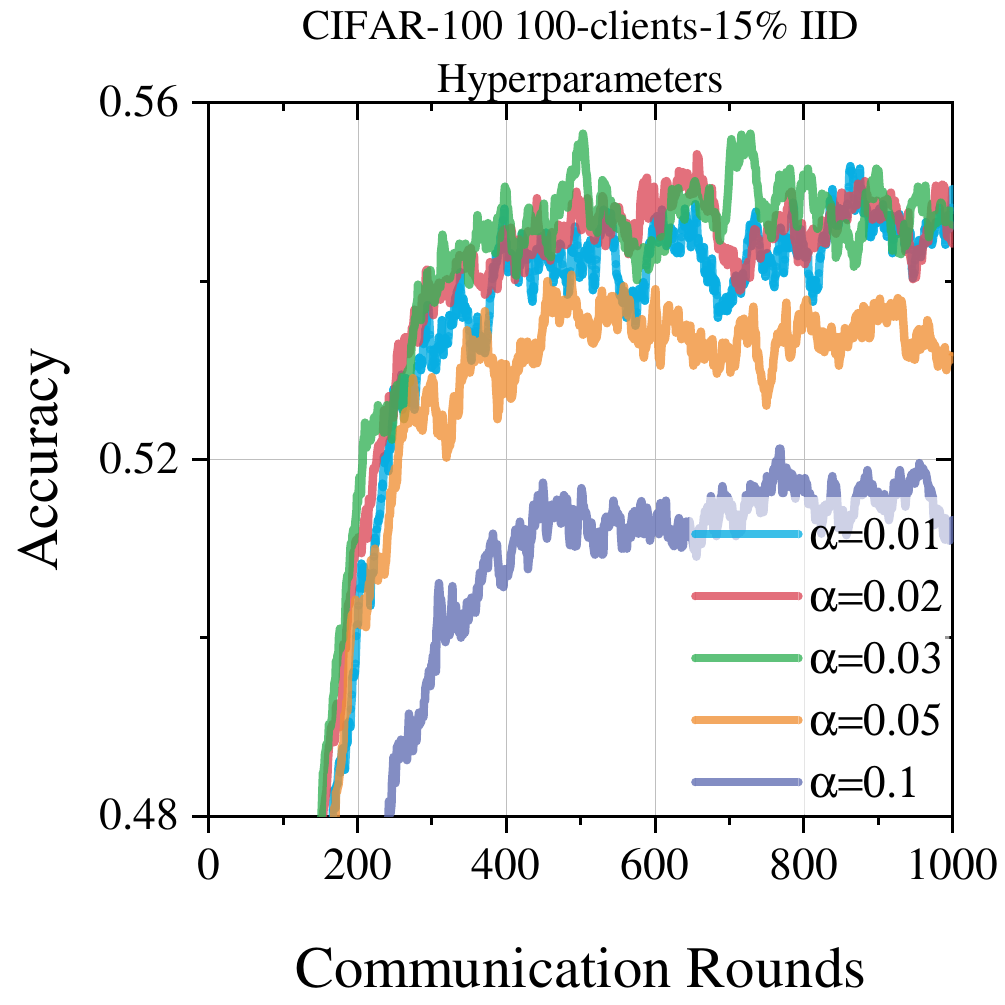}} \hfill
    \subfloat[]{\includegraphics[width=0.23\textwidth]{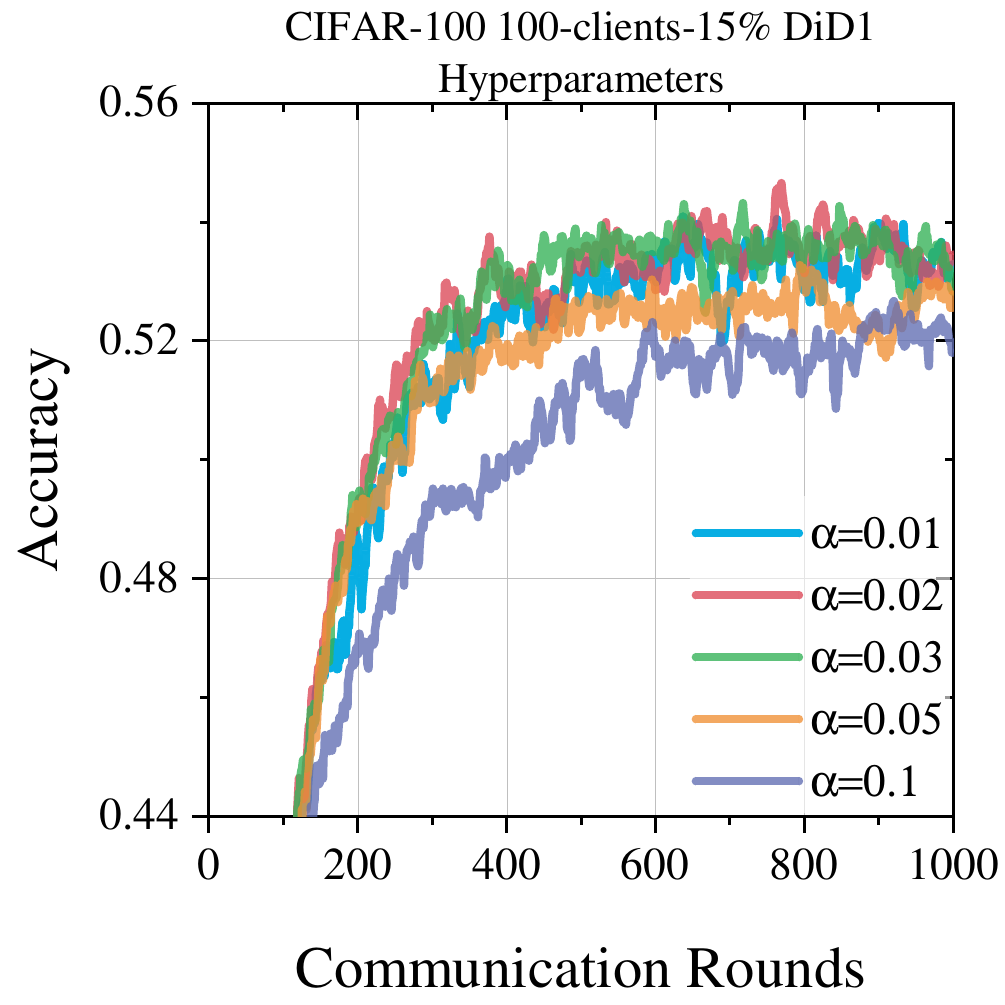}} \hfill
    \subfloat[]{\includegraphics[width=0.23\textwidth]{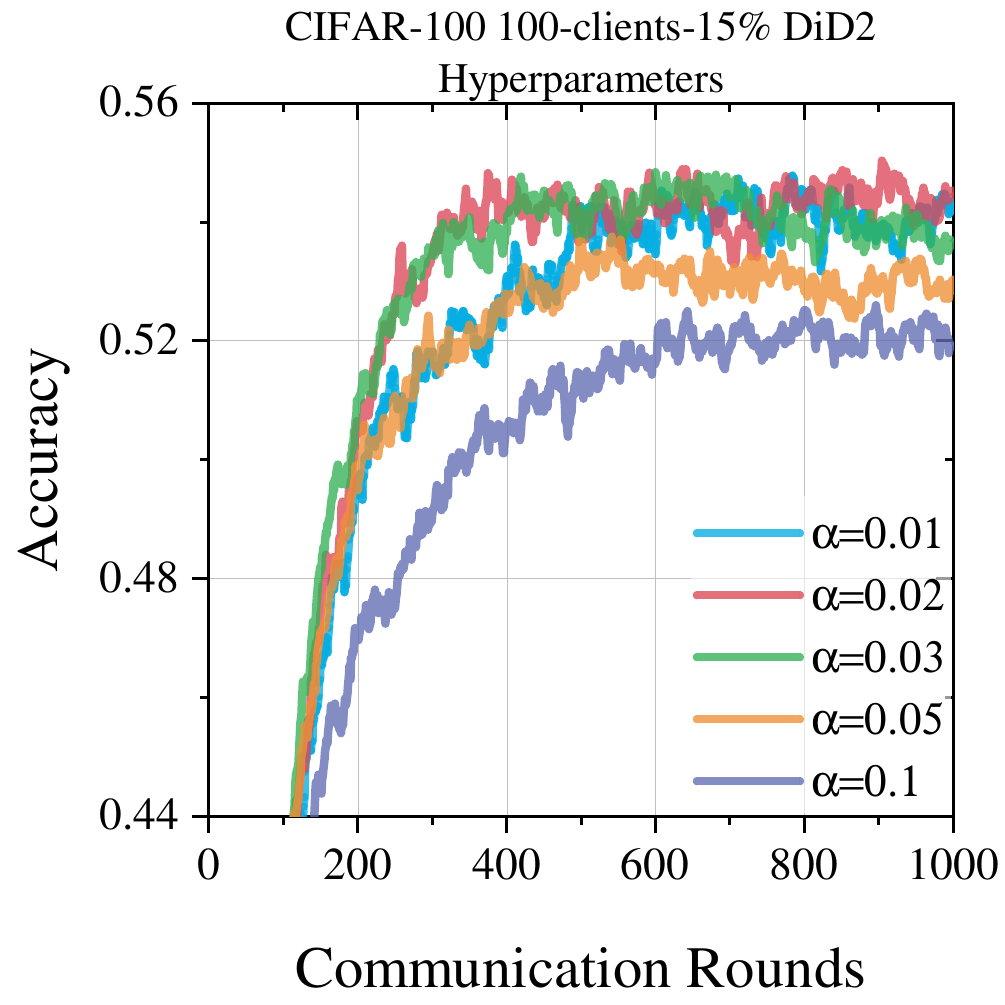}}


    \caption{Learning curves and loss curve of FedSSG's different hyper-parameters, with 100-clients-15\% settings on CIFAR-100 and on different distribution respectively.}
    \label{fig:C100P15H}
\end{figure*}

\begin{figure*}
    \centering
    \subfloat[]{\includegraphics[width=0.23\textwidth]{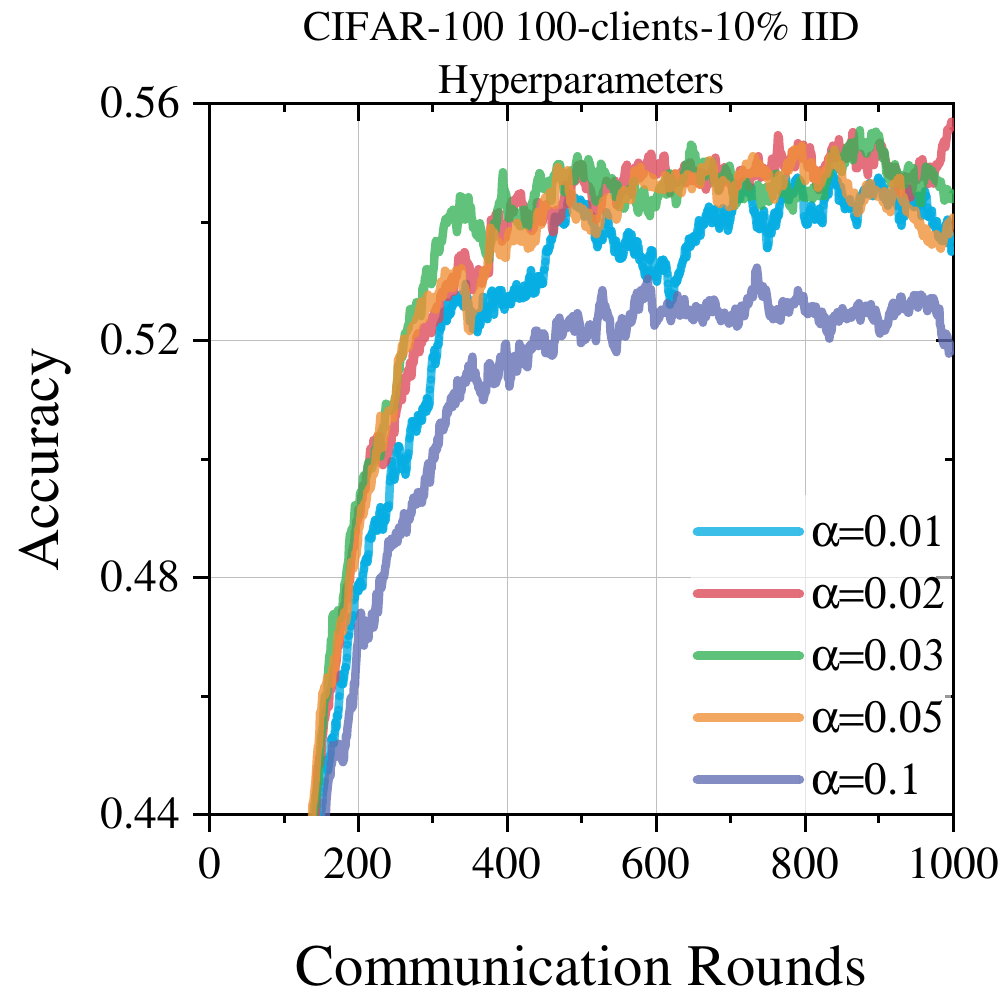}} \hfill
    \subfloat[]{\includegraphics[width=0.23\textwidth]{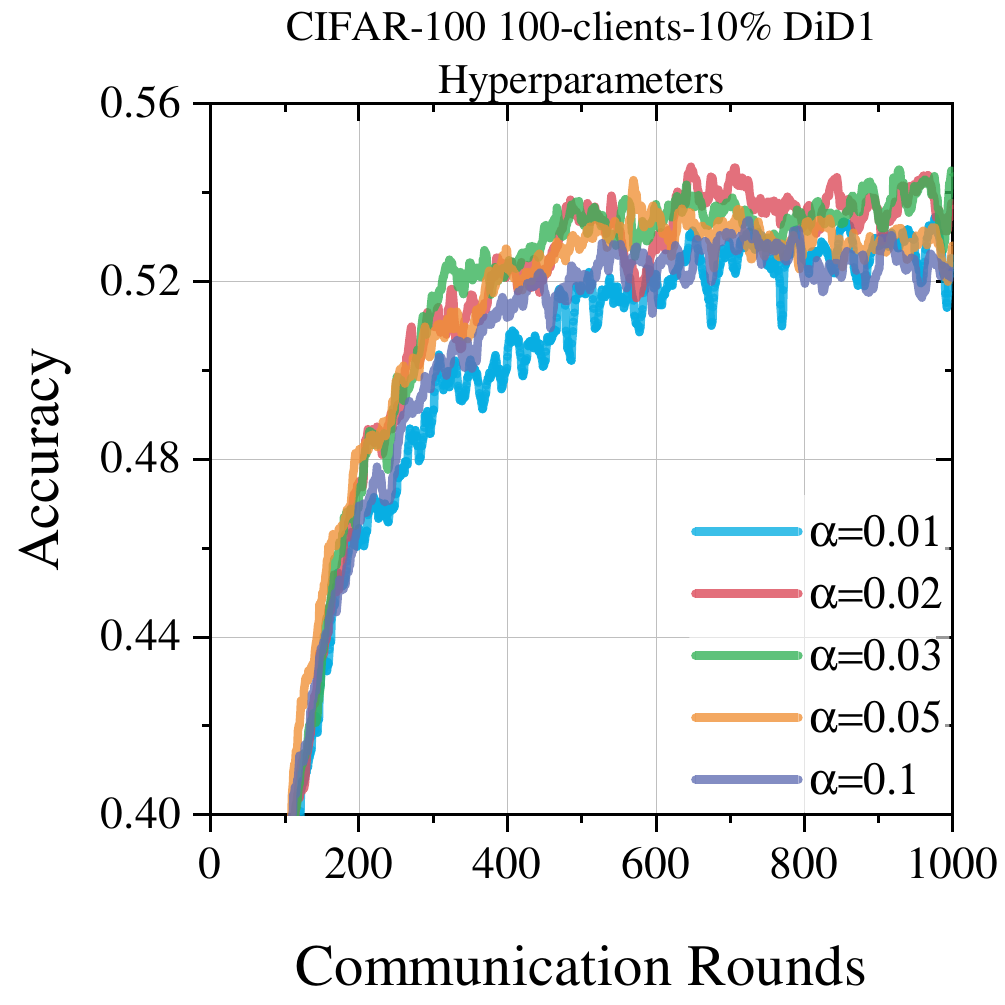}} \hfill
    \subfloat[]{\includegraphics[width=0.23\textwidth]{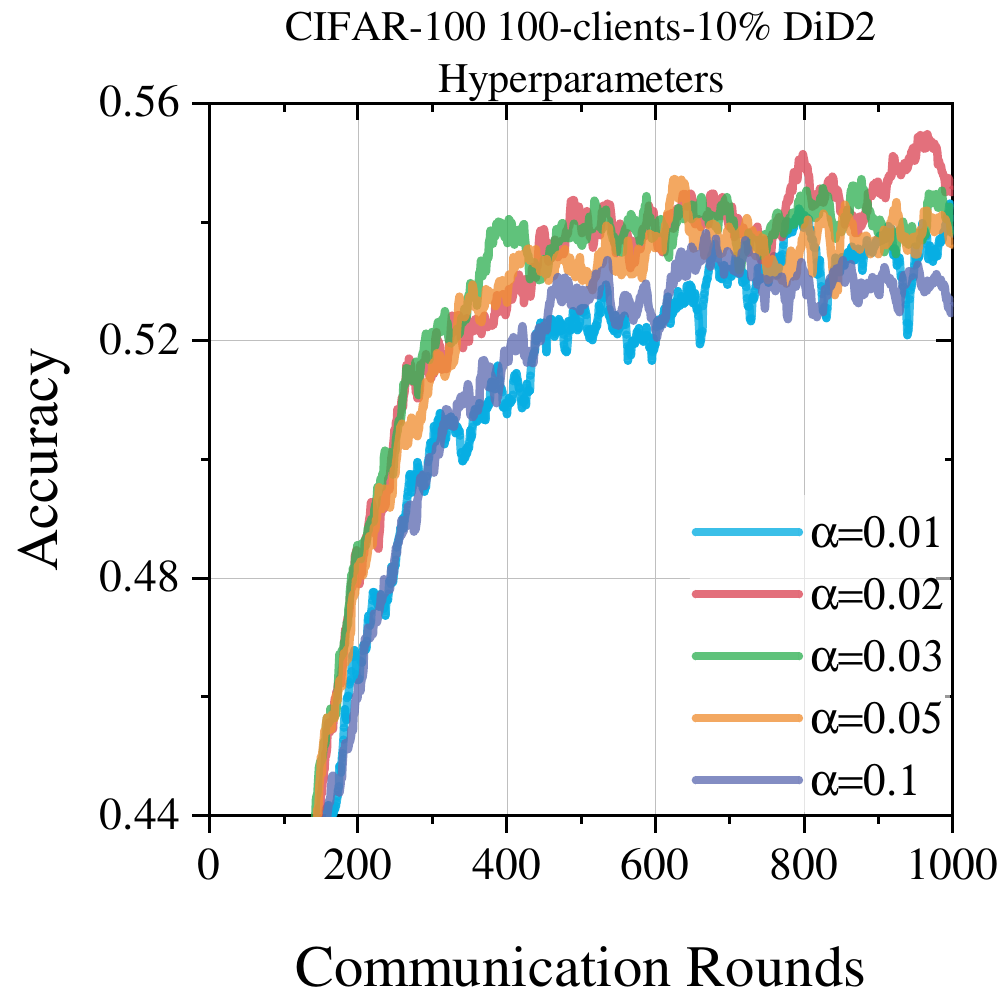}}


    \caption{Learning curves and loss curve of FedSSG's different hyper-parameters, with 100-clients-10\% settings on CIFAR-100 and on different distribution respectively.}
    \label{fig:C100P10H}
\end{figure*}

\begin{figure*}
    \centering
    \subfloat[]{\includegraphics[width=0.23\textwidth]{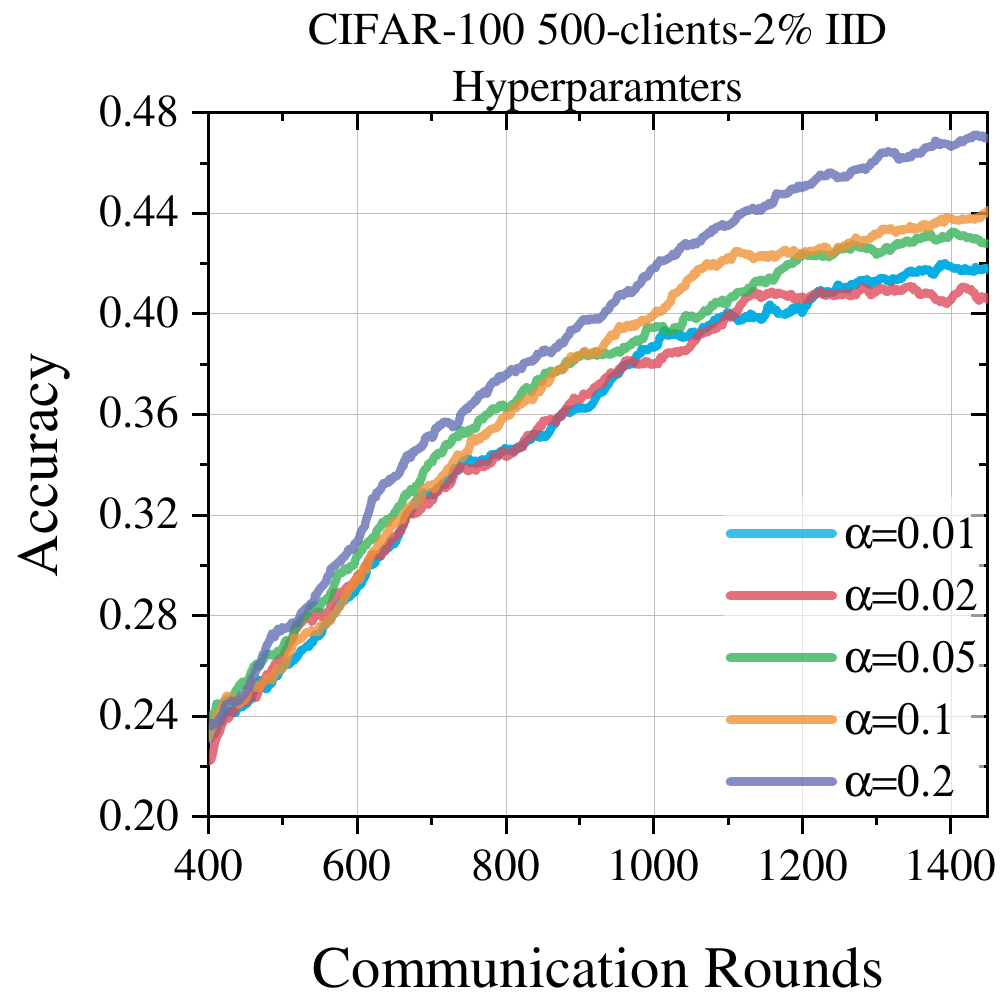}} \hfill
    \subfloat[]{\includegraphics[width=0.23\textwidth]{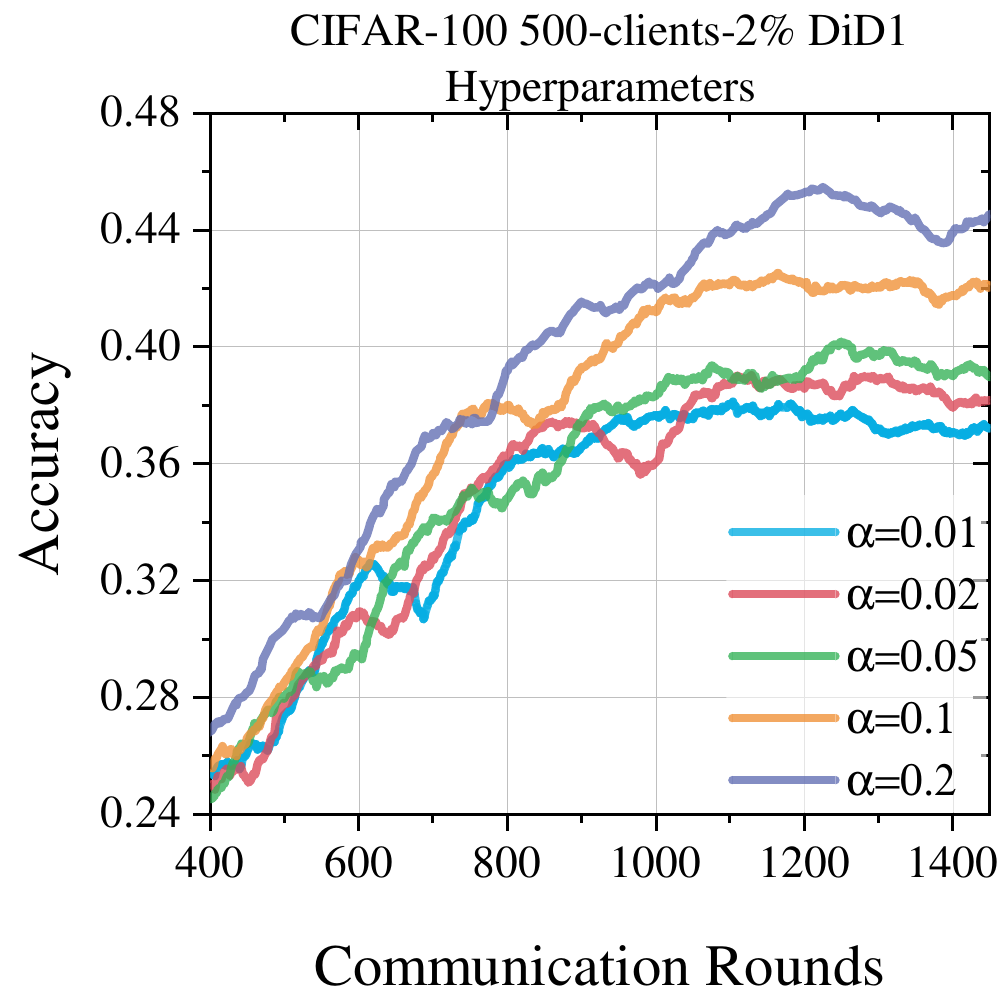}} \hfill
    \subfloat[]{\includegraphics[width=0.23\textwidth]{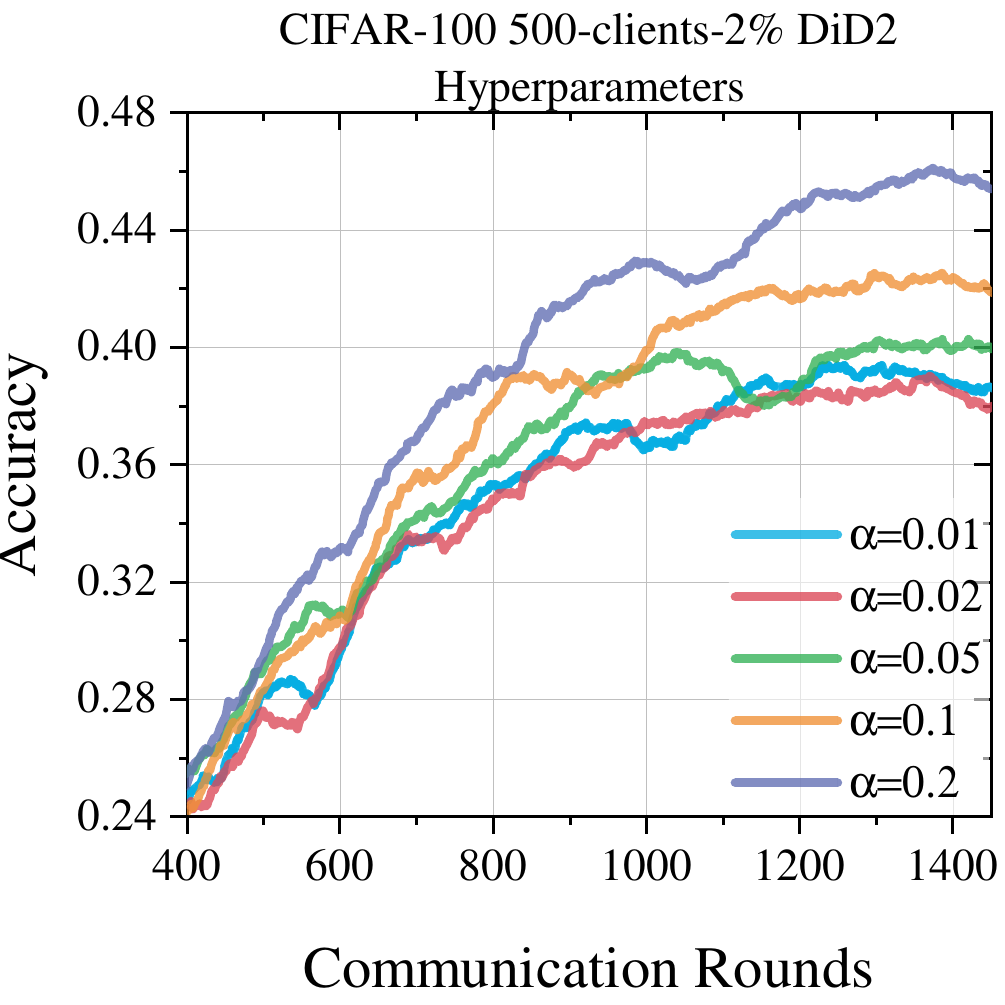}}


    \caption{Learning curves and loss curve of FedSSG's different hyper-parameters, with 500-clients-2\% settings on CIFAR-100 and on different distribution respectively.}
    \label{fig:C100P2H}
\end{figure*}

%% file: appendixB.tex
\section{Discussion and Proof}\label{appendix:b}

\begin{algorithm}[t]
	\caption{FedSSG}
	\label{alg:FedSSG}
	\textbf{Input}: Random initial global model parameter $\omega^0$.\\
	\textbf{Parameters}: Traning round $T$, the number of clients $K$, initial local drift variables as all zero matrix, the learning rate $\eta$, the number of local traing batches $N$, local counter $c_i=0$.\\
	\textbf{Output}: The trained global model $\omega^{T-1}$.\\
	\begin{algorithmic}[1]
		\STATE Initial all parameters.
		\WHILE{$t<T$}
		\STATE Sample the selected client set $\mathcal{S}^t\in\left\{1,2,...,N\right\}$
		\WHILE{each client $i\in{\mathcal{S}^t}$ in parallel}
		\STATE Set the local model parameter $\theta_i^t=\omega^{t-1}$
		\WHILE{$e<E$}
		\colorlet{shadecolor}{yellow!40}
		\STATE Update the local model parameters\\ $\theta_i^t=\theta_i^t-\eta\frac{\partial{f_i(\theta_i^t)}}{\partial{\theta_i^t}}$
		\ENDWHILE
		\STATE Update the local counter $c_i=c_i+1$
		\STATE Update the local gradient drift $\Delta{\theta_i^t}=\theta_i^t-\omega^{t-1}$
		\colorlet{shadecolor}{yellow!40}
		\STATE Update the local drift with correction $h_i=h_i+(\xi_i+\beta)\Delta{\theta^t_i}$, where $\xi_i=\frac{c_i}{T|\mathcal{S}^t|/N}$
		\ENDWHILE
		\STATE Update the global model 
		$\omega=\frac{1}{N}\sum_{i\in{\mathcal{S}^t}} ({\theta^t_i}+h_i)$
		\ENDWHILE
		\STATE \textbf{return} $\omega^{T-1}$
	\end{algorithmic}
\end{algorithm}

\subsection{Rethink of Federated ADMM Algorithm}

For the optimization problem with federated ADMM algorithm, corresponding AL function can be defined as following:

\begin{equation}
	\begin{aligned}
		{f(\omega^t)}=\sum_{i=1}^N{{f_i(\theta_i^t;\omega^{t-1},\lambda_i^{t})}},
		\label{eqn:al}
	\end{aligned}
\end{equation}
where $\theta_i^t$ is client $i$'s local model of communication round $t$ and $\theta_i^{t-1}=\omega^{t-1}$; $\lambda_i$ is the Lagrangian dual variable; and $f_i(\cdot)$ can be defined as Eqn.\ref{eqn:lal}:

\begin{equation}
	\begin{aligned}
		f_i(\theta_i^t;\omega^{t-1},\lambda_i^{t})=L_i(\theta_i^t)&+\langle{\lambda_i^{t},\theta_i^t-\omega^{t-1}}\rangle\\
		&+\frac{\gamma_i}{2}\rVert{\theta_i^t-\omega^{t-1}}\rVert^2,
		\label{eqn:lal}
	\end{aligned}
\end{equation}
where $\gamma_i>0$ is the corresponding penalty parameter. Through client-variance-reduction (CVR), the federated ADMM algorithm mitigate the impact of client drift caused by non-iid data \cite{fedvra}. \cite{feddyn,feddc} utilize these CVR scheme with other forms, the objective function can be summarized as Eqn.\ref{eqn:tldm1}:

\begin{equation}
	\begin{aligned}
		f_i(\theta_i^t;\omega^{t-1},h_i)=L_i(\theta_i^t)+R_i(\theta_i^t;\omega^{t-1},h_i),
		\label{eqn:tldm1}
	\end{aligned}
\end{equation}
where $L_i(\cdot)$ is the empirical loss term; $R_i(\cdot)$ is the AL-regular term; $h_i$ is local drift variables for correct client drift and can be defined as follows:

\begin{equation}
	\begin{aligned}
		h_i=h_i+\Delta{\theta_i^t},
		\label{eqn:tldm2}
	\end{aligned}
\end{equation}

\subsection{FedSSG Algorithm}

For each client with FedSSG, its local objective function of $t$-th communication round can be defined as Eqn.\ref{eqn:4}:

\begin{equation}
	\begin{aligned}
		f_i(\theta^t_i)=L_i(\theta_i^t)&+P_i(\theta^t_i;\omega^{t-1},h_i^{t-1})\\
		&+G_i(\theta^t;\Delta{\theta^{t-1}},\Delta{\omega^{t-1}}),
	\end{aligned}
	\label{eqn:4}
\end{equation}
where $\theta^t_i=\omega^{t-1}$; $h_i$ is the memo-corrected local drift variables; $L_i(\theta^t_i)$ is the empirical loss term; $P_i(\theta^t_i;\omega^{t-1},h_i^{t-1})$ is the penalized term and $G_i(\theta^t_i;\Delta{\theta^{t-1}_i},\Delta{\omega^{t-1}})$ is the gradient correction term. The penalized term uses stochastic sampling-guided local drift to track the parameter gap between the local and the global and break through the inherent performance with non-iid data. The gradient correction term would shorten the convergence time while ensuring convergence of FedSSG. For $e$-th local training iteration, local model updates satisfies as Eqn.\ref{eqn:5}:
\begin{equation}
	\theta_i^{t,e}=\theta_i^{t,e-1}-\eta\nabla_{\omega}{f_i({\theta_i^{t,e-1}})},
	\label{eqn:5}
\end{equation}
where $\eta$ is the learning rate of SGD. It is worth noting that the local drift is source from local model updates. The algorithm of FedSSG is shown in Supplementary B.

\subsubsection{Phase Strategy on Objective Function}

To utilize phase strategy to regularize the clients' model training via SSG local drift as Eqn.\ref{eqn:1}, the penalized term can be expressed as Eqn.\ref{eqn:6}:
\begin{equation}
	P_i(\theta_i^t)=\alpha{\langle \theta_i^t-(\omega^{t-1}-h_i^{t-1}), h_i^{t-1} \rangle},
	\label{eqn:6}
\end{equation}
where $\alpha$ is the hyper-parameters of FedSSG. In numerical terms, we can obtain Ineqn.\ref{eqn:7}:
\begin{equation}
	\begin{aligned}
		{\rVert{h_i^{t-1}}\rVert}^2&\le\rVert{P_i(\theta_i^t)}\rVert\\
		&\le{\rVert{\theta_i^t-(\omega^{t-1}-h_i^{t-1})}\rVert}^2,
	\end{aligned}
	\label{eqn:7}
\end{equation}
To realize the idea proposed above, in initial stage of communication, the penalized term satisfies as Eqn.\ref{eqn:8}:
\begin{equation}
	\begin{aligned}
		P_i(\theta_i^t)\approx\alpha{\langle \theta_i^t-(\omega^{t-1}-\mathit{o}(\omega^{t-1})),\mathit{o}(\omega^{t-1}) \rangle},
	\end{aligned}
	\label{eqn:8}
\end{equation}
When communication round $t$ is small enough, $P_i(\theta^t,\omega^{t-1},h_i^{t-1})\rightarrow0$. And in final stage of communication, the penalized term satisfies as Eqn.\ref{eqn:9}: 
\begin{equation}
	P_i(\theta_i^t)\approx{\rVert{h_i^{t-1}}\rVert}^2,
	\label{eqn:9}
\end{equation}
In this way, we realize the dynamic correction of objective function in different phases of the global training. This phase strategy ensure the consistency of FL's model.


To speed up the convergence time, the gradient correction term can be expressed as Eqn.\ref{eqn:10}:
\begin{equation}
	\begin{aligned}
		G_i(\theta_i^t)=\frac{1}{\eta{E}}{\langle \theta_i^t, \Delta{\theta_i^{t-1}}-\Delta{\omega^{t-1}} \rangle},
	\end{aligned}
	\label{eqn:10}
\end{equation}
where $E$ is the local training epochs of each communication. As Eqn.\ref{eqn:5}, the gradient correction term would correct local model training gradient as Eqn.\ref{eqn:11}:
\begin{equation}
	\nabla_{\omega}{f_i(\theta_i^{t,e})}=\nabla_{\omega}L_i(\theta_i^{t,e})+\frac{1}{\eta{E}}(\Delta{\theta_i^{t-1}}-\Delta{\omega^{t-1}}),
	\label{eqn:11}
\end{equation}
We have proved that corrected gradient is strictly bounded in Supplementary B.

\subsection{Discussion of FedSSG}

\begin{figure}[!t]
	\centering
        \setlength{\belowcaptionskip}{10pt}
	\includegraphics[width=0.46\textwidth]{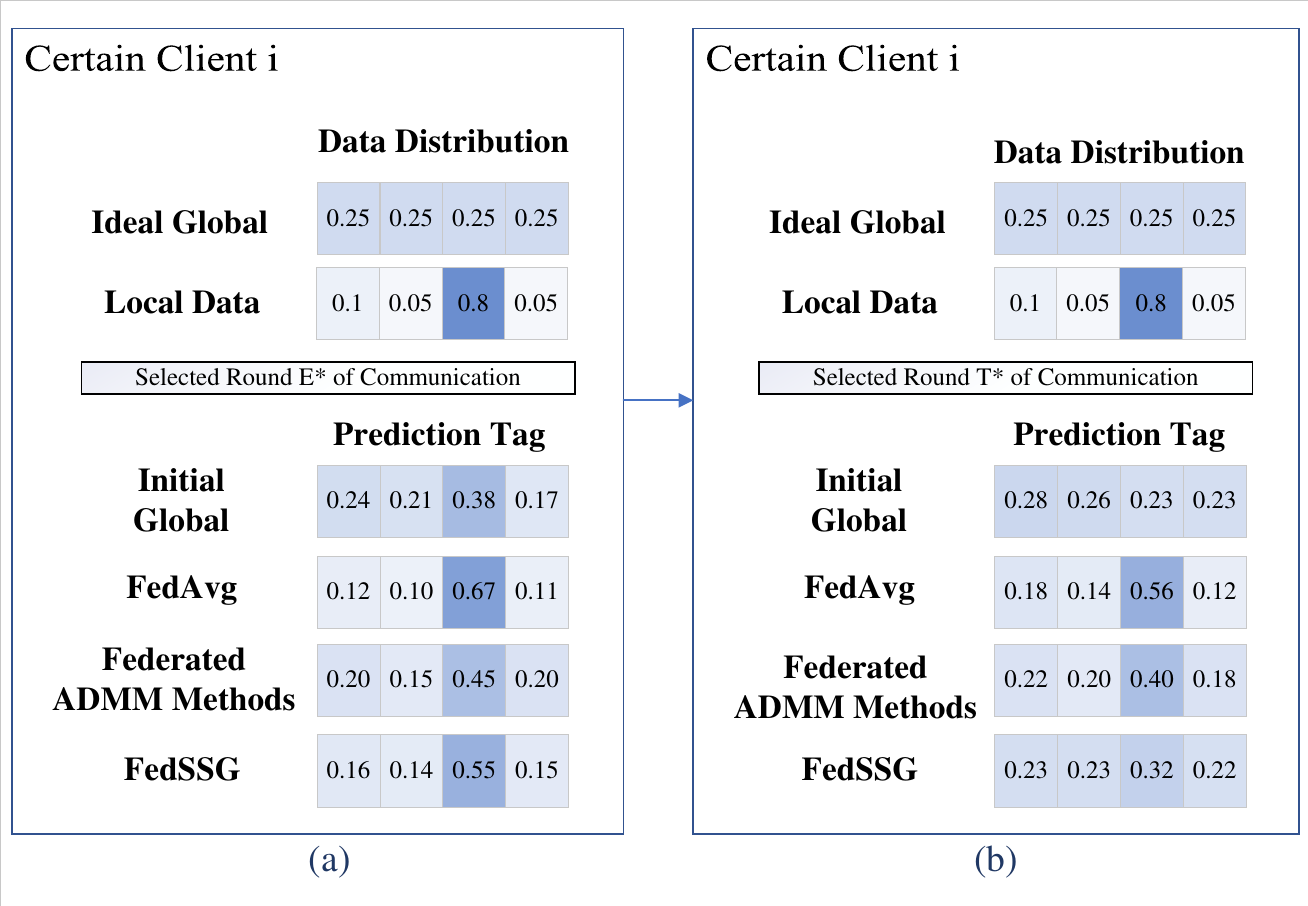}
	\caption{
		To show the phase strategy of FedSSG, (a) and (b) depict the prediction tag of different methods at selected round E*/T*, which means the early/terminal phase of its selected round. 
	}
	\label{fig:memo}
\end{figure}

In algorithm \ref{alg:FedSSG}, we present the algorithm flow of FedSSG, where $f$ is the objective function. In simple terms, for each communication round $t$, the server sample the selected client set $\mathcal{S}^t$, where $\mathcal{S}^t\in\left\{1,2,...,N\right\}$, and boardcasts the global model parameters to those clients selected. Then, for each selected client, it will update their local model parameter as the global and train the model by their own data. Through updating the local drift variables, it corrects the local updates and return it to the server. Finally, the server aggregates the corrected local updates and update the global model parameters. After $T$ times communication rounds, the server will obtain trained model ultimately. The scheme of FedSSG is shown as Figure \ref{fig:framwork}.

\subsubsection{Algorithm Flow and Convergence Process}

Before we begin the discussion, we have the following hypothesis, similar to other works in FL:\\

\textit{H.1: Because of non-iid data of various clients in the training, we suppose the local optimal points $\theta_i^*$ of client $i$ are arbitrarily different from $\theta_j^*$ of client $j$, where $i,j\in[N]$.} \\

\textit{H.2: To facilitate the writing of symbols, we assume that all clients participate in the training in each communication round. The $N$-clients means that all clients have taken part in each communication round. In practical, the number of participants does not affect the validity of the proof.}\\

\textit{H.3: The empirical loss term $L$ and $L_i$ satisfies the smoothness assumption.}\\

Under these hypothesis, we propose FedSSG. Compared with previous algorithms, in order to better mitigate the impact of statistical heterogeneity, we propose the correction term of local drift, and the objective equation after correcting local drift. In the text, we introduce the origin of the idea of correcting items. In this section, we mainly discuss the objective function, and prepare for the convergence proof in the next section. FedSSG has the following objective function as \ref{eqnofapp:1}:
\begin{equation}
	\begin{aligned}
		f_i(\theta^t_i)&=L_i(\theta^t_i)+P_i(\theta^t_i;\omega^{t-1},h_i^{t-1})\\
		&\ \ \ \ \ \ \ \ \ \ \ \ \ \ \ \ \ \ +G_i(\theta^t;\Delta{\theta^{t-1}}_i,\Delta{\omega^{t-1}})\\
		&=\mathbb{E}_{(x,y)\in{D_i}}l(\theta^t_i;(x,y))\\
		&\ \ \ \ \ \ \ \ \ \ \ \ \ \ \ \ \ \ +\alpha{\langle \theta_i^t-(\omega^{t-1}-h_i^{t-1)}), h_i^{t-1} \rangle}\\
		&\ \ \ \ \ \ \ \ \ \ \ \ \ \ \ \ \ \ +\frac{1}{\eta{E}}{\langle \theta_i^t, \Delta{\theta^{t-1}}-\Delta{\omega^{t-1}}, \rangle}
		\label{eqnofapp:1}
	\end{aligned}
\end{equation}
where $\Delta{\theta^{t-1}_i=g_i^{t-1}}$ and $\Delta{\omega^{t-1}}=\mathbb{E}_{i\in\mathcal{S}^t}g_i^{t-1}$, due to $g_i=\Delta{\theta_i}$ and $g=\Delta{\omega}$. Specifically, the objective function of FedSSG is consist of empirical loss term, penalized term and gradient correction term. Each client utilize this objective function to train local model parameters in $t$-th communication round. For each local training iteration $e$-th, the client optimizes local model parameters with local gradient descent as \ref{eqnofapp:2}:
\begin{equation}
	\theta_i^{t,e}=\theta_i^{t,e-1}-\eta\nabla_{\omega}{f_i({\theta_i^{t,e-1}})},
	\label{eqnofapp:2}
\end{equation}
Meanwhile, the local drift update value as \ref{eqnofapp:3}:
\begin{equation}
	\Delta{h_i^t}=-\xi_i^t\eta\nabla_{\omega}{f_i({\theta^{t,e}_i};h_i^{t,e})},
	\label{eqnofapp:3}
\end{equation}
where $\xi_i^t$ is the correction term for local drift variables. And then, clients upload their corrected local model parameters to the server and the server aggregate local model parameters to obtain the global model parameters. The global model updates can be written as \ref{eqnofapp:4}:
\begin{equation}
	\omega^t=\mathbb{E}_{i\in\mathcal{S}^t}(\theta^{t-1}_i+h_i^{t-1}+\Delta{\theta^t_i}+\Delta{h_i^t}),
	\label{eqnofapp:4}
\end{equation}

While global model has been updated like \ref{eqnofapp:4}, one communication round has been completed. After $T$ times communication rounds, the global model will converges to the stationary point when each client approach their own local optimal points. Specifically, $\Delta{\theta_i^t}\rightarrow0$, where $i\in[N]$, means the local model converges to the local optimal point. Ideally, the local model will no longer be updated as \ref{eqnofapp:5}:
\begin{equation}
	\theta^{*,t}_i=\theta^{*,t-1}_i+\Delta{\theta^{*,t}_i}=\theta^{*,t-1}_i,
	\label{eqnofapp:5}
\end{equation}
When all clients approach the local optimal point, namely $\sum_{i=1}^N\Delta{\theta^t}\rightarrow0$, the global model will converge to the stationary point in ideal as \ref{eqnofapp:6}:
\begin{equation}
	\begin{aligned}
		\omega^{*,t}&=\omega^{*,t-1}+\Delta{\omega^{*,t}}\\
		&=\omega^{*,t-1}+\sum_{i=1}^N\Delta{\theta^{*,t}_i}=\omega^{*,t-1},	
		\label{eqnofapp:6}
	\end{aligned}
\end{equation}
At this point, FedSSG will converge. We will discuss the convergence of FedSSG in more detail in the next section and prove it rigorously. 

\begin{figure}[t]
	\centering
        \setlength{\belowcaptionskip}{10pt}
	\includegraphics[width=0.46\textwidth]{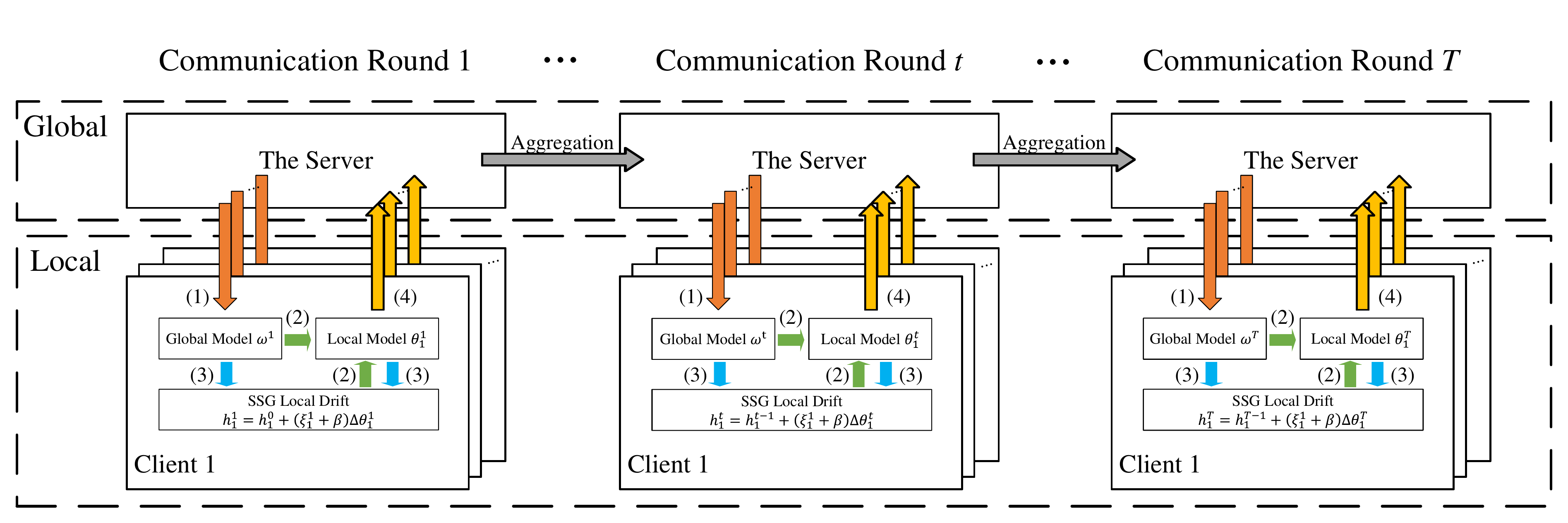}
	\caption{
	The architecture diagram of FedSSG. FedSSG records the model update history during the update and uses the phase-aware strategy based on SSG to enhance the local discrimination ability for heterogeneity.
}
\label{fig:framwork}
\end{figure}

After discussing the algorithmic flow and convergence process of FedSSG, we will discuss the penalty term and the gradient correction term of the objective function in the following pages.

\subsubsection{Penalized Term}

In real-world situation, it is impractical to assume that all clients can train local model to converge to a stationary point with non-iid data. Therefore, we design the penalized term to use corrected local drift variables to track the parameter gap between the local and the global. FedDC\cite{feddc} denotes that parameter deviation from clients' models to the server's model is caused by two factors: one is current round's update drift, and one is residual parameter deviation. FedProx\cite{fedprox} and SCAFFOLD\cite{scaffold} have proved that it is effective to decrease the update drift for reduce convergence time. To ensure the update drift decrement, which statistically equal to local drift, not only can speed up the convergence time but also help it track the parameter gap better, we design the correction term $\xi_i$ and use penalized term $\alpha{\langle \theta-(\omega^{t-1}-h_i^{t-1)}), h_i^{t-1} \rangle}$ to realize mitigate the parameter deviation. In the text, we clearly illustrate the correction term. The penalized term follows the train of correction thought, by using the difference between local drift variables and local gradient drift in the early and late training period to track the training process, adaptive global model decoupling and parameter difference tracking are realized.

\subsubsection{Gradient Correction Term}

From \ref{eqnofapp:1}, we split the objective function as three parts: the empirical loss term, penalized term and gradient correction term. In the process of gradient descend, as\ref{eqnofapp:2}, gradient correction term has corrected local model training gradient. Specifically, with empirical loss term $L_i(\theta^t)$ and $L(\omega)$, the corrected gradient of client $i$ satisfies \ref{eqnofapp:7}:
\begin{equation}
\begin{aligned}
	g_i^{t,e}&=\nabla_\omega{f_i(\theta^t_i)}\\
	&=\nabla_{\omega}L_i(\theta^{t,e}_i)+\frac{1}{\eta{E}}(\Delta{\theta^{t-1}_i}-\Delta{\omega^{t-1}})\\
	&\approx\nabla_{\omega}L_i(\theta^{t,e}_i)+(\nabla_{\omega}L(\theta^{t-1,e}_i)-\nabla_{\omega}L_i(\theta^{t-1,e}_i)),
	\label{eqnofapp:7}
\end{aligned}
\end{equation}
Therefore, the global gradient of $e$-th epoch can be expressed as \ref{eqnofapp:8}:
\begin{equation}
g^{t,e}=\frac{1}{N}\sum_{i=1}^Ng_i^{t,e}\approx\nabla{f(\theta^t_i)},
\label{eqnofapp:8}
\end{equation}
Thus, according to \ref{eqnofapp:7} and \ref{eqnofapp:8}, the gradient variance satisfies \ref{eqnofapp:9}:
\begin{equation}
\begin{aligned}
	\frac{1}{N}\sum_{i=1}^{N}&{\rVert{g_i^{t,e}-g^{t,e}}\rVert}^2\\
	&\approx\frac{1}{N}\sum_{i=1}^N\rVert\nabla_{\omega}L_i(\theta^{t,e}_i)+(\nabla_{\omega}L(\theta^{t-1,e}_i)\\
	&\ \ \ \ \ \ \ \ \ \ \ \ \ \ \ \ \ -\nabla_{\omega}L_i(\theta^{t-1,e}_i))-\nabla_{\omega}L(\theta^{t,e}_i)\rVert^2\\
	&\le\frac{2}{N}\sum_{i=1}^N({\rVert\nabla_{\omega}L_i(\theta^{t,e}_i)-\nabla_{\omega}L(\theta^{t-1,e}_i)\rVert}^2\\
	&\ \ \ \ \ \ \ \ \ \ \ \ \ \ \ \ \ +{\rVert\nabla_{\omega}L_i(\theta^{t-1,e}_i-\nabla_{\omega}L(\theta^{t,e}_i)\rVert}^2),
	\label{eqnofapp:9}
\end{aligned}
\end{equation}
So far, we have shown that the corrected gradient variance is strictly bounded. With H.3, we have proved that it is effective for gradient correction term to reduce gradient drift. As mentioned earlier, such a move would shorten the convergence time.

\subsection{Convergence proof of FedSSG}
\begin{assumption} \label{a1}
($L - Smoothness$). For a convex function f, which satisfies: $||\nabla f(w_1) - \nabla f(w_2)|| \leq L ||w_1 - w_2||$, for all $w_1,w_2 \in \mathbb{R}^n$. 
\end{assumption}

\begin{assumption} \label{a2}
($\beta$ - convex).  For a $\beta - $ convex function $f$ and $\beta > 0$ which satisfies: $f(w_2) \geq f(w_1) + <\nabla f(w_1), w_2 - w_1 > + \frac{\beta}{2} ||w_1 - w_2||^2 $, for all $w_1,w_2 \in \mathbb{R}^n.$
\end{assumption}

\begin{definition} \label{d1}
($B$ - local dissimilarity bounded). The local objective $f_i$ are $B$-local dissimilar at $\theta$ where $\mathbb{E}_i[||\nabla f_i(\theta)||^2] \leq ||\nabla f(w)||^2 B^2.$ Besides, we define $B(w) = \sqrt{\frac{\mathbb{E}||\nabla f_i(\theta)||^2}{||\nabla f(w)||^2}}$. 
\end{definition}

\begin{assumption} \label{a3}
($\psi$ - inexact solution). Define a function $M_i(\theta_i,\hat{\theta}_i) = f_i(\theta_i) + \frac{\alpha}{2} ||\theta_i - \hat{\theta}_i||^2$, where $\alpha \in [0,1]$ and $\hat{\theta}_i$ = $w - h_i$, if $\theta_i^*$ is a $\xi$-inexact solution of $\min_{\theta}M_i(\theta_i,\hat{\theta}_i)$, which satisfies:
$||\nabla M_i(\theta_i^*;\hat{\theta}_i)|| \leq \psi||M_i(\hat{\theta}_i;\hat{\theta}_i)||,$ and $||\nabla M_i(\theta_i;\hat{\theta}_i)|| = \nabla f_i(\theta_i) + \alpha(\theta_i - \hat{\theta}_i).$
\end{assumption}

\begin{assumption} \label{a4}
(Bounded dissimilarity). For some $\epsilon > $ 0, there exists a $B_{\epsilon}$, for all $w$ satisfies $||\nabla f(w)||^2 > \epsilon$ and $B(w) \leq B_\epsilon.$ 
\end{assumption}

\subsection{Detailed convergence proof of FedSSG} \label{sec:b.3}

\begin{theorem} \label{t1}
Given a non-convex and $L$-Lipschitz smooth local objective function $f_i$, and $f_i$ is $B$-dissimilarity, there exists $L\_ > 0$, such that $\nabla^2 f_i \geq -L\_I$ and $\overline{\alpha} = \alpha-L\_ > 0$. The global objective function $f$ satisfies: 
\begin{equation}  \label{eq18}
	\mathbb{E}_{S^t}f(\omega^{t+1})\leq f(\omega^{t})-\frac{p||\nabla f(\omega^r)||^2}{2},
\end{equation}
where $p= (1-\psi B) (\frac{2}{\alpha} ) - \frac{B(1+\psi)\sqrt{2}}{\sqrt{K}} (\frac{2}{\overline{\alpha}} ) - L(1+B)\psi (\frac{2}{\alpha \overline{\alpha}} ) - L B^2 (1+\psi)^2 (\frac{1}{2} + \frac{2\sqrt{2K} + 2}{K}) (\frac{2}{\overline{\alpha}^2}) > 0$, and $S^r$ is selected clients in $t$-th round.
\end{theorem}

\begin{proof}

Define $e_i^{r+1}$, which satisfies:

\begin{equation} \label{eq20}
	\begin{split}
		& \nabla f_i(\theta_i^{t+1}) + \alpha_i (\theta_i^{t+1} - w^t) - e_i^{t+1} = 0 \\
		&\ \ \ \ \ \ \ \ \ \ \ \ \ \ ||e_i^{t+1}|| \leq \psi ||\nabla f_i(w^t)||,
	\end{split}
\end{equation}

Let us define $\overline{w}^{t+1} = \mathbb E_{i\in\mathcal{S}} [\theta_i^{t+1}] $ and have:
\begin{equation} \label{eq21}
	\overline{w}^{t+1} - w^t = \frac{-1}{\alpha_i} \mathbb E_{i\in\mathcal{S}} [\nabla f_i(w^t)] + \frac{1}{\alpha_i} \mathbb E_{i\in\mathcal{S}} [e_i^{t+1}].
\end{equation}
Let us define $\overline{\alpha}_i = \alpha_i - L\_ > 0$ and $\Tilde{\theta}_i^{t+1}$ = $\operatorname*{argmin}_{\theta_i} {M_i(\theta_i;w^t)}$, due to the $\overline{\alpha}_i$-strong convexity of $M_i(\theta_i;w^t)$, we have:
\begin{equation} \label{eq22}
	||\Tilde{\theta}_i^{t+1} - \theta_i^{t+1}|| \leq \frac{\psi}{\overline{\alpha}_i} ||\nabla f_i(w^t)||.
\end{equation}
Use $\overline{\alpha}_i$-strong convexity of $M_i(\theta_i;w^t)$ again and have:
\begin{equation} \label{eq23}
	||\Tilde{\theta}_i^{t+1} - w^t|| \leq \frac{1}{\overline{\alpha}_i}||\nabla f_i(w^t)||.
\end{equation}
Using triangle inequality of Eq.~\ref{eq22} and Eq.~\ref{eq23}, we have:
\begin{equation} \label{eq24}
	||\Tilde{\theta}_i^{t+1} - w^t|| \leq \frac{1+\psi}{\overline{\alpha}_i} ||\nabla f_i(w^t)||.
\end{equation}

Therefore, we have:
\begin{equation} \label{eq25}
	\begin{split}
		&||\overline{w}^{t+1} - w^t|| \leq \mathbb E_{i\in\mathcal{S}} [||\theta_i^{t+1} - w^t||] \\
		&\ \ \ \ \ \ \ \ \ \ \ \ \ \ \ \ \ \ \ \ \ \ \  \leq \frac{1+\psi}{\overline{\alpha}_i} \mathbb E_{i\in\mathcal{S}} [||\nabla f_i(w^t)||] \\
		&\ \ \ \ \ \ \ \ \ \ \ \ \ \ \ \ \ \ \ \ \ \ \  \leq \frac{1+\psi}{\overline{\alpha}_i} \sqrt{\mathbb E_{i\in\mathcal{S}} [||\nabla f_i(w^t)||^2]} \\
		&\ \ \ \ \ \ \ \ \ \ \ \ \ \ \ \ \ \ \ \ \ \ \  \leq \frac{B(1+\psi)}{\overline{\alpha}_i} ||\nabla f(w^t)||,
	\end{split}
\end{equation}
where the first and the third inequalities are due to the nature of expectations, the last inequality is due to the bounded dissimilarity assumption.

Let us define $M_{t+1}$ such that $\overline{w}^{t+1} - w^t = -\frac{1}{\alpha_i} (\nabla f(w^t) + M_{t+1})$, i.e. $M_{t+1} = \mathbb E_i [\nabla f_i(\theta_i^{t+1}) - \nabla f_i(w^t) - e_i^{t+1}]$. Now let us bound $M_{t+1}$:
\begin{equation} \label{eq26}
	\begin{split}
		&||M_{t+1}|| \leq \mathbb E_{i\in\mathcal{S}} [L ||\theta_i^{t+1} - \theta_i^t|| + ||e_i^{t+1}||] \\
		&\ \ \ \ \ \ \ \ \ \ \ \ \ \  \leq ( \frac{L(1+\psi)}{\overline{\alpha}_i} + \psi) \times \mathbb E_{i\in\mathcal{S}} [||\nabla f_i(w^t)||] \\
		&\ \ \ \ \ \ \ \ \ \ \ \ \ \  \leq (\frac{L(1+\psi)}{\overline{\alpha}_i} + \psi) B ||\nabla f(w^t)||,
	\end{split}
\end{equation}
where the last inequality follows from bounded dissimilarity assumption.

Notably, due to $\Tilde{\theta}_i = w - \xi_i h_i$, define $\xi = \frac{1}{N} \sum_{i}^{N} \xi_i$ we have $\mathbb E_{i\in\mathcal{S}}(\theta_i^{t+1} - \theta_i^{t} ) =  \xi \mathbb E_{i\in\mathcal{S}}(h_i^{t+1} - h_i^{t})$, therefore, we have:
\begin{equation} \label{eq35}
	\begin{split}
		& w^{t+1} - w^{t} = \mathbb E_{i\in\mathcal{S}}(\theta_i^{t+1} - \xi_i h_i^{t+1}) - \mathbb E_{i\in\mathcal{S}}(\theta_i^{t} - \xi_i h_i^{t}) \\
		&\ \ \ \ \ \ \ \ \ \ \ \ \ \ \ \ \ \  = (1 + \xi) \mathbb E_{i\in\mathcal{S}}(\theta_i^{t+1} - \theta_i^{t}),
	\end{split}
\end{equation}

Based on $L$-lipschitz smoothness of $f$ and Taylor expansion, we have:
\begin{equation} \label{eq27}
	\begin{aligned}
		& f(\overline{w}^{t+1}) \leq f(w^t) + <\nabla f(w^t), (\overline{w}^{t+1}-w^t)>\\
		&\ \ \ \ \ \ \ \ \ \ \ \  \ \ \ \ \ \ \ \ \ \ \ \  \ \ \ \ \ \ \ \ \ \ \ \  \ \ \ \ \ \ \ \ \ \ \ \   + \frac{L}{2} ||\overline{w}^{t+1} - w^t||^2 \\
		&\ \ \ \ \ \ \ \ \ \ \ \  \leq f(w^t) - \frac{1+\xi}{\alpha_i}||\nabla f(w^t)||^2\\
		&\ \ \ \ \ \ \ \ \ \ \ \  \ \ \ \ \ \ \ \ \ \ \ \  \ \ \ \ \ \ \ \ \ \ \ \  \   - \frac{1+\xi}{\alpha_i}<\nabla f_i(w^t), M_{t+1}> \\
		&\ \ \ \ \ \ \ \ \ \ \ \ \ \ \ \ \ \ \ \ \ \ \ \ \  + \frac{L(1+\psi)^2 B^2 (1+\xi)^2}{ 2 \overline{\alpha}_i^2}||\nabla f(w^t)||^2 \\
		&\ \ \ \ \ \ \ \ \ \ \ \ \ \  \leq f(w^t) - (1+\xi)(\frac{1-\psi B}{\alpha_i} - \frac{LB(1+\psi)}{\overline{\alpha}_i\alpha_i}\\
		&\ \ \ \ \ \ \ \ \ \ \ \ \ \ \ \ \ \ \ \ \ \  - \frac{L(1+\psi)^2 B^2 (1+\xi )}{2 \overline{\alpha}_i^2}) 
		\times ||\nabla f(w^t)||^2,
	\end{aligned}
\end{equation}
So we can get the conclusion that $f(\overline{w}^{t+1} - f(w^t))$ is proportional to $||\nabla f(w^t)||^2$. Besides, for the partial clients participating the training each round, we need to find $\mathbb E[f(w^{t+1})]$.

Using the local Lipschitz continuity of the function $f$, we have:
\begin{equation} \label{eq28}
	f(w^{t+1}) \leq f(\overline{w}^{t+1}) + L_0 ||w^{t+1} - \overline{w}^{t+1}||,
\end{equation}
where $L_0$ is the local Lipschitz continuity constant for function $f$ and we have:
\begin{equation} \label{eq29}
	\begin{split}
		& L_0 \leq ||\nabla f(w^t)|| + L \times \max(||\overline{w}^{t+1} - w^t||, ||w^{t+1} - w^t||) \\
		&\ \ \ \ \ \ \ \ \ \ \ \ \ \ \ \ \ \  ||\nabla f(\theta_i^t)|| + L ( ||\overline{w}^{t+1} - w^t|| + ||w^{t+1} - w^t||).
	\end{split}
\end{equation}
Therefore, tack expectation with respect to selected clients in round $t$, we need to bound:
\begin{equation} \label{eq30}
	\mathbb E_{\mathcal{S}^t} [f(w^{t+1})] \leq f(\overline{w}^{t+1}) + Q_t,
\end{equation}
where $Q_t = \mathbb E_{\mathcal{S}^t} [L_0 ||w^{t+1} - \overline{w}^{t+1}||]$. Using Eq.~\ref{eq29}, we have:
\begin{equation} \label{eq31}
	\begin{split}
		& Q_t \leq \mathbb E_{\mathcal{S}^t} [(||\nabla f(w^t)|| + L ||\overline{w}^{t+1} - w^t|| + ||w^{t+1} - w^t||) \\
		&\ \ \ \ \ \ \ \ \ \ \ \ \ \ \ \ \ \ \ \ \ \ \ \ \ \ \ \ \ \ \ \ \ \ \ \ \ \ \ \ \ \ \ \ \ \ \ \ \ \ \ \ \ \times ||w^{t+1} - \overline{w}^{t+1}||] \\
		&\ \ \ \ \ \ \ \ \ \  (||\nabla f(w^t)|| + L ||\overline{w}^{t+1} - w^t||) \mathbb E_{\mathcal{S}^t} [||w^{t+1} - \overline{w}^{t+1}||] \\
		&\ \ \ \ \ \ \ \ \ \ \ \ \ \ \ \ \ \ \ \ \ \ \ \ \ \ \ + L \mathbb E_{\mathcal{S}^t} [||w^{t+1} - w^t|| * || w^{t+1} - \overline{w}^{t+1}||] \\
		&\ \ \ \ \ \ \ \  (||\nabla f(w^t)|| + 2L ||\overline{w}^{t+1} - w^t||) \mathbb E_{\mathcal{S}^t} [||w^{t+1} - \overline{w}^{t+1}||] \\
		&\ \ \ \ \ \ \ \ \ \ \ \ \ \ \ \ \ \ \ \ \ \ \ \ \ \ \ \ \ \ \ \ \ \ \ \ \ \ \ \ \ \ \ \ + L \mathbb E_{\mathcal{S}^t} [||w^{t+1} - \overline{w}^{t+1}||^2].
	\end{split}
\end{equation}
Due to $||\overline{w}^{t+1} - w^t|| \leq \frac{B(1+\psi)}{\overline{\alpha}_i} ||
\nabla f(w^t)||$, we have:
\begin{equation} \label{eq32}
	\begin{split}
		&\mathbb E_{\mathcal{S}^t} [||w^{t+1} - \overline{w}^{t+1}||^2] \\
		&\ \ \ \ \ \ \ \ \ \ \ \ \ \ \ \ \ \ \  \leq \frac{1}{K} \mathbb E_i [||\theta_i^{t+1} - \overline{w}^{t+1}||^2] \\
		&\ \ \ \ \ \ \ \ \ \ \ \ \ \ \ \ \ \ \  \leq \frac{2}{K} \mathbb E_i[||\theta_i^{t+1} - w^t||^2] \\
		&\ \ \ \ \ \ \ \ \ \ \ \ \ \ \ \ \ \ \  \leq \frac{2}{K} \frac{(1+\psi)^2}{\overline{\alpha}_i} \mathbb E_i [||\nabla f_i(\theta^t)||^2] \\
		&\ \ \ \ \ \ \ \ \ \ \ \ \ \ \ \ \ \ \  \leq \frac{2B^2(1+\psi)^2}{K\overline{\alpha}_i^2} ||\nabla f(w^t)||^2,
	\end{split}
\end{equation}
where the first inequality follows from $|S^t|$ clients being chosen to get $w^t$, the second inequality follows from $\overline{w}^{t+1} = \mathbb E_i [\theta_i^{t+1}]$, the third inequality follows from Eq.~\ref{eq25} and the last inequality comes from th bounded dissimilarity assumption. 

Combining Eq.~\ref{eq31} and Eq.~\ref{eq32}, we have:
\begin{equation} \label{eq33}
	\begin{aligned}
		&Q_t\leq (\frac{\sqrt{2}B(1+\psi)}{\sqrt{K} \overline{\alpha}_i}\\
		&\ \ \ \ \ \ \ \ \ \ \ \ \ \ \ \ \ \  + \frac{LB^2(1+\psi)^2}{\overline{\alpha}_i^2 K }(2\sqrt{2K} + 2)) ||\nabla f(w^t)||^2, 
	\end{aligned}
\end{equation}
and combining Eq.~\ref{eq27},~\ref{eq28},~\ref{eq29} and~\ref{eq33}, we get:
\begin{equation} \label{eq34}
	\begin{split}
		& \mathbb E_{S^t} [f(w^{t+1})] \leq f(w^t) - (1+\xi)(\frac{1}{\alpha_i} - \frac{\psi B}{\alpha_i}\\
		&\ \ \ \ - \frac{B(1+\psi)\sqrt{2}}{\overline{\alpha}_i\sqrt{K}} - \frac{LB(1+\psi)} {\overline{\alpha}_i \alpha_i} -  \frac{L(1+\psi)^2 B^2 (1+\xi)} 
		{2\overline{\alpha}_i^2}\\
		&\ \ \ \ \ \ \ \ \ \ \ \ \ \ \ \ \ \ \ \ \ \ \ \ -  \frac{LB^2(1+\psi)^2}{\overline{\alpha}_i^2 K } (2\sqrt{2K} + 2) ) ||\nabla f(w^t)||^2.
	\end{split}
\end{equation}

\end{proof}
\subsection{Bounded Gradients} \label{sec:b.4}
Let Assumption~\ref{a1} hold, e.g. for each function $f_i(\theta)$, the bounded variance of gradient satisfies: 
\begin{equation} \label{eq39}
\mathbb E_{i\in\mathcal{S}}[||\nabla f_i(\theta_i) - \nabla f(w)||^2] \leq \sigma^2.
\end{equation}
Then, for any $\epsilon > 0$, if follows that $B_\epsilon \leq \sqrt{1 + \frac{\sigma^2}{\epsilon}}$
\begin{proof}
\begin{equation} \label{eq40}
	\begin{split}
		&\mathbb E_{i\in\mathcal{S}} [||\nabla f_i(\theta_i) - \nabla f(w)||^2]\\
		&\ \ \ \ \ \ \ \ \ \ \ \ \ \ \ \ \ \ \ \ \ \  = \mathbb E_{i\in\mathcal{S}} [||\nabla f_i(\theta_i)||^2] - ||\nabla f(w)||^2 \leq \sigma^2  \\
		& \mathbb E_{i\in\mathcal{S}} [||\nabla f_i(\theta_i)||^2] \leq ||\nabla f(w)||^2 + \sigma^2 \\
		& B_\epsilon = \sqrt{\frac{\mathbb E_{i\in\mathcal{S}} ||\nabla f_i(\theta_i)||^2}{||\nabla f(w)||^2}} \leq \sqrt{1+ \frac{\sigma^2}{\epsilon}}.
	\end{split}
\end{equation}
\end{proof}

\subsection{Convergence of FedSSG in non-convex case} \label{sec:b.5}
Given a non-convex and $L$-Lipschitz smooth local objective function $f_i$, and $f_i$ is $B$-dissimilarity, there exists $L\_ > 0$, such that $\nabla^2 f_i \geq -L\_I$ and $\overline{\alpha} = \alpha-L\_ > 0$. The global objective function $f$ satisfies: 
\begin{equation} \label{eq41}
\mathbb{E}_{\mathcal{S}^t}f(\omega^{t+1})\leq f(\omega^{t})-\frac{p||\nabla f(\omega^r)||^2}{2},
\end{equation}
where $\mathcal{S}^t$ is the client-set in $t$th round.

\subsection{Convergence of FedSSG in convex case} \label{sec:b.6}
Suppose $L\_ = 0$, $\overline{\alpha} = \alpha$ in convex case, if $\psi = 0$, which means all clients are accurately, we can get a decrease proportional to $||\nabla f(w^t)||^2$ if $B < \sqrt{K}$. Assume $1<< B \leq 0.5\sqrt{K}$, we get:
\begin{equation} \label{eq42}
\mathbb E_{\mathcal{S}^t} f(w^{r+1}) \leq f(w^t) - ((1+\xi)(\frac{1}{2\alpha} - \frac{3 L B^2}{2\alpha^2}) ) ||\nabla f(w^t)||^2).
\end{equation}
Setting $\alpha = 6LB^2$, we get:
\begin{equation} \label{eq43}
\mathbb E_{\mathcal{S}^t} f(w^{t+1}) \leq f(w^t) - \frac{(1+\xi)}{24LB^2} ||\nabla f(w^t)||^2,
\end{equation}
Using the above inequality, we characterize the FedAGC's convergence rate. Given a threshold $\epsilon$ where $\sum_{t=1}^{T} ||\nabla f(w^t)||^2 \leq \epsilon$, we denote $\Gamma=f(w^0)-L(w^*)$ to represent the optimal point at $T$th round. From Eq.~\ref{eq43}, we have:
\begin{equation} \label{eq44}
\begin{split}
	\mathbb E_{\mathcal{S}^T} f(w^T) - f(w^0) &= \mathbb E_ {\mathcal{S}^t} f(w^*) - f(w^0) \\
	& \leq - \sum_{t=1}^{T} \frac{(1+\xi)}{24LB^2} ||\nabla f(w^t)||^2, 
\end{split}
\end{equation}
from Eq.~\ref{eq44}, we get:
\begin{equation} \label{eq45}
\sum_{t=1}^{T} ||\nabla f(w^t)||^2 \leq \frac{24LB^2}{(1+\xi)} (f(w^0) - f(w^*)).
\end{equation}
Therefore, FedSSG spend $\mathcal{O}(\frac{24LB^2\Gamma}{(1+\xi)\epsilon})$ to achieve global model convergence.